\documentclass{article}

\pdfoutput=1

\usepackage[preprint]{neurips_2026}

\usepackage[utf8]{inputenc} 
\usepackage[T1]{fontenc}    
\usepackage{hyperref}       
\usepackage{url}            
\usepackage{booktabs}       
\usepackage{amsfonts}       
\usepackage{nicefrac}       
\usepackage{microtype}      
\usepackage{xcolor}         
\usepackage{graphicx}
\usepackage{multirow}
\usepackage{algorithm}
\usepackage{algorithmic}
\providecommand{\STATEx}{\item[]} 
\usepackage{amsthm}
\usepackage{amssymb}
\usepackage{amsmath}
\usepackage{bm}
\usepackage{enumitem}
\usepackage{titlesec}
\titlespacing*{\section}{0pt}{1.0ex plus 0.2ex minus 0.1ex}{0.7ex plus 0.2ex}

\newtheorem{theorem}{Theorem}
\newtheorem{lemma}[theorem]{Lemma}
\newtheorem{proposition}[theorem]{Proposition}
\newtheorem{corollary}[theorem]{Corollary}

\newtheorem{remark}{Remark}

\newcommand{\tsnn}{\textsc{Tsnn}}
\newcommand{\RR}{\mathbb{R}}
\newcommand{\transpose}{\top}
\DeclareMathOperator{\diag}{diag}
\DeclareMathOperator{\blkdiag}{blkdiag}
\DeclareMathOperator{\GRU}{GRU}
\DeclareMathOperator{\MLP}{MLP}
\DeclareMathOperator{\softmx}{softmax}
\DeclareMathOperator{\softplus}{softplus}

\newcommand{\concat}{\,\|\,}

\title{Temporal Sheaf Neural Networks with Dynamic Orthogonal Transport}

\author{%
  Md Sadek Hossain Asif\thanks{Equal contribution.} \\
  University of Dhaka \\
  \texttt{\small mdsadekhossain-2021211191@cs.du.ac.bd}
  \And
  Tanzila Khan\footnotemark[1] \\
  University of Dhaka \\
  \texttt{\small tanzila-2021011201@cs.du.ac.bd}
  \And
  Md.\ Mosaddek Khan \\
  University of Dhaka \\
  \texttt{\small mosaddek@du.ac.bd}
}

\begin{document}

\maketitle

\begin{abstract}
We introduce \textbf{Temporal Sheaf Neural Networks (\tsnn{})}, a temporal link prediction framework that equips each node with a time-varying orthogonal frame and compares node states only after explicit transport between local coordinate systems. In contrast to existing continuous-time graph models that operate in a shared global embedding space, \tsnn{} models node-specific and evolving interaction semantics through dynamic local frames. The model parameterizes per-node frames via efficient low-rank Householder products, preserves stored hidden states exactly under frame updates, and uses a geometric-residual decoder that anchors predictions on transported distances while learning residual corrections. All computations are strictly causal and use only the pre-event history. We show that the symmetric degree-normalized sheaf Laplacian is orthogonally similar to the symmetric normalized graph Laplacian, with the random-walk normalized form similar in the corresponding degree metric; the full-active, feature-scaled diffusion used by \tsnn{} is exactly a metric-gradient step on the combinatorial sheaf Dirichlet energy, with a degree-free monotone-descent and non-expansiveness guarantee. Frame drift perturbs updates only linearly. Across TGB~v2 link-prediction and temporal-heterogeneous leaderboards, together with the DGB benchmark suite, \tsnn{} matches or surpasses the strongest prior methods on most benchmarks, with the largest improvements on graphs exhibiting strong node-role heterogeneity. Ablations confirm the distinct benefit of dynamic frames, orthogonal transport, and geometric-residual decoding.\footnote{Code, configs, and reproduction scripts: \url{https://github.com/TanzilaKhan1/TSNN-Code}.}
\end{abstract}

\section{Introduction} \label{sec:introduction}
Temporal graphs model continuously evolving interactions across high-impact domains, including social networks, recommender platforms, financial transactions, communication networks, and knowledge bases \citep{huang2023tgb,gastinger2024tgb2}. A central challenge is \emph{temporal link prediction}: given a chronological stream of past events, infer the next connection. Accurate prediction enables personalized recommendation, fraud and anomaly detection, emerging-relationship discovery, demand forecasting, and proactive resource allocation. As event-driven relational data scales up, reasoning over temporal interactions becomes a fundamental challenge in dynamic relational learning.

A continuous-time temporal graph model carries a hidden state per node and reads it whenever a new interaction arrives. Every downstream prediction therefore depends on a tacit assumption: the coordinate axes defining a node's state at training time still mean the same thing at scoring time. This assumption silently fails. Roles drift as users switch communities, products change categories, and proteins acquire new partners; recurrent updates keep stored vectors numerically stable but not semantically stable, because the basis itself has shifted. Two memories can be aligned on disk yet measure different concepts, and a score that compares them can rank the wrong destination above the right one. This is the \emph{representation-alignment problem}: latent coordinates evolve heterogeneously across nodes and time, while the score function treats them as if they did not. Existing continuous-time temporal GNNs do not address it: memory-based architectures \citep{trivedi2019dyrep,kumar2019jodie,rossi2020tgn,cong2023graphmixer,wang2024contig} update states in a single shared basis, and event- and neighborhood-centric encoders \citep{xu2020tgat,wang2021cawn,luo2022nat,yu2023dygformer,lu2024tpnet,souza2022pint,wang2021tcl,zhang2024tncn} encode walks, patches, and co-occurrence in the same basis. Both inherit the unstated global-alignment assumption. Causality and scalability are necessary but not sufficient: a temporal model also needs an explicit, node-local notion of meaning and an operator that aligns interacting states before comparing them.

A complementary line of work shows the value of node-local coordinate systems and explicit transport in geometric graph learning. Sheaf-based models assign local vector spaces to nodes and connect them through learnable maps, inducing operators that generalize graph diffusion and improve robustness to heterophily and over-smoothing \citep{hansen2019spectral,hansen2020sheaf,bodnar2022neural,barbero2022attention}. Connection-based variants study orthogonal transport and spectrally controlled operators \citep{barbero2022sheaf,bamberger2025bundle,ribeiro2026csnn,borgi2025polynsd,choi2026sgpc}. Related geometric architectures also exploit local frames and transport, including LoCS \citep{kofinas2021locs}, Gauge Equivariant Mesh CNNs \citep{dehaan2021gemcnn}, the Gauge Equivariant Transformer \citep{he2021get}, ClofNet \citep{du2022clofnet}, and Hermes \citep{park2023hermes}. These works show that local coordinates improve representation consistency, but primarily target static graphs, meshes, or physical dynamical systems rather than continuous-time event streams.

This creates a precise gap. Temporal graph models deliver causal memory updates, neighborhood encoding, and event prediction at scale, but they do not provide an operator-level mechanism for preserving representation meaning under changing node-local semantics. Sheaf-based and geometric transport models provide local vector spaces, restriction maps, and coordinate alignment, yet their applications remain static graphs, meshes, or physical dynamical systems rather than chronological event streams. Other dynamic-graph priors improve connectivity, curvature, or propagation pathways, but they do not solve the same coordinate-consistency problem. A temporal model that unifies evolving node-local frames, explicit transport, and online causal memory updates is therefore missing. We address this challenge with \textbf{Temporal Sheaf Neural Networks (\tsnn{})}, a temporal link-prediction model that extends sheaf-based representation learning to continuous-time interaction streams. \tsnn{} equips each node with a time-varying orthogonal frame, aligns interacting states by explicit transport, and carries node memories consistently across frame changes, preserving semantic content under evolving local coordinates. Scoring is strictly causal, uses only $\mathcal{E}_{<t}$, and combines basis-independent geometric alignment with a residual decoder. Our main contributions are:
\begin{enumerate}[leftmargin=*,
                  itemsep=2pt,
                  topsep=2pt,
                  parsep=2pt,
                partopsep=0pt]
\item \textbf{A sheaf-structured temporal GNN for causal link prediction} (\S\ref{sec:method}). Each node has a time-varying orthogonal frame parameterized by a low-rank Householder product, requiring only $\mathcal{O}(kd)$ parameters per node, while candidate links are scored strictly from pre-event history.
   \item \textbf{Coordinate-consistent memory evolution} (\S\ref{sec:sheaf_construction}--\S\ref{sec:event_processing}). \tsnn{} transports interacting states into a common frame before aggregation and carries stored hidden states exactly across frame updates, preserving semantic meaning as local coordinates evolve.
\item \textbf{Spectral and stability guarantees} (\S\ref{sec:theory}). The induced degree-normalized sheaf Laplacian has spectrum in $[0,2]$; the full-active feature-scaled diffusion operator is metric-gradient descent on the combinatorial sheaf Dirichlet energy with degree-free step-size bound $\eta \le 1/\lambda_{\max}(\mathbf{D}_\theta)$; and perturbation grows only linearly with frame drift.
\item \textbf{Strong empirical results across temporal benchmarks}. Across TGB v2, heterogeneous temporal graph leaderboards, and 13 DGB datasets, \tsnn{} matches or outperforms strong baselines, with ablations validating the roles of dynamic frames, transport, and geometric-residual scoring.
\end{enumerate}

\section{Preliminaries}
\label{sec:preliminaries}
A continuous-time dynamic graph (CTDG) $\mathcal{G} = (\mathcal{V}, \mathcal{E})$ comprises a node set $\mathcal{V}$ and a stream of timestamped edges $\mathcal{E} = \{(u_i, v_i, t_i, r_i, \mathbf{x}_i)\}_{i=1}^{|\mathcal{E}|}$, where $u_i, v_i \in \mathcal{V}$, $t_i \in \RR_{\geq 0}$ is the timestamp, $r_i$ is the edge type (relation), and $\mathbf{x}_i \in \RR^{\dim(\mathbf{x})}$ is the edge feature vector. Edges are \emph{directed}: $(u, v, t, r, \mathbf{x})$ records an interaction from source $u$ to destination $v$ at time $t$. Events are sorted chronologically.

\textit{
A cellular sheaf $\mathcal{F}$ over a graph $G = (V, E)$ assigns:
\begin{itemize}
    \item A vector space $\mathcal{F}(v) \cong \RR^d$ to each node $v$ (the \emph{stalk} at $v$), and
    \item A linear map $\mathcal{F}_{v \unlhd e} : \mathcal{F}(v) \to \mathcal{F}(e) \cong \RR^d$ to each incidence of node $v$ in edge $e$ (the \emph{restriction map}).
\end{itemize}}
For an edge $e = (u, v)$, the two restriction maps $\mathcal{F}_{u \unlhd e}$ and $\mathcal{F}_{v \unlhd e}$ define how the stalks of $u$ and $v$ are related through $e$. The standard \emph{sheaf Laplacian} (degree-0) is the block matrix $\mathbf{L}_{\mathcal{F}} \in \RR^{|V|d \times |V|d}$ whose $(u, v)$-block is:
\begin{equation}
[\mathbf{L}_{\mathcal{F}}]_{uv} = \begin{cases}
        \sum_{e \ni u} \mathcal{F}_{u \unlhd e}^{\transpose} \mathcal{F}_{u \unlhd e} & \text{if } u = v, \\
        -\sum_{e:\,\{u,v\}\subseteq\partial e} \mathcal{F}_{u \unlhd e}^{\transpose} \mathcal{F}_{v \unlhd e} & \text{if } u\sim v \text{ in } E, \\
        \mathbf{0} & \text{otherwise,}
    \end{cases}
    \label{eq:sheaf_laplacian}
\end{equation}
where $\partial e=\{u,v\}$ denotes the endpoints of edge $e$ and $E$ is treated as a multiset, so parallel edges contribute multiplicity to both the diagonal and the off-diagonal blocks. The \emph{pairwise transport map} is $\mathbf{Q}_{uv} = \mathcal{F}_{u \unlhd e}^{\transpose} \mathcal{F}_{v \unlhd e} \in \RR^{d \times d}$. When restriction maps are orthogonal, $\mathbf{Q}_{uv}$ is also orthogonal and translates between the coordinate frames of $u$ and $v$. \tsnn{} instantiates these restriction maps by setting $\mathcal{F}_{v \unlhd e} = \mathbf{U}_v$, a per-node orthogonal frame; the pairwise transport then specializes to $\mathbf{Q}_{uv} = \mathbf{U}_u^\top \mathbf{U}_v$, so $\mathbf{Q}_{vu} = \mathbf{Q}_{uv}^\top$. Given the history $\mathcal{E}_{<t} = \{(u_i, v_i, t_i, r_i, \mathbf{x}_i) : t_i < t\}$, the goal is to estimate the destination $v$ of a query $(u, ?, r, t)$. 
\section{Temporal Sheaf Neural Networks}
\label{sec:method}

\tsnn{} processes a temporal edge stream in chronological order under a strict \emph{predict-then-update} protocol: when an event $(u, v, t, r, \mathbf{x})$ arrives, candidate scores are computed solely from the pre-event history $\mathcal{E}_{<t}$, and the observed interaction may update internal states only after the ranking loss is formed. Each node $v \in \mathcal{V}$ maintains a stalk vector $\mathbf{h}_v(t) \in \RR^d$ (semantic memory) and frame parameters $\mathbf{F}_v(t) \in \RR^{k \times d}$ whose rows are Householder vectors parameterizing an orthogonal node-local frame $\mathbf{U}_v(t) \in \RR^{d \times d}$. These states are advanced through three per-event stages: \emph{Score} ranks the true destination against $K_{\text{neg}}$ negatives via a geometric-residual decoder combined with an EdgeBank memory prior; \emph{Update} refreshes endpoint frames, transports pre-event stalks into the new coordinates, and runs a GRU on frame-aligned messages; \emph{Diffuse} applies $K$ rounds of full-active feature-scaled normalized sheaf diffusion over the event-local active graph from each endpoint's $K_{\text{nbr}}$ most recent neighbors. The remainder follows Algorithm~\ref{alg:tsnn}; Figure~\ref{fig:overview} summarizes the pipeline.

\begin{algorithm}[htbp]
\caption{\tsnn{}: Processing a single event $(u, v, t, r, \mathbf{x})$}
\label{alg:tsnn}
\footnotesize
\begin{algorithmic}[1]
\REQUIRE Event $(u, v, t, r, \mathbf{x})$; current stalks $\mathbf{h}_u(t^-), \mathbf{h}_v(t^-)$; frames $\mathbf{F}_u(t^-), \mathbf{F}_v(t^-)$
\ENSURE Query-time score $s_{\text{final}}(u,v,r,t^-)$; updated stalks and frames

\STATE \textbf{// Step 1: Query-time scoring from history only (relation-conditioned)}
\STATE Apply relation frame bias (Eq.~\eqref{eq:rel_frame_bias}); compute $\tilde{\mathbf{Q}}_{uv}^-(r)$ and $s_{\text{final}}(u,v,r,t^-)$ via Eqs.~\eqref{eq:dsq}--\eqref{eq:edgebank_mixture}
\STATE Form ranking loss against negative destinations using pre-event scores only

\STATE \textbf{// Step 2: Observe the positive event and encode time}
\STATE Compute $\phi(\delta t_u)$, $\phi(\delta t_v)$ via sinusoidal encoding (Eq.~\eqref{eq:time_encoding})

\STATE \textbf{// Step 3: Update frames using transported pre-event states}
\STATE $\Delta \mathbf{F}_u, \Delta \mathbf{F}_v \leftarrow$ Eq.~\eqref{eq:frame_update}; update $\mathbf{F}_u(t), \mathbf{F}_v(t)$
\STATE Rebuild $\mathbf{U}_u(t), \mathbf{U}_v(t)$ using the Householder product from \S\ref{sec:sheaf_construction}
\STATE $\bar{\mathbf{h}}_u^- \leftarrow \mathbf{U}_u(t)^\top \mathbf{U}_u(t^-)\mathbf{h}_u(t^-)$;  $\bar{\mathbf{h}}_v^- \leftarrow \mathbf{U}_v(t)^\top \mathbf{U}_v(t^-)\mathbf{h}_v(t^-)$ \hfill (Prop.~\ref{prop:carry_over})

\STATE \textbf{// Step 4: Message computation and endpoint update}
\STATE $\mathbf{m}_u, \mathbf{m}_v \leftarrow$ Eq.~\eqref{eq:message};  $\mathbf{h}_u(t), \mathbf{h}_v(t) \leftarrow$ Eq.~\eqref{eq:gru_update}

\STATE \textbf{// Step 5: Full-active sheaf diffusion}
\STATE Build $\mathcal{A}_t$ from the endpoint histories; retrieve $(\mathbf{h}_a^{(0)},\mathbf{F}_a)$ for every $a\in\mathcal{A}_t$, using the post-GRU stalks for $u,v$
\STATE Set $\bar d_a(t) := \max\{|\mathcal{N}_a^{\mathrm{act}}(t)|,1\}$ for every active node $a\in\mathcal{A}_t$
\FOR{$\ell = 0$ to $K-1$}
    \FOR{each active node $a \in \mathcal{A}_t$}
        \STATE $\mathbf{h}_a^{(\ell+1)} \leftarrow \mathbf{h}_a^{(\ell)} - \dfrac{\eta}{\bar d_a(t)}\, \mathbf{D}_\theta \sum_{b \in \mathcal{N}_a^{\mathrm{act}}(t)} \left(\mathbf{h}_a^{(\ell)} - \mathbf{Q}_{ab}(t) \mathbf{h}_b^{(\ell)}\right)$ \hfill (Eq.~\eqref{eq:diffusion_local})
    \ENDFOR
\ENDFOR

\STATE \textbf{// Step 6: Commit active states and endpoint metadata}
\STATE Persist $\{\mathbf{h}_a^{(K)}:a\in\mathcal{A}_t\}$; persist updated endpoint frames $\mathbf{F}_u(t),\mathbf{F}_v(t)$ and endpoint last times; update neighbor buffers $\mathcal{N}_u^-,\mathcal{N}_v^-$ and EdgeBank statistics.
\end{algorithmic}
\end{algorithm}

\begin{figure*}[htbp]
\vspace{-8pt}
\centering
\IfFileExists{figures/tsnn_overview.png}{%
    \includegraphics[width=1.0\textwidth, height=0.42\textheight]{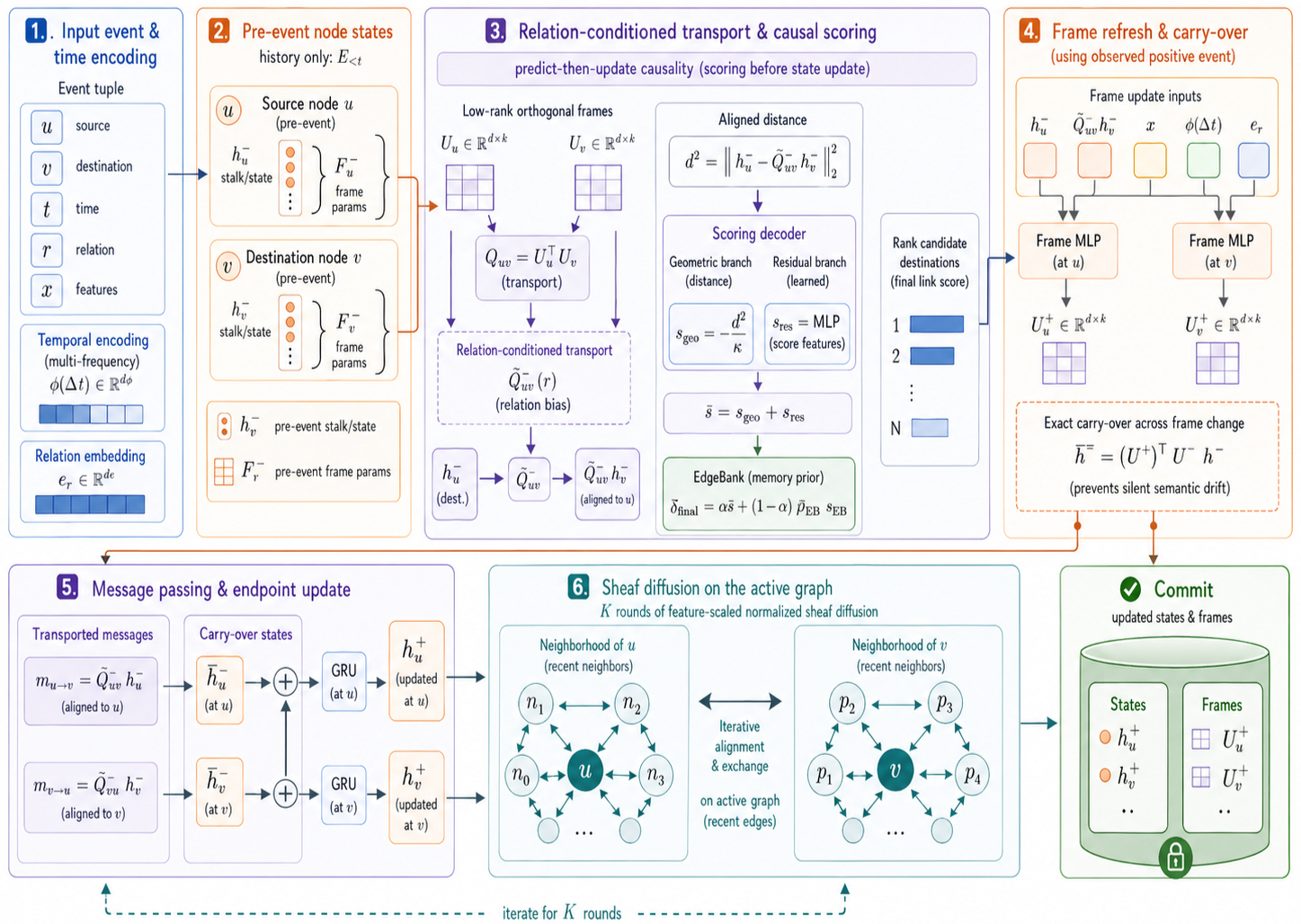}%
}{%
    \fbox{%
        \begin{minipage}[c][0.15\textheight][c]{\textwidth}
        \centering
        {\Large TSNN Overview Figure Placeholder}\\[0.75em]
        Drop your image at:\\[0.4em]
        \texttt{figures/tsnn\_overview.png}\\[0.75em]
        The paper will then automatically render it here at full width.
        \end{minipage}%
    }%
}
\vspace{-8pt}
\caption{Overview of the \tsnn{} pipeline for a single event $(u, v, t, r, \mathbf{x})$. Candidate destinations are ranked using only pre-event history $\mathcal{E}_{<t}$ via the geometric-residual decoder. The frames $\mathbf{U}_u, \mathbf{U}_v$ and stalk states $\mathbf{h}_u, \mathbf{h}_v$ are updated, with the carry-over $\bar{\mathbf{h}}^- := (\mathbf{U}^+)^\top \mathbf{U}^- \mathbf{h}^-$. The event-local active graph $\mathcal{A}_t$ is then smoothed by $K$ rounds of full-active, feature-scaled normalized sheaf diffusion (Eq.~\eqref{eq:diffusion_local}), with a single representative round depicted in the figure. \textbf{Takeaway:} every interaction is scored, updated, and diffused only after states have been transported into a common local frame, so semantic content is preserved as the per-node coordinate systems evolve over time.}
\label{fig:overview}
\end{figure*}

\subsection{Frames, Orthogonal Maps, and Transport}
\label{sec:sheaf_construction}
For every node $v \in \mathcal{V}$, we maintain \emph{frame parameters} $\mathbf{F}_v(t) \in \RR^{k \times d}$ whose $k$ rows are Householder vectors. These vectors parameterize an orthogonal map $\mathbf{U}_v(t) \in \RR^{d \times d}$, which serves as the restriction map in sheaf-theoretic terms. Let $\mathbf{f}_v^{(1)}, \ldots, \mathbf{f}_v^{(k)} \in \RR^d$ denote the rows of $\mathbf{F}_v(t)$. We use the same product convention as Algorithm~\ref{alg:transport}:
\begin{equation}
    \mathbf{U}_v(t)
    :=
    H_\varepsilon(\mathbf{f}_v^{(k)})\cdots H_\varepsilon(\mathbf{f}_v^{(1)}).
    \label{eq:householder_product}
\end{equation}
Thus $\mathbf{U}_v(t)\mathbf{h}$ is implemented by applying the row reflectors in the order $1,\ldots,k$, while $\mathbf{U}_v(t)^\top\mathbf{h}$ is implemented by applying them in the reverse order. Here $H_\varepsilon(\mathbf{f})=\mathbf{I}-2\hat{\mathbf{f}}\hat{\mathbf{f}}^\top$ with $\hat{\mathbf{f}}=\mathbf{f}/\|\mathbf{f}\|$ when $\|\mathbf{f}\|>\varepsilon$, and $H_\varepsilon(\mathbf{f})=\mathbf{I}$ otherwise.
The threshold $\varepsilon$ is a numerical safeguard against division by near-zero $\|\mathbf{f}\|$ in a single reflector. Our random initialization and additive update keep $\min_{w,i}\|\mathbf{f}_w^{(i)}\|$ comfortably above $\varepsilon=10^{-6}$ throughout training, so this threshold is never crossed in the operating regime, and the model analyzed in \S\ref{sec:theory} is the differentiable Householder reflection ($\varepsilon=0$ limit) on $\{\mathbf{f}:\|\mathbf{f}\|>0\}$. A distinct, looser operating threshold $\tau_{\mathrm{H}}\!\gg\!\varepsilon$ (typically $\tau_{\mathrm{H}}\!\approx\!10^{-2}$) governs the Lipschitz constant of the \emph{product} map $\mathbf{U}_v$ in Lemma~\ref{lem:houshlip} and Eq.~\eqref{eq:smooth_lipschitz}; the two thresholds play disjoint roles and are not interchangeable. Because $H_\varepsilon(\mathbf{f})$ is orthogonal for every $\mathbf{f}$, the product $\mathbf{U}_v(t)$ is orthogonal by construction. The resulting frame requires only $\mathcal{O}(kd)$ parameters per node.
The transport aligning $v$ to $u$ is $\mathbf{Q}_{uv}(t)=\mathbf{U}_u(t)^\top\mathbf{U}_v(t)$ (write $\mathbf{Q}_{uv}^\pm$ for the $t^\pm$ variants). Applying $\mathbf{Q}_{uv}(t)$ costs $\mathcal{O}(kd)$ via sequential Householder applications, so we never materialize the full $d \times d$ matrix. After the query-time score for the event has been formed, we update the endpoint frame parameters using transported partner states and raw time gaps $\delta t_w:=t-t_w^{\mathrm{last}}$:
\begin{subequations}\label{eq:frame_update}       
  \begin{align}
  \Delta \mathbf{F}_u &= \MLP_{\text{frame}}\!\big(\mathbf{h}_u^- \concat \mathbf{Q}_{uv}^- \mathbf{h}_v^- \concat \mathbf{x} \concat \phi(\delta t_u) \concat \mathbf{e}_r\big),                           
  \label{eq:frame_update_u}\\[-1pt]             
  \Delta \mathbf{F}_v &= \MLP_{\text{frame}}\!\big(\mathbf{h}_v^- \concat \mathbf{Q}_{vu}^- \mathbf{h}_u^- \concat \mathbf{x} \concat \phi(\delta t_v) \concat \mathbf{e}_r\big). \label{eq:frame_update_v}     
  \end{align}                                                  
  \end{subequations} 
We update the frame state additively:
$\mathbf{F}_u(t) \leftarrow \mathbf{F}_u(t^-) + \Delta \mathbf{F}_u$,
$\mathbf{F}_v(t) \leftarrow \mathbf{F}_v(t^-) + \Delta \mathbf{F}_v$.
\begin{lemma}[Low-rank orthogonal perturbation]
\label{lem:low_rank}
Let $\mathbf{U}_v$ be defined by Eq.~\eqref{eq:householder_product}. Let $m_v := |\{i : \|\mathbf{f}_v^{(i)}\| > \varepsilon\}|$ be the number of non-degenerate Householder vectors. Then $\mathrm{rank}(\mathbf{U}_v - \mathbf{I}) \leq m_v \leq k$. Consequently, for $\mathbf{Q}_{uv} = \mathbf{U}_u^\top \mathbf{U}_v$, $\mathrm{rank}(\mathbf{Q}_{uv} - \mathbf{I}) \leq m_u + m_v \leq 2k$.
\end{lemma}
The proof is deferred to Appendix~\ref{app:proof_low_rank}. Lemma~\ref{lem:low_rank} bounds the \emph{rank} of the orthogonal correction; the complementary \emph{magnitude} bound, provided by Lemma~\ref{lem:houshlip} (Eq.~\eqref{eq:smooth_lipschitz}), controls operator drift under parameter updates. Each learned transport is therefore a rank-$\le 2k$ orthogonal correction to the identity, which structurally explains why a small $k$ ($k\!=\!2$--$4$ in our experiments) suffices in practice.


\subsection{Causal Query-Time Scoring}
\label{sec:scoring}
Step~1 of Algorithm~\ref{alg:tsnn} executes the query-time prediction. To enforce strict causality, the model ranks the true destination against $K_{\text{neg}}$ negative samples using only pre-event frames and stalks, prior to any internal state update. We assemble the prediction into the optional relation-frame bias; then, the geometric-residual decoder; and finally, the EdgeBank mixture that yields the score $s_{\text{final}}$ consumed by the ranking loss.

When relation information is available, we apply a learned relation-specific perturbation to the source frame before computing transport:
\begin{equation}
    \tilde{\mathbf{F}}_u^-(r) = \mathbf{F}_u(t^-) + \sigma(\gamma_r)\,\mathbf{b}_r,
    \label{eq:rel_frame_bias}
\end{equation}
where $\mathbf{b}_r \in \RR^{k \times d}$ is a per-relation frame bias, $\gamma_r$ is a learned scale, and $\sigma$ is the sigmoid function. As our Householder parameterization maps any unconstrained matrix to a valid orthogonal group element, this perturbed frame yields a safely orthogonal relation-conditioned transport $\tilde{\mathbf{Q}}_{uv}^-(r)$. Instead of parameterizing the ranking score as a monolithic neural function, we employ a \emph{geometric-residual} formulation. Let
\begin{subequations}\label{eq:sdist}
\begin{align}
d^2(u,v,t^-) &:= \big\|\mathbf{h}_u^- - \tilde{\mathbf{Q}}_{uv}^-(r)\, \mathbf{h}_v^-\big\|^2, \label{eq:dsq}\\[-1pt]
s_{\text{geo}}(u,v,t^-) &= -\,d^2(u,v,t^-)/\kappa, \label{eq:sgeo}\\[-1pt]
s_{\text{res}}(u,v,r,t^-) &= \MLP_{\text{score}}\!\big(\boldsymbol{\psi}_{uvr}^-\big), \label{eq:sres}
\end{align}
\end{subequations}
where $\kappa > 0$ is a learnable scalar temperature, parameterized as $\kappa=\kappa_{\min}+\softplus(\rho_\kappa)$ with $\kappa_{\min}=10^{-6}$ to ensure strict positivity and numerical stability. The geometric branch $s_{\text{geo}}$ enforces a contrastive inductive bias by favoring nearby transported pairs, while the residual branch $s_{\text{res}}$ corrects for systematic deviations driven by relation type, recency, and node-type metadata. The residual branch consumes scalar invariants of the aligned pair, augmented with relation and node-type embeddings:
\begin{equation}
    \boldsymbol{\psi}_{uvr}^- =
    \big[
    s_{\text{geo}},\;
    \langle \mathbf{h}_u^-, \tilde{\mathbf{Q}}_{uv}^- \mathbf{h}_v^- \rangle,\;
    \|\mathbf{h}_u^-\|^2,\;
    \|\mathbf{h}_v^-\|^2,\;
    \delta t_u,\;
    \delta t_v,\;
    \log(1+d^2),\;
    \mathbf{e}_r,\;
    \mathbf{e}_{\tau(u)},\;
    \mathbf{e}_{\tau(v)}
    \big],
    \label{eq:score_features}
\end{equation}
where $\mathbf{e}_r$ denotes the relation embedding and $\mathbf{e}_{\tau(u)},\mathbf{e}_{\tau(v)}$ are node-type embeddings; here $\tau\!:\!\mathcal{V}\!\to\!\mathcal{T}$ is the type-indexing function over the type set $\mathcal{T}$, distinct from the temperatures $\tau_{\text{geo}}$ (Eq.~\eqref{eq:loss_geo}) and $\tau_{\mathrm{H}}$ (\S\ref{sec:sheaf_construction}). The entries $\delta t_u, \delta t_v$ are \emph{raw} time gaps since the last event at $u$ and $v$; when the dataset-specific toggle $\mathbf{1}[\phi(\delta t)\text{ in }\boldsymbol{\psi}]$ is enabled, the encoded gaps $\phi(\delta t_u),\phi(\delta t_v)$ are additionally concatenated to $\boldsymbol{\psi}_{uvr}^-$. When relation or type information is absent, the corresponding entries are omitted. By design, the distance signal is injected twice: as the global geometric logit $s_{\text{geo}}$ to anchor the overall ranking, and as the scalar feature $\log(1+d^2)$ to allow the residual branch to learn distance-dependent corrections. The full neural score is $\tilde{s}(u,v,r,t^-)=s_{\text{geo}}(u,v,t^-)+s_{\text{res}}(u,v,r,t^-)$.
By construction, both branches are basis-independent: $s_{\text{geo}}$ depends strictly on the transported squared Euclidean distance, which is invariant under any common orthogonal change of frame, while $s_{\text{res}}$ consumes exclusively scalar invariants alongside frame-free metadata embeddings.

Following the insight that memory-based heuristics are surprisingly strong \citep{poursafaei2022edgebank}, we combine the neural score with an EdgeBank score $s_{\text{EB}}(u,v,t^-)$, defined as the recency-weighted count of prior $(u,v)$ interactions in a sliding time window:
\begin{equation}
    s_{\text{final}}(u, v, r, t^-) = \alpha_r \cdot \tilde{s}(u, v, r, t^-) + (1 - \alpha_r) \cdot \beta_{\text{EB}} \cdot s_{\text{EB}}(u, v, t^-),
    \label{eq:edgebank_mixture}
\end{equation}
where $\alpha_r \in [0,1]$ is a per-relation learnable gate (sigmoid-parameterized; shared across all relations when $r$ is absent) and $\beta_{\text{EB}}$ is a global temperature. Evaluated against $K_{\text{neg}}$ negative destinations, these scores yield the ranking loss defined in \S\ref{sec:training}; that loss is committed \emph{before} any state update is allowed to proceed.

\subsection{Event Processing and State Update}
\label{sec:event_processing}
Once the ranking loss has been committed, we observe the positive event $(u,v,t,r,\mathbf{x})$ and compute the raw time gap $\delta t_w:=t-t_w^{\text{last}}$ and its encoding $\phi(\delta t_w)$ in Step~2 of Algorithm~\ref{alg:tsnn} using sinusoidal features \citep{xu2020tgat}. We first log-compress nonnegative time gaps,
$\widetilde{\delta t}
    :=
    \log\!\left(1+\max(\delta t,0)\right),$ and then apply sinusoidal features:
\begin{equation}
    \phi(\delta t)_j =
    \begin{cases}
        \sin(\omega_j \widetilde{\delta t}) & j \text{ even}, \\
        \cos(\omega_j \widetilde{\delta t}) & j \text{ odd},
    \end{cases}
    \qquad
    \omega_j = 10000^{-2\lfloor j/2 \rfloor / d_t}.
    \label{eq:time_encoding}
\end{equation}
The logarithmic compression prevents large timestamp gaps from dominating the periodic encoding while preserving ordering and recency information. Step~3 applies the frame-update rule from
\S\ref{sec:sheaf_construction}. The update
uses transported pre-event partner states to compute
$\Delta\mathbf{F}_u$ and $\Delta\mathbf{F}_v$, adds these increments to
the current frame parameters, and rebuilds the corresponding orthogonal
maps via the Householder product in \S\ref{sec:sheaf_construction}. We write
$\mathbf{U}_w^- := \mathbf{U}_w(t^-)$ and
$\mathbf{U}_w^+ := \mathbf{U}_w(t)$ for the pre- and post-update frames
of endpoint $w$. The pre-event stalks are then expressed in the new local
coordinates.

\begin{proposition}[Coordinate-consistent state carry-over]
\label{prop:carry_over}
For each endpoint $w \in \{u,v\}$, define $\bar{\mathbf{h}}_w^-
    :=
    (\mathbf{U}_w^+)^\top \mathbf{U}_w^- \mathbf{h}_w(t^-). $
Then $\mathbf{U}_w^+ \bar{\mathbf{h}}_w^-
    =
    \mathbf{U}_w^- \mathbf{h}_w(t^-).$
Moreover, $\bar{\mathbf{h}}_w^-$ is the unique vector in the
post-update local coordinates whose representation in the common
synchronized frame equals the pre-event represented state. Thus, a
change of local frame does not by itself alter the semantic content of
the stored stalk.
\end{proposition}

Step~4 computes endpoint messages in the post-update local coordinates.
Let
$\mathbf{Q}_{uv}^+ := (\mathbf{U}_u^+)^\top \mathbf{U}_v^+$
denote the post-update transport. The message to $u$ is
\begin{equation}
    \mathbf{m}_u
    =
    \MLP_{\mathrm{msg}}\!\left(
        \bar{\mathbf{h}}_u^-
        \concat
        \mathbf{Q}_{uv}^+ \bar{\mathbf{h}}_v^-
        \concat
        \mathbf{x}
        \concat
        \phi(\delta t_u)
        \concat
        \mathbf{e}_r
    \right),
    \label{eq:message}
\end{equation}
and the message to $v$ is defined symmetrically using
$\mathbf{Q}_{vu}^+ = (\mathbf{Q}_{uv}^+)^\top$. The endpoint stalks are refreshed by a GRU cell ($\mathbf{h}_v(t)$ symmetric):
\begin{equation}
    \mathbf{h}_u(t)
    =
    \GRU\!\left(\bar{\mathbf{h}}_u^-,\, \mathbf{m}_u\right),
    \label{eq:gru_update}
\end{equation}

\subsection{Sheaf Laplacian Diffusion}
\label{sec:diffusion}

Step~5 of Algorithm~\ref{alg:tsnn} runs $K$ rounds of full-active sheaf Laplacian diffusion over an event-local active graph. Let $\mathcal{N}_u^-(t),\mathcal{N}_v^-(t)$ be the multisets of the $K_{\text{nbr}}$ most recent historical neighbors of $u,v$ in $\mathcal{E}_{<t}$, with active vertex set $\mathcal{A}_t:=\{u,v\}\cup\mathrm{supp}(\mathcal{N}_u^-(t))\cup\mathrm{supp}(\mathcal{N}_v^-(t))$ ($|\mathcal{A}_t|\le 2K_{\text{nbr}}+2$). The active graph $G_t^{\mathrm{act}}$ has a symmetrized edge multiset induced by these interactions, so each temporal edge contributes bidirectional message-passing. Repeated interactions are retained, so the combinatorial sheaf Laplacian's largest eigenvalue can grow with the maximum active degree, which would force a degree-dependent step size; we therefore use a degree-normalized operator. Let $\bar d_a(t):=\max\{|\mathcal{N}_a^{\mathrm{act}}(t)|,1\}$ and $\boldsymbol{\Delta}_t:=\blkdiag((\bar d_a(t)\mathbf{I}_d)_{a\in\mathcal{A}_t})$ denote the safe block-diagonal active-degree matrix. The max convention makes
isolated active nodes well-defined and leaves them unchanged because their
neighbor sums are empty. We define
\begin{equation}
\widetilde{\mathbf{L}}_{\mathcal{F}}(t)
:=
\boldsymbol{\Delta}_t^{-1}\mathbf{L}_{\mathcal{F}}(t),
\qquad
\widetilde{\mathbf{L}}_{\mathcal{F}}^{\mathrm{sym}}(t)
:=
\boldsymbol{\Delta}_t^{-1/2}
\mathbf{L}_{\mathcal{F}}(t)
\boldsymbol{\Delta}_t^{-1/2}.
\label{eq:norm_sheaf_lap}
\end{equation}
Under our node-induced orthogonal transport construction,
$\widetilde{\mathbf{L}}_{\mathcal{F}}^{\mathrm{sym}}(t)$ is
orthogonally similar to the symmetric normalized graph Laplacian on the
active graph, so
$\sigma(\widetilde{\mathbf{L}}_{\mathcal{F}}^{\mathrm{sym}}(t))
\subseteq [0,2]$.

To allow feature-wise control over the amount of smoothing, we introduce
a learnable diagonal gain
$\mathbf{D}_\theta := \diag(d_1,\ldots,d_d)$, where
$d_j = \softplus(\theta_j) > 0$, and lift it to the active-node stacked
state as $\mathbf{P}:=\mathbf{I}_{|\mathcal{A}_t|}\otimes\mathbf{D}_\theta$.
The diagonal gain $\mathbf{D}_\theta$ lets the model learn which stalk
coordinates should be contracted more strongly during diffusion. Setting
$\mathbf{D}_\theta=\mathbf{I}_d$ recovers the normalized-only update.

For active-node ordering $\mathcal{A}_t=\{a_1,\ldots,a_m\}$, stack $\mathbf{z}^{(\ell)}:=[(\mathbf{h}_{a_1}^{(\ell)})^\top,\ldots,(\mathbf{h}_{a_m}^{(\ell)})^\top]^\top\in\mathbb{R}^{md}$, with $m=|\mathcal{A}_t|$. We perform $K$ steps of feature-scaled normalized sheaf diffusion:
\begin{equation}
    \mathbf{z}^{(\ell+1)}
    =
    \mathbf{z}^{(\ell)}
    -
    \eta\,
    \mathbf{P}\,
    \widetilde{\mathbf{L}}_{\mathcal{F}}(t)\,
    \mathbf{z}^{(\ell)},
    \qquad
    \ell = 0,\ldots,K-1.
    \label{eq:diffusion_global}
\end{equation}
Equivalently, the implemented full-active operator updates each active node $a \in \mathcal{A}_t$ by
\begin{equation}
    \mathbf{h}_a^{(\ell+1)}
    =
    \mathbf{h}_a^{(\ell)}
    -
    \frac{\eta}{\bar d_a(t)}
    \mathbf{D}_\theta
    \sum_{b \in \mathcal{N}_a^{\mathrm{act}}(t)}
    \left(
        \mathbf{h}_a^{(\ell)}
        -
        \mathbf{Q}_{ab}(t)\mathbf{h}_b^{(\ell)}
    \right).
    \label{eq:diffusion_local}
\end{equation}
Equations~\eqref{eq:diffusion_global}--\eqref{eq:diffusion_local} are the implemented event-local diffusion operator: all active stalks in $\mathcal{A}_t$ are updated, while only the endpoint frames and endpoint last-time metadata are changed by the current event. This is the same operator analyzed in \S\ref{sec:theory}. The active graph $G_t^{\mathrm{act}}$ is structurally close to a 2-star around endpoints $u,v$ since edges in $E_t^{\mathrm{act}}$ are induced from $u$-historical and $v$-historical interactions; nonetheless, Eq.~\eqref{eq:diffusion_local} jointly updates every active node, including the second-hop neighbors that share an endpoint, so the operator analyzed in \S\ref{sec:theory} is the full $|\mathcal{A}_t|d \times |\mathcal{A}_t|d$ block sheaf Laplacian rather than two independent stars. This update differs from the combinatorial sheaf diffusion step in two
ways. First, the factor
$1/\bar d_a(t)$ implements safe active-degree normalization, keeping the
normalized operator spectrum bounded independently of node degree while
handling isolated active nodes. Second, the feature-wise gain $\mathbf{D}_\theta$ allows the
model to learn anisotropic smoothing strengths across stalk coordinates.
When $\mathbf{D}_\theta=\mathbf{I}_d$, Eq.~\eqref{eq:diffusion_local}
reduces to normalized sheaf diffusion.

\subsection{Training Objective and Scalability}
\label{sec:training}
After diffusion completes, Step~6 of Algorithm~\ref{alg:tsnn} persists
the updated stalks, frames, neighbor buffers, and EdgeBank statistics to the state store, closing the per-event loop. The forward pass described
in \S\ref{sec:scoring}--\S\ref{sec:diffusion} is trained end-to-end with
the following objective. We use a cross-entropy ranking loss over $K_{\text{neg}}$ destination negatives. The sampling distribution is chosen to match each benchmark's native protocol: on Track~A (TGB~v2) we use a type-aware mixture of $50\%$ historical, $25\%$ popularity-weighted, and $25\%$ uniform negatives, consistent with the hard-negative filtered-MRR evaluation; on Track~B (DGB) we draw \emph{uniformly random} destinations only, to remain directly comparable with the random-negative AP/AUC baselines reproduced from \citet{lu2024tpnet}. Full sampler details are given in Appendix~\ref{app:neg_sampling}. The per-positive loss is
\begin{equation}
    \mathcal{L}_{\text{CE}} = -\log \frac{\exp(s_{\text{final}}(u, v^+, r, t^-))}{\exp(s_{\text{final}}(u, v^+, r, t^-)) + \sum_{j=1}^{K_{\text{neg}}} \exp(s_{\text{final}}(u, v_j^-, r, t^-))}.
    \label{eq:loss_ce}
\end{equation}
In addition to the final-score cross-entropy, we use two score-time auxiliary
regularizers that are evaluated on the same pre-event candidate set
$\mathcal{C}_e=\{v^+\}\cup\{v^-_1,\ldots,v^-_{K_{\text{neg}}}\}$.

First, we add an auxiliary geometric contrastive loss on the aligned inner products $a(u,c,r,t^-):=\langle\mathbf{h}_u^-,\tilde{\mathbf{Q}}_{uc}^-(r)\mathbf{h}_c^-\rangle$ for $c\in\mathcal{C}_e$. With temperature $\tau_{\text{geo}}>0$, this loss is
\begin{equation}
    \mathcal{L}_{\text{geo}}
    =
    -\log
    \frac{
        \exp(a(u,v^+,r,t^-)/\tau_{\text{geo}})
    }{
        \sum_{c\in\mathcal{C}_e}
        \exp(a(u,c,r,t^-)/\tau_{\text{geo}})
    } .
    \label{eq:loss_geo}
\end{equation}
This auxiliary term gives direct gradient signal to the transport geometry
even when the residual scorer or EdgeBank mixture can partially explain the
ranking.

Second, to keep the residual branch from overwhelming the geometric branch, we penalize the squared residual logits over the event candidate set; we additionally apply two sheaf-specific regularizers on post-event states:
\begin{subequations}\label{eq:loss_aux}
\begin{align}
\mathcal{L}_{\text{bias}}   &= \tfrac{1}{|\mathcal{C}_e|}\!\sum_{c\in\mathcal{C}_e}\! s_{\text{res}}(u,c,r,t^-)^2, \label{eq:loss_bias}\\[-1pt]
\mathcal{L}_{\text{smooth}} &= \sum_{e=(u,v,t)}\! \big(\|\Delta \mathbf{F}_u\|_F^2 + \|\Delta \mathbf{F}_v\|_F^2\big), \label{eq:loss_smooth}\\[-1pt]
\mathcal{L}_{\text{energy}} &= \sum_{e=(u,v,t)} \|\mathbf{h}_u^{(K)}(t) - \mathbf{Q}_{uv}^+(t) \mathbf{h}_v^{(K)}(t)\|^2. \label{eq:loss_energy}
\end{align}
\end{subequations}
Here $\Delta \mathbf{F}_u, \Delta \mathbf{F}_v$ are the Householder-vector deltas from Eq.~\eqref{eq:frame_update}. $\mathcal{L}_{\text{smooth}}$ is an $\mathcal{O}(kd)$ parameter-space regularizer that \emph{provably} upper-bounds the squared orthogonal-map discrepancy: under the operating regime $\min_{w,i}\|\mathbf{f}_w^{(i)}\|\ge\tau_{\mathrm{H}}$ (\S\ref{sec:sheaf_construction}), Lemma~\ref{lem:houshlip} gives the explicit Lipschitz bound
\begin{equation}
\big\|\mathbf{I} - \mathbf{U}_w(t)^\top \mathbf{U}_w(t^-)\big\|_F^2
\;\le\; \frac{64\,k}{\tau_{\mathrm{H}}^2}\,\|\Delta \mathbf{F}_w\|_F^2,
\label{eq:smooth_lipschitz}
\end{equation}
so each application of $\mathcal{L}_{\text{smooth}}$ controls the orthogonal-map drift with explicit constant $64k/\tau_{\mathrm{H}}^2$. $\mathcal{L}_{\text{energy}}$ is evaluated on the full active state after the $K$ post-update diffusion iterations so that the diffusion operator receives a direct gradient signal. With explicit positive weights $\lambda_{\text{geo}},\lambda_{\text{bias}},\lambda_s,\lambda_e>0$, the full training objective is $\mathcal{L} = \mathcal{L}_{\text{CE}} + \lambda_{\text{geo}}\mathcal{L}_{\text{geo}} + \lambda_{\text{bias}}\mathcal{L}_{\text{bias}} + \lambda_s\mathcal{L}_{\text{smooth}} + \lambda_e\mathcal{L}_{\text{energy}}$, so that all four auxiliary terms are weighted symmetrically in $\mathcal{L}$.

\section{Theoretical Properties}
\label{sec:theory}
We establish three operator-level guarantees at a fixed time $t$, with the
symmetrized event-local active multigraph
$\mathcal{G}_t=(\mathcal{V},\mathcal{E}_t,w_t)$ of \S\ref{sec:diffusion} and node
frames $\{\mathbf{U}_w(t)\}$ held fixed; repeated interactions contribute
multiplicity to active degrees, and isolated nodes are inert. All statements
and proofs are formalized in Appendix~\ref{app:proofs}.

The node-induced transport
$\mathbf{Q}_{uv}(t)=\mathbf{U}_u(t)^{\!\top}\mathbf{U}_v(t)$ has trivial
holonomy: along any walk $u_0\!\to\!\cdots\!\to\!u_n$, the product telescopes
to $\prod_{i=1}^{n}\mathbf{Q}_{u_{i-1}u_i}(t)=\mathbf{U}_{u_0}(t)^{\!\top}\mathbf{U}_{u_n}(t)$,
so aligned inner products equal those of synchronized coordinates
$\tilde{\mathbf{x}}_w=\mathbf{U}_w(t)\mathbf{x}_w$ (Lemma~\ref{lem:flat_transport},
Proposition~\ref{prop:sync_equiv}). The normalized sheaf Laplacian therefore
admits the lifted orthogonal similarity
$\widetilde{\mathbf{L}}_{\mathcal{F}}^{\mathrm{sym}}(t)=\mathbf{S}^{\!\top}\!\left((\boldsymbol{\Pi}_{G_t}-\mathbf{D}_t^{-1/2}\mathbf{A}_t\mathbf{D}_t^{-1/2})\otimes\mathbf{I}_d\right)\!\mathbf{S}$,
where $\boldsymbol{\Pi}_{G_t}$ zeros isolated active vertices under our safe-degree convention,
yielding $\sigma(\widetilde{\mathbf{L}}_{\mathcal{F}}^{\mathrm{sym}}(t))\subseteq[0,2]$
and $\rho(\widetilde{\mathbf{L}}_{\mathcal{F}}(t))\le2$, uniformly in
$\Delta_{\max}(\mathcal{G}_t)$ (Corollary~\ref{cor:spectral_equivalence}).

For the stacked active state $\mathbf{z}$, let
$\mathcal{E}_t(\mathbf{z})=\tfrac12\mathbf{z}^{\top}\mathbf{L}_{\mathcal{F}}(t)\mathbf{z}$ be the combinatorial sheaf Dirichlet energy and
$\mathbf{M}_t:=\mathbf{D}_G\otimes\mathbf{D}_\theta^{-1}$ the degree/feature metric. The full active-graph update in Eq.~\eqref{eq:diffusion_global} is the $\mathbf{M}_t$-gradient step
$\mathbf{z}^{+}=\mathbf{z}-\eta\,\mathrm{grad}_{\mathbf{M}_t}\mathcal{E}_t(\mathbf{z})$ and satisfies
\[
\mathcal{E}_t(\mathbf{z}^{+})\le
\mathcal{E}_t(\mathbf{z})-
\eta\!\left(1-\tfrac{\eta}{2}\lambda_{\max}(\mathbf{B}_t)\right)
\|\mathbf{L}_{\mathcal{F}}(t)\mathbf{z}\|_{\mathbf{M}_t^{-1}}^2,
\quad
\mathbf{B}_t:=\mathbf{M}_t^{-1/2}\mathbf{L}_{\mathcal{F}}(t)\mathbf{M}_t^{-1/2}.
\]
Because $\lambda_{\max}(\mathbf{B}_t)\le2\lambda_{\max}(\mathbf{D}_\theta)$, the condition
$0<\eta\le1/\lambda_{\max}(\mathbf{D}_\theta)$ gives degree-free monotone descent and non-expansiveness in the $\mathbf{M}_t$-norm for the implemented full-active operator (Theorem~\ref{thm:feature_scaled}). Firm non-expansiveness is claimed only for the normalized-only case $\mathbf{D}_\theta=\mathbf{I}_d$ with $0<\eta\le\tfrac12$ (Theorem~\ref{thm:nonexpansive}).

At inference, replacing the reference frame $\hat{\mathbf{U}}_w(t)$ (e.g., the always-refreshed $\mathbf{U}_w(t)$) by a stored, possibly stale frame $\tilde{\mathbf{U}}_w(t)$ perturbs the operator linearly in
$\delta_t\!:=\!\max_{w}\|\tilde{\mathbf{U}}_w(t)-\hat{\mathbf{U}}_w(t)\|_2$:
$\|\widetilde{\mathbf{L}}_{\mathcal{F}}^{\mathrm{sym},\mathrm{stored}}(t)-\widetilde{\mathbf{L}}_{\mathcal{F}}^{\mathrm{sym}}(t)\|_2\le 2\delta_t$,
inflating the update Lipschitz constant by at most
$2\eta\lambda_{\max}(\mathbf{D}_\theta)\delta_t$ (Theorem~\ref{thm:stale}).
Lemma~\ref{lem:houshlip} pins this drift to the penalty,
$\delta_t^2\lesssim \tau_{\mathrm{H}}^{-2}\,\mathcal{L}_{\text{smooth}}$, so optimizing
$\mathcal{L}_{\text{smooth}}$ directly tightens the stability constant.

  Trivial holonomy at frozen time makes the diffusion spectrally equivalent to ordinary normalized graph diffusion by design, supplying stable coordinate-consistent smoothing on $\mathcal{G}_t$; \tsnn{}'s    
  expressive gain comes from time-varying frames $\mathbf{U}_w(t)$, exact carry-over $\mathbf{h}_v^-\!\mapsto\!(\mathbf{U}_v^+)^\top\mathbf{U}_v^-\,\mathbf{h}_v^-$, and nonlinear event updates in local   
  coordinates.                                   
\section{Experiments}
\label{sec:experiments}
We evaluate \tsnn{} along two complementary tracks. Track~A targets the Temporal Graph Benchmark under filtered MRR with hard negatives; Track~B targets the DGB benchmarks under AP/AUC with random negatives.

\subsection{Empirical Results}
\label{sec:empirical_results}
\emph{Track~A (TGB)} \citep{huang2023tgb, gastinger2024tgb2} comprises two link-prediction datasets (tgbl-wiki and tgbl-review) and a temporal heterogeneous graph (thgl-software). For this track, we compare against 15 official TGB leaderboard baselines (Table~\ref{tab:tgb_results}). Evaluation strictly follows the official TGB protocol \citep{huang2023tgb}: chronological train/val/test splits (70/15/15), filtered MRR against TGB-provided hard negative destinations, and strict no-future-leakage constraints on the splits.

\begin{table*}[htbp]
\caption{Test MRR (\%) on TGB~v2, reported as mean~$\pm$~std over 5 seeds. Baselines are from the official TGB~v2 leaderboard (\href{https://tgb.complexdatalab.com}{tgb.complexdatalab.com}, accessed May~2026). Best in \textbf{bold}, second best \underline{underlined}; {---}: unavailable on the official leaderboard. Abbreviations: EB$_{\infty}$/EB$_{tw}$~=~EdgeBank (unlimited / time window), Heur~=~Heuristic (LocalGlobal), TGN$_{et}$~=~TGN (edge type), GMixer~=~GraphMixer, DyGFmr~=~DyGFormer, DyGM~=~DyGMamba, HypEv~=~HyperEvent.}
\label{tab:tgb_results}
\centering
\tiny
\setlength{\tabcolsep}{1.2pt}
\resizebox{\textwidth}{!}{%
\begin{tabular}{lcccccccccccccccc}
\toprule
\textbf{Dataset} & \textbf{EB$_{\infty}$} & \textbf{EB$_{tw}$} & \textbf{Heur} & \textbf{TGN} & \textbf{TGN$_{et}$} & \textbf{CAWN} & \textbf{GMixer} & \textbf{NAT} & \textbf{TNCN} & \textbf{DyGFmr} & \textbf{DyGM} & \textbf{CTAN} & \textbf{STHN} & \textbf{HypEv} & \textbf{TPNet} & \textbf{\tsnn{} (Ours)} \\
\midrule
wiki-v2 & 49.50 & 57.10 & 82.10 & 39.60{\scriptsize$\pm$6.00} & {---} & 71.10{\scriptsize$\pm$0.60} & 11.80{\scriptsize$\pm$0.20} & 74.90{\scriptsize$\pm$1.00} & 71.80{\scriptsize$\pm$0.10} & 79.80{\scriptsize$\pm$0.40} & 73.90{\scriptsize$\pm$0.90} & 66.80{\scriptsize$\pm$0.70} & {---} & 81.00{\scriptsize$\pm$0.20} & \underline{82.70{\scriptsize$\pm$0.10}} & \textbf{82.82{\scriptsize$\pm$0.05}} \\
review-v2 & 2.30 & 2.50 & {---} & 34.90{\scriptsize$\pm$2.00} & {---} & 19.30{\scriptsize$\pm$0.10} & \textbf{52.10{\scriptsize$\pm$1.50}} & 34.10{\scriptsize$\pm$2.00} & 37.70{\scriptsize$\pm$1.00} & 22.40{\scriptsize$\pm$1.50} & {---} & 40.50{\scriptsize$\pm$0.40} & {---} & {---} & {---} & \underline{50.28{\scriptsize$\pm$0.22}} \\
thgl-software & 44.90 & 28.80 & {---} & 32.40{\scriptsize$\pm$1.70} & 42.40{\scriptsize$\pm$1.30} & {---} & {---} & {---} & {---} & {---} & {---} & {---} & \textbf{73.10{\scriptsize$\pm$0.50}} & {---} & {---} & \underline{71.13{\scriptsize$\pm$0.36}} \\
\bottomrule
\end{tabular}%
}
\end{table*}

\emph{Track~B (DGB)} \citep{yu2023dygformer} evaluates broad transductive generalization across 13 standard benchmarks: Wikipedia, MOOC, Reddit, LastFM, UCI, Enron, Social~Evo., Can.~Parl., Flights, US~Legis., UN~Vote, UN~Trade, and Contact. We adopt the standard DGB transductive protocol \citep{yu2023dygformer, lu2024tpnet}, utilizing chronological 70/15/15 splits and reporting Average Precision (AP) and Area Under the Curve (AUC) with one random negative per positive at test time. Baseline results for this track are sourced directly from \citet{lu2024tpnet} (Table~\ref{tab:dyglib}).

\begin{table*}[htbp]
\caption{Transductive temporal link prediction on the 13 DGB datasets. AP and AUC (\%) are reported as mean~$\pm$~std over 5 seeds. Best in \textbf{bold}, second best \underline{underlined}. Baseline values are directly taken from \citet{lu2024tpnet}. Avg.~Rank is the mean rank across the 13 datasets.}
\label{tab:dyglib}
\centering
\tiny
\setlength{\tabcolsep}{1.5pt}
\resizebox{\textwidth}{!}{%
\begin{tabular}{llccccccccccccc}
\toprule
\textbf{Metric} & \textbf{Dataset} & \textbf{JODIE} & \textbf{DyRep} & \textbf{TGAT} & \textbf{TGN} & \textbf{CAWN} & \textbf{EdgeBank} & \textbf{TCL} & \textbf{GraphMixer} & \textbf{NAT} & \textbf{PINT} & \textbf{DyGFormer} & \textbf{TPNet} & \textbf{\tsnn{} (Ours)} \\
\midrule
\multirow{14}{*}{\rotatebox[origin=c]{90}{\textbf{AP}}}
 & Wikipedia   & 96.50{\scriptsize$\pm$0.14} & 94.86{\scriptsize$\pm$0.06} & 96.94{\scriptsize$\pm$0.06} & 98.45{\scriptsize$\pm$0.06} & 98.76{\scriptsize$\pm$0.03} & 90.37{\scriptsize$\pm$0.00} & 96.47{\scriptsize$\pm$0.16} & 97.25{\scriptsize$\pm$0.03} & 98.03{\scriptsize$\pm$0.07} & 98.45{\scriptsize$\pm$0.04} & 99.03{\scriptsize$\pm$0.02} & \underline{99.32{\scriptsize$\pm$0.03}} & \textbf{99.41{\scriptsize$\pm$0.02}} \\
 & Reddit      & 98.31{\scriptsize$\pm$0.14} & 98.22{\scriptsize$\pm$0.04} & 98.52{\scriptsize$\pm$0.02} & 98.63{\scriptsize$\pm$0.06} & 99.11{\scriptsize$\pm$0.01} & 94.86{\scriptsize$\pm$0.00} & 97.53{\scriptsize$\pm$0.02} & 97.31{\scriptsize$\pm$0.01} & 99.13{\scriptsize$\pm$0.10} & 99.15{\scriptsize$\pm$0.02} & 99.22{\scriptsize$\pm$0.01} & \underline{99.27{\scriptsize$\pm$0.00}} & \textbf{99.66{\scriptsize$\pm$0.07}} \\
 & MOOC        & 80.23{\scriptsize$\pm$2.44} & 81.97{\scriptsize$\pm$0.49} & 85.84{\scriptsize$\pm$0.15} & 89.15{\scriptsize$\pm$1.60} & 80.15{\scriptsize$\pm$0.25} & 57.97{\scriptsize$\pm$0.00} & 82.38{\scriptsize$\pm$0.24} & 82.78{\scriptsize$\pm$0.15} & 85.88{\scriptsize$\pm$0.55} & 88.08{\scriptsize$\pm$0.86} & 87.52{\scriptsize$\pm$0.49} & \underline{96.39{\scriptsize$\pm$0.09}} & \textbf{99.36{\scriptsize$\pm$0.03}} \\
 & LastFM      & 70.85{\scriptsize$\pm$2.13} & 71.92{\scriptsize$\pm$2.21} & 73.42{\scriptsize$\pm$0.21} & 77.07{\scriptsize$\pm$3.97} & 86.99{\scriptsize$\pm$0.06} & 79.29{\scriptsize$\pm$0.00} & 67.27{\scriptsize$\pm$2.16} & 75.61{\scriptsize$\pm$0.24} & 88.02{\scriptsize$\pm$1.94} & 89.66{\scriptsize$\pm$1.81} & 93.00{\scriptsize$\pm$0.12} & \underline{94.50{\scriptsize$\pm$0.08}} & \textbf{94.51{\scriptsize$\pm$0.03}} \\
 & Enron       & 84.77{\scriptsize$\pm$0.30} & 82.38{\scriptsize$\pm$3.36} & 71.12{\scriptsize$\pm$0.97} & 86.53{\scriptsize$\pm$1.11} & 89.56{\scriptsize$\pm$0.09} & 83.53{\scriptsize$\pm$0.00} & 79.70{\scriptsize$\pm$0.71} & 82.25{\scriptsize$\pm$0.16} & 90.60{\scriptsize$\pm$0.66} & 92.20{\scriptsize$\pm$0.15} & 92.47{\scriptsize$\pm$0.12} & \underline{92.90{\scriptsize$\pm$0.17}} & \textbf{94.81{\scriptsize$\pm$0.19}} \\
 & Social Evo. & 89.89{\scriptsize$\pm$0.55} & 88.87{\scriptsize$\pm$0.30} & 93.16{\scriptsize$\pm$0.17} & 93.57{\scriptsize$\pm$0.17} & 84.96{\scriptsize$\pm$0.09} & 74.95{\scriptsize$\pm$0.00} & 93.13{\scriptsize$\pm$0.16} & 93.37{\scriptsize$\pm$0.07} & 88.92{\scriptsize$\pm$3.45} & 94.42{\scriptsize$\pm$0.03} & 94.73{\scriptsize$\pm$0.01} & \underline{94.73{\scriptsize$\pm$0.02}} & \textbf{95.68{\scriptsize$\pm$0.14}} \\
 & UCI         & 89.43{\scriptsize$\pm$1.09} & 65.14{\scriptsize$\pm$2.30} & 79.63{\scriptsize$\pm$0.70} & 92.34{\scriptsize$\pm$1.04} & 95.18{\scriptsize$\pm$0.06} & 76.20{\scriptsize$\pm$0.00} & 89.57{\scriptsize$\pm$1.63} & 93.25{\scriptsize$\pm$0.57} & 93.40{\scriptsize$\pm$0.26} & 96.45{\scriptsize$\pm$0.11} & 95.79{\scriptsize$\pm$0.17} & \textbf{97.35{\scriptsize$\pm$0.04}} &  \underline{97.06{\scriptsize$\pm$0.02}} \\
 & Flights     & 95.60{\scriptsize$\pm$1.73} & 95.29{\scriptsize$\pm$0.72} & 94.03{\scriptsize$\pm$0.18} & 97.95{\scriptsize$\pm$0.14} & 98.51{\scriptsize$\pm$0.01} & 89.35{\scriptsize$\pm$0.00} & 91.23{\scriptsize$\pm$0.02} & 90.99{\scriptsize$\pm$0.05} & 98.57{\scriptsize$\pm$0.12} & 98.80{\scriptsize$\pm$0.02} & \underline{98.91{\scriptsize$\pm$0.01}} & \textbf{98.93{\scriptsize$\pm$0.02}} & 98.83{\scriptsize$\pm$0.01} \\
 & Can.~Parl.  & 69.26{\scriptsize$\pm$0.31} & 66.54{\scriptsize$\pm$2.76} & 70.73{\scriptsize$\pm$0.72} & 70.88{\scriptsize$\pm$2.34} & 69.82{\scriptsize$\pm$2.34} & 64.55{\scriptsize$\pm$0.00} & 68.67{\scriptsize$\pm$2.67} & 77.04{\scriptsize$\pm$0.46} & 79.72{\scriptsize$\pm$1.76} & 68.36{\scriptsize$\pm$1.43} & \underline{97.36{\scriptsize$\pm$0.45}} & 90.28{\scriptsize$\pm$0.37} & \textbf{97.89{\scriptsize$\pm$0.04}} \\
 & US~Legis.   & 75.05{\scriptsize$\pm$1.52} & 75.34{\scriptsize$\pm$0.39} & 68.52{\scriptsize$\pm$3.16} & 75.99{\scriptsize$\pm$0.58} & 70.58{\scriptsize$\pm$0.48} & 58.39{\scriptsize$\pm$0.00} & 69.59{\scriptsize$\pm$0.48} & 70.74{\scriptsize$\pm$1.02} & 78.71{\scriptsize$\pm$0.87} & 74.85{\scriptsize$\pm$0.97} & 71.11{\scriptsize$\pm$0.59} & \underline{80.58{\scriptsize$\pm$0.23}} & \textbf{87.08{\scriptsize$\pm$0.02}} \\
 & UN~Trade    & 64.94{\scriptsize$\pm$0.31} & 63.21{\scriptsize$\pm$0.93} & 61.47{\scriptsize$\pm$0.18} & 65.03{\scriptsize$\pm$1.37} & 65.39{\scriptsize$\pm$0.12} & 60.41{\scriptsize$\pm$0.00} & 62.21{\scriptsize$\pm$0.03} & 62.61{\scriptsize$\pm$0.27} & 73.95{\scriptsize$\pm$1.16} & 70.20{\scriptsize$\pm$0.58} & 66.46{\scriptsize$\pm$1.29} & \underline{87.24{\scriptsize$\pm$0.65}} & \textbf{87.26{\scriptsize$\pm$0.05}}\\
 & UN~Vote     & 63.91{\scriptsize$\pm$0.81} & 62.81{\scriptsize$\pm$0.80} & 52.21{\scriptsize$\pm$0.98} & 65.72{\scriptsize$\pm$2.17} & 52.84{\scriptsize$\pm$0.10} & 58.49{\scriptsize$\pm$0.00} & 51.90{\scriptsize$\pm$0.30} & 52.11{\scriptsize$\pm$0.16} & 70.45{\scriptsize$\pm$0.68} & 66.25{\scriptsize$\pm$0.78} & 55.55{\scriptsize$\pm$0.42} & \underline{75.12{\scriptsize$\pm$0.29}} & \textbf{75.69{\scriptsize$\pm$0.25}} \\
 & Contact     & 95.31{\scriptsize$\pm$1.33} & 95.98{\scriptsize$\pm$0.15} & 96.28{\scriptsize$\pm$0.09} & 96.89{\scriptsize$\pm$0.56} & 90.26{\scriptsize$\pm$0.28} & 92.58{\scriptsize$\pm$0.00} & 92.44{\scriptsize$\pm$0.12} & 91.92{\scriptsize$\pm$0.03} & 97.39{\scriptsize$\pm$0.22} & 98.64{\scriptsize$\pm$0.02} & 98.29{\scriptsize$\pm$0.01} & \underline{98.66{\scriptsize$\pm$0.01}} & \textbf{99.04{\scriptsize$\pm$0.01}} \\
\cmidrule(lr){2-15}
  & \textit{Avg.~Rank} 
& 8.85 & 9.77 & 9.54 & 6.04 & 8.00 & 11.54 & 10.77 & 9.31 & 5.15 & 4.73 & 4.12 & \underline{1.96} & \textbf{1.23} \\
\midrule
\multirow{14}{*}{\rotatebox[origin=c]{90}{\textbf{AUC}}}
 & Wikipedia   & 96.33{\scriptsize$\pm$0.07} & 94.37{\scriptsize$\pm$0.09} & 96.67{\scriptsize$\pm$0.07} & 98.37{\scriptsize$\pm$0.07} & 98.54{\scriptsize$\pm$0.04} & 90.78{\scriptsize$\pm$0.00} & 95.84{\scriptsize$\pm$0.18} & 96.92{\scriptsize$\pm$0.03} & 97.75{\scriptsize$\pm$0.11} & 98.16{\scriptsize$\pm$0.06} & 98.91{\scriptsize$\pm$0.02} & \underline{99.30{\scriptsize$\pm$0.02}} & \textbf{99.31{\scriptsize$\pm$0.02}} \\
 & Reddit      & 98.31{\scriptsize$\pm$0.05} & 98.17{\scriptsize$\pm$0.05} & 98.47{\scriptsize$\pm$0.02} & 98.60{\scriptsize$\pm$0.06} & 99.01{\scriptsize$\pm$0.01} & 95.37{\scriptsize$\pm$0.00} & 97.42{\scriptsize$\pm$0.02} & 97.17{\scriptsize$\pm$0.02} & 99.09{\scriptsize$\pm$0.10} & 99.09{\scriptsize$\pm$0.03} & 99.15{\scriptsize$\pm$0.01} & \underline{99.22{\scriptsize$\pm$0.00}} & \textbf{99.59{\scriptsize$\pm$0.10}} \\
 & MOOC        & 83.81{\scriptsize$\pm$2.09} & 85.03{\scriptsize$\pm$0.58} & 87.11{\scriptsize$\pm$0.19} & 91.21{\scriptsize$\pm$1.15} & 80.38{\scriptsize$\pm$0.26} & 60.86{\scriptsize$\pm$0.00} & 83.12{\scriptsize$\pm$0.18} & 84.01{\scriptsize$\pm$0.17} & 87.42{\scriptsize$\pm$0.58} & 90.55{\scriptsize$\pm$0.43} & 87.91{\scriptsize$\pm$0.58} & \underline{97.17{\scriptsize$\pm$0.08}} & \textbf{99.49{\scriptsize$\pm$0.01}} \\
 & LastFM      & 70.49{\scriptsize$\pm$1.66} & 71.16{\scriptsize$\pm$1.89} & 71.59{\scriptsize$\pm$0.18} & 78.47{\scriptsize$\pm$2.94} & 85.92{\scriptsize$\pm$0.10} & 83.77{\scriptsize$\pm$0.00} & 64.06{\scriptsize$\pm$1.16} & 73.53{\scriptsize$\pm$0.12} & 86.92{\scriptsize$\pm$2.72} & 89.28{\scriptsize$\pm$1.63} & 93.05{\scriptsize$\pm$0.10} & \underline{94.39{\scriptsize$\pm$0.04}} & \textbf{95.04{\scriptsize$\pm$0.13}} \\
 & Enron       & 87.96{\scriptsize$\pm$0.52} & 84.89{\scriptsize$\pm$3.00} & 68.89{\scriptsize$\pm$1.10} & 88.32{\scriptsize$\pm$0.99} & 90.45{\scriptsize$\pm$0.14} & 87.05{\scriptsize$\pm$0.00} & 75.74{\scriptsize$\pm$0.72} & 84.38{\scriptsize$\pm$0.21} & 91.68{\scriptsize$\pm$0.83} & 92.87{\scriptsize$\pm$0.34} & 93.33{\scriptsize$\pm$0.13} & \underline{93.98{\scriptsize$\pm$0.26}} & \textbf{95.72{\scriptsize$\pm$0.17}} \\
 & Social Evo. & 92.05{\scriptsize$\pm$0.46} & 90.76{\scriptsize$\pm$0.21} & 94.76{\scriptsize$\pm$0.16} & 95.39{\scriptsize$\pm$0.17} & 87.34{\scriptsize$\pm$0.08} & 81.60{\scriptsize$\pm$0.00} & 94.84{\scriptsize$\pm$0.17} & 95.23{\scriptsize$\pm$0.07} & 90.84{\scriptsize$\pm$3.72} & 96.16{\scriptsize$\pm$0.02} & 96.30{\scriptsize$\pm$0.01} & \underline{96.43{\scriptsize$\pm$0.02}} & \textbf{96.88{\scriptsize$\pm$0.03}} \\
 & UCI         & 90.44{\scriptsize$\pm$0.49} & 68.77{\scriptsize$\pm$2.34} & 78.53{\scriptsize$\pm$0.74} & 92.03{\scriptsize$\pm$1.13} & 93.87{\scriptsize$\pm$0.08} & 77.30{\scriptsize$\pm$0.00} & 87.82{\scriptsize$\pm$1.36} & 91.81{\scriptsize$\pm$0.67} & 92.31{\scriptsize$\pm$0.37} & 95.57{\scriptsize$\pm$0.16} & 94.49{\scriptsize$\pm$0.26} & \textbf{96.79{\scriptsize$\pm$0.05}} & \underline{96.23{\scriptsize$\pm$0.05}} \\
 & Flights     & 96.21{\scriptsize$\pm$1.42} & 95.95{\scriptsize$\pm$0.62} & 94.13{\scriptsize$\pm$0.17} & 98.22{\scriptsize$\pm$0.13} & 98.45{\scriptsize$\pm$0.01} & 90.23{\scriptsize$\pm$0.00} & 91.21{\scriptsize$\pm$0.02} & 91.13{\scriptsize$\pm$0.01} & 98.69{\scriptsize$\pm$0.10} & 98.89{\scriptsize$\pm$0.02} & \underline{98.93{\scriptsize$\pm$0.01}} & \textbf{99.00{\scriptsize$\pm$0.02}} & 98.91{\scriptsize$\pm$0.03} \\
 & Can.~Parl.  & 78.21{\scriptsize$\pm$0.23} & 73.35{\scriptsize$\pm$3.67} & 75.69{\scriptsize$\pm$0.78} & 76.99{\scriptsize$\pm$1.80} & 75.70{\scriptsize$\pm$3.27} & 64.14{\scriptsize$\pm$0.00} & 72.46{\scriptsize$\pm$3.23} & 83.17{\scriptsize$\pm$0.53} & 84.04{\scriptsize$\pm$1.13} & 77.96{\scriptsize$\pm$1.46} & \underline{97.76{\scriptsize$\pm$0.41}} & 92.05{\scriptsize$\pm$0.34} & \textbf{97.80{\scriptsize$\pm$0.03}} \\
 & US~Legis.   & 82.85{\scriptsize$\pm$1.07} & 82.28{\scriptsize$\pm$0.32} & 75.84{\scriptsize$\pm$1.99} & 83.34{\scriptsize$\pm$0.43} & 77.16{\scriptsize$\pm$0.39} & 62.57{\scriptsize$\pm$0.00} & 76.27{\scriptsize$\pm$0.63} & 76.96{\scriptsize$\pm$0.79} & 85.36{\scriptsize$\pm$0.52} & 82.10{\scriptsize$\pm$0.85} & 77.90{\scriptsize$\pm$0.58} & \underline{86.49{\scriptsize$\pm$0.18}} & \textbf{92.10{\scriptsize$\pm$0.05}} \\
 & UN~Trade    & 69.62{\scriptsize$\pm$0.44} & 67.44{\scriptsize$\pm$0.83} & 64.01{\scriptsize$\pm$0.12} & 69.10{\scriptsize$\pm$1.67} & 68.54{\scriptsize$\pm$0.18} & 66.75{\scriptsize$\pm$0.00} & 64.72{\scriptsize$\pm$0.05} & 65.52{\scriptsize$\pm$0.51} & 77.61{\scriptsize$\pm$1.36} & 74.87{\scriptsize$\pm$0.53} & 70.20{\scriptsize$\pm$1.44} & \underline{89.17{\scriptsize$\pm$0.46}} & \textbf{89.36{\scriptsize$\pm$0.17}} \\
 & UN~Vote     & 68.53{\scriptsize$\pm$0.95} & 67.18{\scriptsize$\pm$1.04} & 52.83{\scriptsize$\pm$1.12} & 69.71{\scriptsize$\pm$2.65} & 53.09{\scriptsize$\pm$0.22} & 62.97{\scriptsize$\pm$0.00} & 51.88{\scriptsize$\pm$0.36} & 52.46{\scriptsize$\pm$0.27} & 75.32{\scriptsize$\pm$0.63} & 70.69{\scriptsize$\pm$1.02} & 57.12{\scriptsize$\pm$0.62} & \underline{79.88{\scriptsize$\pm$0.30}} & \textbf{81.34{\scriptsize$\pm$0.15}} \\
 & Contact     & 96.66{\scriptsize$\pm$0.89} & 96.48{\scriptsize$\pm$0.14} & 96.95{\scriptsize$\pm$0.08} & 97.54{\scriptsize$\pm$0.35} & 89.99{\scriptsize$\pm$0.34} & 94.34{\scriptsize$\pm$0.00} & 94.15{\scriptsize$\pm$0.09} & 93.94{\scriptsize$\pm$0.02} & 97.79{\scriptsize$\pm$0.16} & 98.90{\scriptsize$\pm$0.02} & 98.53{\scriptsize$\pm$0.01} & \underline{98.91{\scriptsize$\pm$0.01}} & \textbf{99.17{\scriptsize$\pm$0.01}} \\
\cmidrule(lr){2-15}
  & \textit{Avg.~Rank} 
& 8.15 & 9.69 & 9.92 & 6.08 & 8.15 & 11.31 & 11.15 & 9.62 & 5.12 & 4.50 & 4.15 & \underline{1.92} & \textbf{1.23} \\
\bottomrule
\end{tabular}%
}
\end{table*}

\tsnn{} establishes a new state of the art on the DGB suite, attaining the lowest average rank across both AP and AUC and delivering the largest absolute gains on the most challenging benchmarks (e.g., $+2.97$ AP on MOOC and $+6.50$ AP on US~Legis.\ over TPNet), while also surpassing DyGFormer on Can.~Parl.\ by $+0.53$ AP. Following \citet{demsar2006}, a Wilcoxon signed-rank test on the 13 dataset-level scores confirms that \tsnn{} significantly outperforms TPNet on both AP and AUC ($p<0.01$), and a Friedman test across all 13 methods rejects the null of equal performance, with pairwise post-hoc comparisons consistent with the average-rank ordering.

The per-event cost is $\mathcal{O}\bigl((K_{\text{cand}} + K\,K_{\text{nbr}})\,k\,d + d^2\bigr)$ ($K_{\text{cand}}=K_{\text{neg}}+1$ candidates per event) via $\mathcal{O}(kd)$ Householder transports, yielding per-epoch wall-clock that scales near-linearly in $|\mathcal{E}|$ and peak VRAM independent of $|\mathcal{V}|$; full per-event, per-epoch, and per-pool memory derivations are deferred to Appendix~\ref{app:complexity}.

All reported experiments were conducted on a single NVIDIA RTX PRO 6000 Blackwell Workstation Edition (96\,GB GDDR7 with ECC). Dataset statistics for both tracks and the hyperparameter settings used in our experiments are described in Appendix~\ref{app:hyperparams}.

\subsection{Ablation Study}
\label{sec:ablations}
We ablate \tsnn{} on USLegis, CanParl, and UCI. For core ablations, EdgeBank is disabled, giving \textsc{TSNN-core}; each variant measures the learned geometric mechanism rather than heuristic recurrence. We test six targeted removals: the residual scorer $s_{\mathrm{res}}$, orthogonal transport ($\mathbf{Q}_{uv}=\mathbf{I}$), frame updates ($\Delta\mathbf{F}=0$), coordinate carry-over, frame smoothness, and sheaf diffusion ($K=0$). Table~\ref{tab:core-ablation} shows that the residual scorer is the dominant component, giving the largest AP drop on all datasets. Removing transport and fixing frames also consistently degrade performance, validating the need for explicit alignment and time-varying local frames. Carry-over, smoothness, and diffusion give smaller gains, supporting their roles as state-consistency, frame-stability, and local-refinement mechanisms. Overall, the ablations identify three primary empirical drivers of \tsnn{}: residual correction over transported geometric scores, explicit orthogonal transport, and time-varying node-local frames.

\begin{table}[htbp]
\centering
\scriptsize
\setlength{\tabcolsep}{2pt}
\renewcommand{\arraystretch}{0.9}
\caption{Core architectural ablations. Entries are test AP (\%), reported as mean $\pm$ std over 3 seeds; parentheses show absolute change in percentage points from TSNN-core.}
\label{tab:core-ablation}
\resizebox{\textwidth}{!}{
\begin{tabular}{lccccccc}
\toprule
Dataset & TSNN-core & w/o residual scorer & w/o transport & Fixed frames & w/o carryover & w/o smooth & w/o diffusion \\
\midrule
USLegis & 81.90 $\pm$ 0.07 & 67.11 $\pm$ 0.49 {\scriptsize ($-14.79$)} & 74.86 $\pm$ 0.22 {\scriptsize ($-7.04$)} & 78.92 $\pm$ 0.10 {\scriptsize ($-2.98$)} & 80.64 $\pm$ 0.08 {\scriptsize ($-1.26$)} & 80.21 $\pm$ 0.06 {\scriptsize ($-1.69$)} & 80.45 $\pm$ 0.11 {\scriptsize ($-1.45$)} \\
CanParl & 97.44 $\pm$ 0.23 & 80.23 $\pm$ 0.35 {\scriptsize ($-17.21$)} & 95.25 $\pm$ 0.10 {\scriptsize ($-2.19$)} & 95.19 $\pm$ 0.14 {\scriptsize ($-2.25$)} & 96.97 $\pm$ 0.05 {\scriptsize ($-0.47$)} & 94.70 $\pm$ 0.19 {\scriptsize ($-2.74$)} & 97.01 $\pm$ 0.09 {\scriptsize ($-0.43$)} \\
UCI     & 95.05 $\pm$ 0.05 & 76.26 $\pm$ 0.51 {\scriptsize ($-18.79$)} & 94.45 $\pm$ 0.07 {\scriptsize ($-0.60$)} & 93.84 $\pm$ 0.12 {\scriptsize ($-1.21$)} & 94.56 $\pm$ 0.03 {\scriptsize ($-0.49$)} & 94.78 $\pm$ 0.08 {\scriptsize ($-0.27$)} & 94.42 $\pm$ 0.10 {\scriptsize ($-0.63$)} \\
\bottomrule
\end{tabular}
}
\end{table}

\section{Conclusion}
\label{sec:conclusion}
We introduced \tsnn{}, a causal temporal link-prediction model that supplants the single shared embedding space with time-varying node-local orthogonal frames, explicit transport, coordinate-consistent state evolution, and basis-independent decoding. The node-induced transport $\mathbf{Q}_{uv}(t)=\mathbf{U}_u(t)^\top\mathbf{U}_v(t)$ is flat: its trivial holonomy makes the symmetric degree-normalized sheaf Laplacian orthogonally similar to the symmetric normalized graph Laplacian, with the random-walk form similar in the degree metric. The full-active feature-scaled sheaf diffusion operator is therefore an explicit metric-gradient step on the combinatorial sheaf Dirichlet energy, with monotone descent and non-expansiveness under $\eta \le 1/\lambda_{\max}(\mathbf{D}_\theta)$. Empirical gains arise not from a richer diffusion spectrum but from dynamic local frames, exact carry-over across frame changes, full-active event-local diffusion, and nonlinear updates in local coordinates.

\tsnn{}'s guarantees rest on a structured transport class: node-induced $\mathbf{Q}_{uv}(t)=\mathbf{U}_u(t)^\top\mathbf{U}_v(t)$ yields synchronized-coordinate equivalence, degree-free spectral control, and monotone energy descent, but also fixes the operator's scope: curved sheaves and non-orthogonal effects (anisotropic scaling, shear) lie outside the analyzed regime; relation- or edge-specific transports would extend expressivity at the cost of a new stability analysis. \tsnn{} is basis-consistent rather than fully gauge-equivariant: transport, carry-over, and scalar decoding respect synchronized orthogonal changes, whereas the frame update, message MLP, and GRU act in chosen local coordinates. Finally, $\mathbf{D}_\theta$ is diagonal and diffusion is event-local, so dense cross-channel coupling or very long-range structure is a natural target for complementary long-range encoders.


\bibliographystyle{plainnat}
\bibliography{main}

@inproceedings{barbero2022sheaf,
  title={Sheaf Neural Networks with Connection {L}aplacians},
  author={Barbero, Federico and Bodnar, Cristian and S{\'a}ez de Oc{\'a}riz Borde, Haitz and Bronstein, Michael and Veli{\v{c}}kovi{\'c}, Petar and Li{\`o}, Pietro},
  booktitle={Proceedings of Topological, Algebraic, and Geometric Learning Workshops 2022},
  pages={28--36},
  year={2022},
  volume={196},
  series={Proceedings of Machine Learning Research},
  publisher={PMLR}
}

@inproceedings{bodnar2022neural,
  title={Neural Sheaf Diffusion: A Topological Perspective on Heterophily and Oversmoothing in Graph Neural Networks},
  author={Bodnar, Cristian and Di Giovanni, Francesco and Chamberlain, Benjamin Paul and Bronstein, Michael M. and Li{\`o}, Pietro},
  booktitle={Advances in Neural Information Processing Systems (NeurIPS)},
  year={2022}
}

@inproceedings{gastinger2024tgb2,
  title={{TGB} 2.0: A Benchmark for Learning on Temporal Knowledge Graphs and Heterogeneous Graphs},
  author={Gastinger, Julia and Huang, Shenyang and Galkin, Mikhail and Loghmani, Erfan and Parviz, Ali and Poursafaei, Farimah and Danovitch, Jacob and Rossi, Emanuele and Koutis, Ioannis and Stuckenschmidt, Heiner and Rabbany, Reihaneh and Rabusseau, Guillaume},
  booktitle={Advances in Neural Information Processing Systems (NeurIPS)},
  year={2024}
}

@inproceedings{hansen2020sheaf,
  title={Sheaf Neural Networks},
  author={Hansen, Jakob and Gebhart, Thomas},
  booktitle={NeurIPS 2020 Workshop on Topological Data Analysis and Beyond},
  year={2020}
}

@article{hansen2019spectral,
  title={Toward a Spectral Theory of Cellular Sheaves},
  author={Hansen, Jakob and Ghrist, Robert},
  journal={Journal of Applied and Computational Topology},
  volume={3},
  number={4},
  pages={315--358},
  year={2019}
}

@inproceedings{huang2023tgb,
  title={Temporal Graph Benchmark for Machine Learning on Temporal Graphs},
  author={Huang, Shenyang and Poursafaei, Farimah and Danovitch, Jacob and Fey, Matthias and Hu, Weihua and Rossi, Emanuele and Leskovec, Jure and Bronstein, Michael and Rabusseau, Guillaume and Rabbany, Reihaneh},
  booktitle={Advances in Neural Information Processing Systems (NeurIPS)},
  year={2023}
}

@inproceedings{kumar2019jodie,
  title={Predicting Dynamic Embedding Trajectory in Temporal Interaction Networks},
  author={Kumar, Srijan and Zhang, Xikun and Leskovec, Jure},
  booktitle={ACM SIGKDD International Conference on Knowledge Discovery \& Data Mining (KDD)},
  year={2019}
}

@inproceedings{rossi2020tgn,
  title={Temporal Graph Networks for Deep Learning on Dynamic Graphs},
  author={Rossi, Emanuele and Chamberlain, Benjamin Paul and Frasca, Fabrizio and Eynard, Davide and Monti, Federico and Bronstein, Michael M.},
  booktitle={ICML Workshop on Graph Representation Learning},
  year={2020}
}

@inproceedings{souza2022pint,
  title={Provably Expressive Temporal Graph Networks},
  author={Souza, Amauri H. and Mesquita, Diego and Kaski, Samuel and Garg, Vikas},
  booktitle={Advances in Neural Information Processing Systems (NeurIPS)},
  year={2022}
}

@inproceedings{xu2019gin,
  title={How Powerful are Graph Neural Networks?},
  author={Xu, Keyulu and Hu, Weihua and Leskovec, Jure and Jegelka, Stefanie},
  booktitle={International Conference on Learning Representations (ICLR)},
  year={2019}
}

@inproceedings{trivedi2019dyrep,
  title={{DyRep}: Learning Representations over Dynamic Graphs},
  author={Trivedi, Rakshit and Farajtabar, Mehrdad and Biswal, Prasenjeet and Zha, Hongyuan},
  booktitle={International Conference on Learning Representations (ICLR)},
  year={2019}
}

@inproceedings{xu2020tgat,
  title={Inductive Representation Learning on Temporal Graphs},
  author={Xu, Da and Ruan, Chuanwei and Korpeoglu, Evren and Kumar, Sushant and Achan, Kannan},
  booktitle={International Conference on Learning Representations (ICLR)},
  year={2020}
}

@inproceedings{yu2023dygformer,
  title={Towards Better Dynamic Graph Learning: New Architecture and Unified Library},
  author={Yu, Le and Sun, Leilei and Du, Bowen and Lv, Weifeng},
  booktitle={Advances in Neural Information Processing Systems (NeurIPS)},
  year={2023}
}

@inproceedings{poursafaei2022edgebank,
  title={Towards Better Evaluation for Dynamic Link Prediction},
  author={Poursafaei, Farimah and Huang, Shenyang and Pelrine, Kellin and Rabbany, Reihaneh},
  booktitle={NeurIPS Datasets and Benchmarks Track},
  year={2022}
}

@article{wang2021tcl,
  title={{TCL}: Transformer-Based Dynamic Graph Modelling via Contrastive Learning},
  author={Wang, Lu and Chang, Xiaofu and Li, Shuang and Chu, Yunfei and Li, Hui and Zhang, Wei and He, Xiaofeng and Song, Le and Zhou, Jingren and Yang, Hongxia},
  journal={arXiv preprint arXiv:2105.07944},
  year={2021}
}

@inproceedings{wang2021cawn,
  title={Inductive Representation Learning in Temporal Networks via Causal Anonymous Walks},
  author={Wang, Yanbang and Chang, Yen-Yu and Liu, Yunyu and Leskovec, Jure and Li, Pan},
  booktitle={International Conference on Learning Representations (ICLR)},
  year={2021}
}

@inproceedings{cong2023graphmixer,
  title={Do We Really Need Complicated Model Architectures for Temporal Networks?},
  author={Cong, Weilin and Zhang, Si and Kang, Jian and Yuan, Baichuan and Wu, Hao and Zhou, Xin and Tong, Hanghang and Mahdavi, Mehrdad},
  booktitle={International Conference on Learning Representations (ICLR)},
  year={2023}
}

@inproceedings{luo2022nat,
  title={Neighborhood-Aware Scalable Temporal Network Representation Learning},
  author={Luo, Yuhong and Li, Pan},
  booktitle={Learning on Graphs Conference (LoG)},
  year={2022}
}

@inproceedings{lu2024tpnet,
  title={Improving Temporal Link Prediction via Temporal Walk Matrix Projection},
  author={Lu, Xiaodong and Sun, Leilei and Zhu, Tongyu and Lv, Weifeng},
  booktitle={Advances in Neural Information Processing Systems (NeurIPS)},
  year={2024}
}

@inproceedings{zhang2024tncn,
  title={Efficient Neural Common Neighbor for Temporal Graph Link Prediction},
  author={Zhang, Xiaohui and Wang, Yanbo and Wang, Xiyuan and Zhang, Muhan},
  booktitle={Learning on Graphs Conference (LoG)},
  year={2024}
}

@inproceedings{barbero2022attention,
  title={Sheaf Attention Networks},
  author={Barbero, Federico and Bodnar, Cristian and S{\'a}ez de Oc{\'a}riz Borde, Haitz and Li{\`o}, Pietro},
  booktitle={NeurIPS 2022 Workshop on Symmetry and Geometry in Neural Representations},
  year={2022}
}

@article{wang2024contig,
  title={{ConTIG}: Continuous Representation Learning on Temporal Interaction Graphs},
  author={Wang, Zihui and Yang, Peizhen and Fan, Xiaoliang and Yan, Xu and Wu, Zonghan and Pan, Shirui and Chen, Longbiao and Zang, Yu and Wang, Cheng and Yu, Rongshan},
  journal={Neural Networks},
  volume={172},
  pages={106151},
  year={2024}
}

@article{borgi2025polynsd,
  title={Polynomial Neural Sheaf Diffusion: A Spectral Filtering Approach on Cellular Sheaves},
  author={Borgi, Alessio and Silvestri, Fabrizio and Li{\`o}, Pietro},
  journal={arXiv preprint arXiv:2512.00242},
  year={2025},
  doi={10.48550/arXiv.2512.00242}
}

@inproceedings{ribeiro2026csnn,
  title={Cooperative Sheaf Neural Networks},
  author={Ribeiro, Andr{\'e} and Ten{\'o}rio, Ana Luiza and Belieni, Juan and Souza, Amauri H. and Mesquita, Diego},
  booktitle={International Conference on Learning Representations (ICLR)},
  year={2026}
}

@inproceedings{bamberger2025bundle,
  title={Bundle Neural Networks for Message Diffusion on Graphs},
  author={Bamberger, Jacob and Barbero, Federico and Dong, Xiaowen and Bronstein, Michael M.},
  booktitle={International Conference on Learning Representations (ICLR)},
  year={2025}
}

@inproceedings{kofinas2021locs,
  title={Roto-translated Local Coordinate Frames For Interacting Dynamical Systems},
  author={Kofinas, Miltiadis and Nagaraja, Naveen Shankar and Gavves, Efstratios},
  booktitle={Advances in Neural Information Processing Systems (NeurIPS)},
  year={2021}
}

@inproceedings{dehaan2021gemcnn,
  title={Gauge Equivariant Mesh CNNs: Anisotropic Convolutions on Geometric Graphs},
  author={de Haan, Pim and Weiler, Maurice and Cohen, Taco and Welling, Max},
  booktitle={International Conference on Learning Representations (ICLR)},
  year={2021}
}

@inproceedings{he2021get,
  title={Gauge Equivariant Transformer},
  author={He, Lingshen and Dong, Yiming and Wang, Yisen and Tao, Dacheng and Lin, Zhouchen},
  booktitle={Advances in Neural Information Processing Systems (NeurIPS)},
  year={2021}
}

@inproceedings{du2022clofnet,
  title={{SE}(3) Equivariant Graph Neural Networks with Complete Local Frames},
  author={Du, Weitao and Zhang, He and Du, Yuanqi and Meng, Qi and Chen, Wei and Zheng, Nanning and Shao, Bin and Liu, Tie-Yan},
  booktitle={International Conference on Machine Learning (ICML)},
  year={2022}
}

@inproceedings{park2023hermes,
  title={Modeling Dynamics over Meshes with Gauge Equivariant Nonlinear Message Passing},
  author={Park, Jung Yeon and Wong, Lawson L. S. and Walters, Robin},
  booktitle={Advances in Neural Information Processing Systems (NeurIPS)},
  year={2023}
}

@inproceedings{choi2026sgpc,
  author    = {Choi, Yoonhyuk and Choi, Jiho and Kim, Chong-Kwon},
  title     = {Sheaf Graph Neural Networks via {PAC}-Bayes Spectral Optimization},
  booktitle = {Proceedings of the AAAI Conference on Artificial Intelligence},
  volume    = {40},
  number    = {25},
  pages     = {20570--20578},
  year      = {2026}
}

@inproceedings{gravina2024ctan,
  title={Long Range Propagation on Continuous-Time Dynamic Graphs},
  author={Gravina, Alessio and Lovisotto, Giulio and Gallicchio, Claudio and Bacciu, Davide and Grohnfeldt, Claas},
  booktitle={International Conference on Machine Learning (ICML)},
  year={2024},
  eprint={2406.02740},
  archivePrefix={arXiv},
  note={Introduces the Continuous-Time Graph Anti-Symmetric Network (CTAN)}
}

@techreport{pennebaker2015liwc,
  title={The Development and Psychometric Properties of {LIWC2015}},
  author={Pennebaker, James W. and Boyd, Ryan L. and Jordan, Kayla and Blackburn, Kate},
  institution={The University of Texas at Austin},
  year={2015},
  url={http://hdl.handle.net/2152/31333}
}

@article{gao2025hyperevent,
  title={{HyperEvent}: Learning Cohesive Events for Large-scale Dynamic Link Prediction},
  author={Gao, Jian and Wu, Jianshe and Ding, JingYi},
  journal={arXiv preprint arXiv:2507.11836},
  year={2025}
}

@inproceedings{li2023sthn,
  author    = {Li, Ce and Hong, Rongpei and Xu, Xovee and Trajcevski, Goce and Zhou, Fan},
  title     = {Simplifying Temporal Heterogeneous Network for Continuous-Time Link Prediction},
  booktitle = {Proceedings of the 32nd ACM International Conference on Information and Knowledge Management (CIKM '23)},
  year      = {2023},
  pages     = {1288--1297},
  publisher = {Association for Computing Machinery},
  address   = {New York, NY, USA},
  doi       = {10.1145/3583780.3615059},
  isbn      = {9798400701245}
}

@article{li2024dygmamba,
  title   = {{DyG-Mamba}: Continuous State Space Modeling on Dynamic Graphs},
  author  = {Li, Dongyuan and Tan, Shiyin and Zhang, Ying and Jin, Ming and Pan, Shirui and Okumura, Manabu and Jiang, Renhe},
  journal = {arXiv preprint arXiv:2408.06966},
  year    = {2024}
}

@article{demsar2006,
  title   = {Statistical Comparisons of Classifiers over Multiple Data Sets},
  author  = {Dem{\v{s}}ar, Janez},
  journal = {Journal of Machine Learning Research},
  volume  = {7},
  pages   = {1--30},
  year    = {2006}
}


\appendix

\section{Proofs and Supporting Derivations}
\label{app:proofs}

We collect complete proofs for all formal statements used by the model and
theory sections. The appendix is organized so that the proof dependencies are
explicit: Appendix~\ref{app:proof_setup} fixes notation; Appendix~\ref{app:frame_level_proofs}
handles frame-level facts; Appendix~\ref{app:frozen_operator_proofs} proves
path-independence, synchronized-coordinate equivalence, and spectral
equivalence; Appendix~\ref{app:energy_proofs} proves energy descent and
non-expansiveness; Appendix~\ref{app:stored_frame_proofs} proves stability for
stored frames; Appendices~\ref{app:nonlinear_scope} and~\ref{app:proof_twl}
record the scope and idealized expressivity statements.

\subsection{Common Notation and Proof Roadmap}
\label{app:proof_setup}

Throughout this appendix, we fix a time $t$ and abbreviate the active
undirected graph by
\[
G = G_t^{\mathrm{act}} = (V,E), \qquad n := |V|.
\]
The edge set is an undirected multiset, matching the active-history
construction in \S\ref{sec:diffusion}. Let
$d_a^{0}:=|\mathcal{N}_a^{\mathrm{act}}(t)|$ be the true active degree and
$d_a:=\bar d_a(t)=\max\{d_a^0,1\}$ the degree used only for
normalization. Thus isolated active nodes have zero Laplacian rows while
normalized operators remain well-defined. Let
\[
    \mathbf{D}_G^0 := \diag(d_a^0)_{a\in V},
    \qquad
    \mathbf{D}_G := \diag(d_a)_{a\in V},
    \qquad
    \boldsymbol{\Delta}_t := \mathbf{D}_G\otimes\mathbf{I}_d \;=\; \blkdiag\!\big((d_a\mathbf{I}_d)_{a\in V}\big),
\]
which is exactly the safe block-diagonal active-degree matrix $\boldsymbol{\Delta}_t$ defined in \S\ref{sec:diffusion} (Eq.~\eqref{eq:norm_sheaf_lap}); the Kronecker and block-diagonal forms are notationally distinct but denote the same operator.
We write the stacked state as
\[
\mathbf{z} = [\mathbf{h}_1^\top,\ldots,\mathbf{h}_n^\top]^\top \in \RR^{nd},
\]
the synchronization operator as
\[
\mathbf{S} = \mathbf{S}(t) = \blkdiag(\mathbf{U}_1(t),\ldots,\mathbf{U}_n(t)),
\]
and the synchronized coordinates as $\mathbf{g} = \mathbf{S}\mathbf{z}$.
The combinatorial and normalized graph Laplacians are
\[
    \mathbf{L}_G = \mathbf{D}_G^0-\mathbf{A}_G,
    \qquad
    \widetilde{\mathbf{L}}_G=\mathbf{D}_G^{-1}\mathbf{L}_G,
    \qquad
    \widetilde{\mathbf{L}}_G^{\mathrm{sym}}
    =\mathbf{D}_G^{-1/2}\mathbf{L}_G\mathbf{D}_G^{-1/2}.
\]
The corresponding sheaf Laplacians are
$\mathbf{L}_{\mathcal{F}}(t)$,
$\widetilde{\mathbf{L}}_{\mathcal{F}}(t)
=\boldsymbol{\Delta}_t^{-1}\mathbf{L}_{\mathcal{F}}(t)$, and
$\widetilde{\mathbf{L}}_{\mathcal{F}}^{\mathrm{sym}}(t)
=\boldsymbol{\Delta}_t^{-1/2}\mathbf{L}_{\mathcal{F}}(t)
\boldsymbol{\Delta}_t^{-1/2}$.
For a square matrix $\mathbf{A}$, $\sigma(\mathbf{A})$ denotes its spectrum and
$\rho(\mathbf{A}) := \max\{|\lambda|:\lambda\in\sigma(\mathbf{A})\}$ its
spectral radius.

The proof order covers all formal claims as follows:
\begin{enumerate}
    \item frame parameterization and coordinate consistency:
    Lemma~\ref{lem:low_rank}, Proposition~\ref{prop:carry_over}, and
    Lemma~\ref{lem:houshlip};
    \item frozen-frame operator structure:
    Lemma~\ref{lem:flat_transport}, Proposition~\ref{prop:sync_equiv}, and
    Corollary~\ref{cor:spectral_equivalence};
    \item descent and stability of diffusion:
    Theorems~\ref{thm:energy}, \ref{thm:feature_scaled}, and
    \ref{thm:nonexpansive}, plus Corollary~\ref{cor:energy_descent};
    \item stored-frame perturbation:
    Theorem~\ref{thm:stale} and the step-size derivation for
    Remark~\ref{rem:spectral_bound};
    \item scope and expressivity:
    Propositions~\ref{prop:psi_equiv} and~\ref{prop:twl}.
\end{enumerate}

\subsection{Frame Parameterization and Coordinate Carry-Over}
\label{app:frame_level_proofs}

\subsubsection{Proof of Lemma~\ref{lem:low_rank} (Low-Rank Orthogonal Perturbation)}
\label{app:proof_low_rank}

\begin{proof}
For a vector $\mathbf{f}$ with $\|\mathbf{f}\| > \varepsilon$, the Householder reflection $H_\varepsilon(\mathbf{f}) = \mathbf{I} - 2\hat{\mathbf{f}}\hat{\mathbf{f}}^\top$ (with $\hat{\mathbf{f}} := \mathbf{f}/\|\mathbf{f}\|$) satisfies $\mathrm{rank}(H_\varepsilon(\mathbf{f})-\mathbf{I}) = 1$ and $\mathrm{im}(H_\varepsilon(\mathbf{f})-\mathbf{I}) \subseteq \mathrm{span}\{\mathbf{f}\}$. If instead $\|\mathbf{f}\| \le \varepsilon$, then $H_\varepsilon(\mathbf{f}) = \mathbf{I}$ and the reflection contributes no rank.

Now fix a node $v$ and recall $\mathbf{U}_v = H_\varepsilon(\mathbf{f}_v^{(k)}) \cdots H_\varepsilon(\mathbf{f}_v^{(1)})$ (Eq.~\eqref{eq:householder_product}), together with the set of non-degenerate indices $I_v := \{i : \|\mathbf{f}_v^{(i)}\|>\varepsilon\}$ and $m_v := |I_v|$. Let
\[
S_v := \mathrm{span}\{\mathbf{f}_v^{(i)} : i \in I_v\}, \qquad \dim S_v \le m_v \le k.
\]
If $\mathbf{x} \in S_v^\perp$, then $\mathbf{x}$ is orthogonal to every non-degenerate $\mathbf{f}_v^{(i)}$, so each non-trivial reflection fixes $\mathbf{x}$; the degenerate factors also fix $\mathbf{x}$ because they equal $\mathbf{I}$. Hence $\mathbf{U}_v\mathbf{x} = \mathbf{x}$ for all $\mathbf{x} \in S_v^\perp$. Equivalently, $\ker(\mathbf{U}_v-\mathbf{I}) \supseteq S_v^\perp$, so $\mathrm{im}(\mathbf{U}_v-\mathbf{I}) \subseteq S_v$ and
\[
\mathrm{rank}(\mathbf{U}_v-\mathbf{I}) \le \dim S_v \le m_v.
\]

For the pairwise transport $\mathbf{Q}_{uv} = \mathbf{U}_u^\top \mathbf{U}_v$, let $S_{uv} := S_u + S_v$. If $\mathbf{x} \in S_{uv}^\perp \subseteq S_u^\perp \cap S_v^\perp$, then $\mathbf{U}_u\mathbf{x}=\mathbf{U}_v\mathbf{x}=\mathbf{x}$. Using orthogonality of $\mathbf{U}_u$,
\[
\mathbf{Q}_{uv}\mathbf{x} \;=\; \mathbf{U}_u^\top \mathbf{U}_v \mathbf{x} \;=\; \mathbf{U}_u^\top \mathbf{x} \;=\; \mathbf{x}.
\]
Thus $(\mathbf{Q}_{uv}-\mathbf{I})\mathbf{x}=\mathbf{0}$ for all $\mathbf{x}\in S_{uv}^\perp$, so
\[
\mathrm{rank}(\mathbf{Q}_{uv}-\mathbf{I}) \;\le\; \dim S_{uv} \;\le\; m_u + m_v \;\le\; 2k.
\]
\end{proof}

\subsubsection{Lipschitz Bound for the Householder Map}
\label{app:smoothness_lip}

We prove the explicit Lipschitz bound that justifies the smoothness regularizer
$\mathcal{L}_{\text{smooth}}$ and underlies
Eq.~\eqref{eq:smooth_lipschitz} in the main text.

\begin{lemma}[Householder Lipschitz bound]
\label{lem:houshlip}
Fix an operating-regime threshold $\tau_{\mathrm{H}}>\varepsilon$ (distinct from the
single-reflector safeguard $\varepsilon$ used inside $H_\varepsilon$ in
\S\ref{sec:sheaf_construction}, the node-type indexing function
$\tau(\cdot)$ of Eq.~\eqref{eq:score_features}, and the geometric
temperature $\tau_{\text{geo}}$ of Eq.~\eqref{eq:loss_geo}). Let
$\mathbf{F}_w(t),\mathbf{F}_w(t^-)\in\mathbb{R}^{k\times d}$ be the Householder
parameter matrices of node $w$ at two times, with rows
$\mathbf{f}_w^{(i)}(t)$ and $\mathbf{f}_w^{(i)}(t^-)$. If
$\min_i\min\big(\|\mathbf{f}_w^{(i)}(t)\|,\|\mathbf{f}_w^{(i)}(t^-)\|\big)\ge\tau_{\mathrm{H}}$, then
\[
\big\|\mathbf{I} - \mathbf{U}_w(t)^\top \mathbf{U}_w(t^-)\big\|_F
\;\le\; \frac{8\sqrt{k}}{\tau_{\mathrm{H}}}\,\|\Delta\mathbf{F}_w\|_F,
\]
and consequently
$\big\|\mathbf{I} - \mathbf{U}_w(t)^\top \mathbf{U}_w(t^-)\big\|_F^2 \le \frac{64\,k}{\tau_{\mathrm{H}}^2}\|\Delta\mathbf{F}_w\|_F^2$.
\end{lemma}

\begin{proof}
We proceed in three steps.

\textbf{Step 1: single-reflection Lipschitz bound.} For
$\mathbf{f},\mathbf{g}\in\mathbb{R}^d$ with $\|\mathbf{f}\|,\|\mathbf{g}\|\ge\tau_{\mathrm{H}}>\varepsilon$, both reflections take the
unit form $H_\varepsilon(\mathbf{f})=\mathbf{I}-2\hat{\mathbf{f}}\hat{\mathbf{f}}^\top$ with
$\hat{\mathbf{f}}=\mathbf{f}/\|\mathbf{f}\|$. Writing
\[
\hat{\mathbf{f}} - \hat{\mathbf{g}}
=
\frac{(\mathbf{f}-\mathbf{g})\|\mathbf{g}\| + \mathbf{g}(\|\mathbf{g}\|-\|\mathbf{f}\|)}{\|\mathbf{f}\|\|\mathbf{g}\|},
\]
and applying the reverse triangle inequality $|\|\mathbf{g}\|-\|\mathbf{f}\||\le\|\mathbf{f}-\mathbf{g}\|$,
\[
\|\hat{\mathbf{f}} - \hat{\mathbf{g}}\|
\;\le\;
\frac{2\|\mathbf{f}-\mathbf{g}\|}{\min(\|\mathbf{f}\|,\|\mathbf{g}\|)}
\;\le\;
\frac{2}{\tau_{\mathrm{H}}}\|\mathbf{f}-\mathbf{g}\|.
\]
For the rank-one outer products,
\[
\hat{\mathbf{f}}\hat{\mathbf{f}}^\top - \hat{\mathbf{g}}\hat{\mathbf{g}}^\top
=
(\hat{\mathbf{f}}-\hat{\mathbf{g}})\hat{\mathbf{f}}^\top
+ \hat{\mathbf{g}}(\hat{\mathbf{f}}-\hat{\mathbf{g}})^\top,
\]
hence by the triangle inequality and $\|\mathbf{a}\mathbf{b}^\top\|_F=\|\mathbf{a}\|\|\mathbf{b}\|$,
$\|\hat{\mathbf{f}}\hat{\mathbf{f}}^\top - \hat{\mathbf{g}}\hat{\mathbf{g}}^\top\|_F \le 2\|\hat{\mathbf{f}}-\hat{\mathbf{g}}\|$.
Therefore
\[
\|H_\varepsilon(\mathbf{f}) - H_\varepsilon(\mathbf{g})\|_F
\;=\; 2\,\|\hat{\mathbf{f}}\hat{\mathbf{f}}^\top - \hat{\mathbf{g}}\hat{\mathbf{g}}^\top\|_F
\;\le\; \frac{8}{\tau_{\mathrm{H}}}\,\|\mathbf{f}-\mathbf{g}\|.
\]

\textbf{Step 2: lifting to a product of $k$ reflections.} We relabel the reflectors so that $\mathbf{U}_w = H_1\cdots H_k$ realizes the product of Eq.~\eqref{eq:householder_product} (the bound is invariant under reordering of indices, since the telescoping below treats every factor symmetrically and the final sum runs over all $k$ rows). Write
$\mathbf{U}_w(t) = H_1^{(t)}\cdots H_k^{(t)}$ and
$\mathbf{U}_w(t^-) = H_1^{(t^-)}\cdots H_k^{(t^-)}$ with
$H_i^{(s)} := H_\varepsilon(\mathbf{f}_w^{(i)}(s))$. Telescoping,
\[
\mathbf{U}_w(t) - \mathbf{U}_w(t^-)
=
\sum_{i=1}^k
H_1^{(t)}\cdots H_{i-1}^{(t)}\,
\big(H_i^{(t)} - H_i^{(t^-)}\big)\,
H_{i+1}^{(t^-)}\cdots H_k^{(t^-)}.
\]
Each factor $H_j^{(s)}$ is orthogonal, so left- and right-multiplication
preserve the Frobenius norm. Combining the triangle inequality with
Step~1,
\[
\|\mathbf{U}_w(t) - \mathbf{U}_w(t^-)\|_F
\;\le\;
\sum_{i=1}^k \|H_i^{(t)} - H_i^{(t^-)}\|_F
\;\le\;
\frac{8}{\tau_{\mathrm{H}}}\sum_{i=1}^k \|\Delta\mathbf{f}_w^{(i)}\|.
\]
Cauchy--Schwarz then yields $\sum_i \|\Delta\mathbf{f}_w^{(i)}\| \le \sqrt{k}\,\|\Delta\mathbf{F}_w\|_F$, so
\[
\|\mathbf{U}_w(t) - \mathbf{U}_w(t^-)\|_F
\;\le\;
\frac{8\sqrt{k}}{\tau_{\mathrm{H}}}\,\|\Delta\mathbf{F}_w\|_F.
\]

\textbf{Step 3: orthogonal-map discrepancy.} Using orthogonality of
$\mathbf{U}_w(t)^\top$,
\[
\|\mathbf{I} - \mathbf{U}_w(t)^\top \mathbf{U}_w(t^-)\|_F
=
\|\mathbf{U}_w(t)^\top \big(\mathbf{U}_w(t) - \mathbf{U}_w(t^-)\big)\|_F
=
\|\mathbf{U}_w(t) - \mathbf{U}_w(t^-)\|_F,
\]
and the conclusion follows from Step~2. Squaring gives the stated bound
on $\|\mathbf{I}-\mathbf{U}_w(t)^\top\mathbf{U}_w(t^-)\|_F^2$.
\end{proof}

\begin{remark}[Operating regime and complementarity with Lemma~\ref{lem:low_rank}]
\label{rem:houshlip}
The hypothesis $\min_i\|\mathbf{f}_w^{(i)}\|\ge\tau_{\mathrm{H}}$ holds throughout training at $\tau_{\mathrm{H}}\approx10^{-2}$: random
initialization sets $\|\mathbf{f}_w^{(i)}\|\sim 1$, and the additive update
keeps frame norms bounded above $\tau_{\mathrm{H}}$, which is several orders of magnitude
larger than the single-reflector safeguard $\varepsilon=10^{-6}$ used inside
$H_\varepsilon$ (\S\ref{sec:sheaf_construction}). The two thresholds play
disjoint roles: $\varepsilon$ guards a single Householder reflector against
unit-vector divergence, whereas $\tau_{\mathrm{H}}$ governs the Lipschitz
constant of the product map $\mathbf{U}_v$ and is the only quantity that
appears in the bound. Lemma~\ref{lem:low_rank}
and Lemma~\ref{lem:houshlip} are complementary: the former bounds the
\emph{rank} of $\mathbf{U}_v-\mathbf{I}$ and $\mathbf{Q}_{uv}-\mathbf{I}$, capturing
expressivity per unit of parameter count; the latter bounds their
\emph{magnitude} under updates, capturing stability of the implemented map
under parameter drift. Together they justify the use of
$\mathcal{L}_{\text{smooth}}$ as an $\mathcal{O}(kd)$ regularizer that
controls the squared orthogonal-map discrepancy with explicit
constant $64k/\tau_{\mathrm{H}}^2$.
\end{remark}

\subsubsection{Proof of Proposition~\ref{prop:carry_over} (Coordinate-Consistent State Carry-Over)}
\label{app:proof_carry_over}

\begin{proof}
We prove the statement for node $u$; the argument for node $v$ is identical. Define $\bar{\mathbf{h}}_u^- := (\mathbf{U}_u^+)^\top \mathbf{U}_u^- \mathbf{h}_u(t^-)$. Using orthogonality of $\mathbf{U}_u^+$,
\[
\mathbf{U}_u^+\bar{\mathbf{h}}_u^-
\;=\;
\mathbf{U}_u^+(\mathbf{U}_u^+)^\top \mathbf{U}_u^- \mathbf{h}_u(t^-)
\;=\;
\mathbf{U}_u^- \mathbf{h}_u(t^-),
\]
so the representations of $\bar{\mathbf{h}}_u^-$ and $\mathbf{h}_u(t^-)$ in the common (synchronized) frame agree. For uniqueness, any $\tilde{\mathbf{h}}$ with $\mathbf{U}_u^+ \tilde{\mathbf{h}} = \mathbf{U}_u^- \mathbf{h}_u(t^-)$ yields $\tilde{\mathbf{h}} = (\mathbf{U}_u^+)^\top \mathbf{U}_u^- \mathbf{h}_u(t^-) = \bar{\mathbf{h}}_u^-$ after left-multiplying by $(\mathbf{U}_u^+)^\top$. Thus $\bar{\mathbf{h}}_u^-$ is the unique local-coordinate vector consistent with the pre-event embedding in synchronized coordinates.
\end{proof}

\subsection{Frozen-Frame Transport and Spectral Equivalence}
\label{app:frozen_operator_proofs}

\subsubsection{Proof of Lemma~\ref{lem:flat_transport} (Path-Independence of Node-Induced Transport)}
\label{app:proof_flat_transport}

\begin{lemma}[Node-induced transport is path-independent]
\label{lem:flat_transport}
For any path $v_0,v_1,\ldots,v_\ell$ in $G_t^{\mathrm{act}}$,
\[
\mathbf{Q}_{v_0v_1}(t)
\mathbf{Q}_{v_1v_2}(t)
\cdots
\mathbf{Q}_{v_{\ell-1}v_\ell}(t)
=
\mathbf{U}_{v_0}(t)^\top \mathbf{U}_{v_\ell}(t).
\]
Consequently, once node frames are fixed, transport depends only on the
path endpoints. In particular, transport around any closed cycle has
trivial holonomy.
\end{lemma}

\begin{proof}
For a path $v_0,v_1,\ldots,v_\ell$ in $G$, repeated substitution of
\[
\mathbf{Q}_{ab}(t) = \mathbf{U}_a(t)^\top \mathbf{U}_b(t)
\]
gives
\begin{align*}
\mathbf{Q}_{v_0v_1}(t)\mathbf{Q}_{v_1v_2}(t)\cdots \mathbf{Q}_{v_{\ell-1}v_\ell}(t)
&=
\mathbf{U}_{v_0}(t)^\top \mathbf{U}_{v_1}(t)\,
\mathbf{U}_{v_1}(t)^\top \mathbf{U}_{v_2}(t)\cdots
\mathbf{U}_{v_{\ell-1}}(t)^\top \mathbf{U}_{v_\ell}(t) \\
&=
\mathbf{U}_{v_0}(t)^\top
\big(\mathbf{U}_{v_1}(t)\mathbf{U}_{v_1}(t)^\top\big)
\mathbf{U}_{v_2}(t)\cdots
\mathbf{U}_{v_{\ell-1}}(t)^\top \mathbf{U}_{v_\ell}(t) \\
&=
\mathbf{U}_{v_0}(t)^\top \mathbf{U}_{v_\ell}(t),
\end{align*}
because each $\mathbf{U}_{v_i}(t)$ is orthogonal and thus satisfies
\[
\mathbf{U}_{v_i}(t)\mathbf{U}_{v_i}(t)^\top = \mathbf{I}.
\]
In particular, transport along any cycle $v_0,\ldots,v_\ell=v_0$ collapses to
\[
\mathbf{U}_{v_0}(t)^\top \mathbf{U}_{v_0}(t)=\mathbf{I},
\]
so holonomy is trivial and transport is path-independent.
\end{proof}

\subsubsection{Proof of Proposition~\ref{prop:sync_equiv} (Synchronized-Coordinate Equivalence, Normalized Form)}
\label{app:proof_sync_equiv}

\begin{proposition}[Synchronized-coordinate equivalence]
\label{prop:sync_equiv}
Let $\mathbf{L}_{\mathcal{F}}(t)$ be the combinatorial sheaf Laplacian
induced by the node-wise orthogonal frames
$\{\mathbf{U}_a(t)\}_{a\in V}$, and let
$\mathbf{L}_{G}$ be the ordinary combinatorial graph Laplacian of the active
graph. Then, for every edge $(u,v)\in E$,
\[
\|\mathbf{h}_u-\mathbf{Q}_{uv}(t)\mathbf{h}_v\|^2
=
\|\mathbf{U}_u(t)\mathbf{h}_u-\mathbf{U}_v(t)\mathbf{h}_v\|^2.
\]
Equivalently,
\[
\langle \mathbf{h}_u,\mathbf{Q}_{uv}(t)\mathbf{h}_v\rangle
=
\langle \mathbf{U}_u(t)\mathbf{h}_u,
        \mathbf{U}_v(t)\mathbf{h}_v\rangle,
\qquad
\|\mathbf{h}_u\|
=
\|\mathbf{U}_u(t)\mathbf{h}_u\|.
\]
Thus aligned distances, aligned inner products, and stalk norms are
basis-independent under synchronized orthogonal coordinates. Moreover,
\[
    \mathbf{L}_{\mathcal{F}}(t)
    =
    \mathbf{S}(t)^\top
    \big(
        \mathbf{L}_{G}\otimes\mathbf{I}_d
    \big)
    \mathbf{S}(t).
\]
The random-walk normalized and symmetric normalized sheaf Laplacians satisfy
\[
    \widetilde{\mathbf{L}}_{\mathcal{F}}(t)
    =
    \mathbf{S}(t)^\top
    \big(
        \widetilde{\mathbf{L}}_{G}
        \otimes\mathbf{I}_d
    \big)
    \mathbf{S}(t),
    \qquad
    \widetilde{\mathbf{L}}_{G}
    :=
    \mathbf{D}_{G}^{-1}
    \mathbf{L}_{G},
\]
and
\[
    \widetilde{\mathbf{L}}_{\mathcal{F}}^{\mathrm{sym}}(t)
    =
    \mathbf{S}(t)^\top
    \big(
        \widetilde{\mathbf{L}}_{G}^{\mathrm{sym}}
        \otimes\mathbf{I}_d
    \big)
    \mathbf{S}(t),
    \qquad
    \widetilde{\mathbf{L}}_{G}^{\mathrm{sym}}
    :=
    \mathbf{D}_{G}^{-1/2}
    \mathbf{L}_{G}
    \mathbf{D}_{G}^{-1/2}.
\]
\end{proposition}

\begin{proof}
We first prove the edgewise identity. For any edge $(u,v)\in E$,
\begin{align*}
\|\mathbf{h}_u - \mathbf{Q}_{uv}(t)\mathbf{h}_v\|^2
&=
\|\mathbf{U}_u(t)\mathbf{h}_u - \mathbf{U}_u(t)\mathbf{Q}_{uv}(t)\mathbf{h}_v\|^2 \\
&=
\|\mathbf{U}_u(t)\mathbf{h}_u - \mathbf{U}_u(t)\mathbf{U}_u(t)^\top \mathbf{U}_v(t)\mathbf{h}_v\|^2 \\
&=
\|\mathbf{U}_u(t)\mathbf{h}_u - \mathbf{U}_v(t)\mathbf{h}_v\|^2,
\end{align*}
using norm preservation under the orthogonal map $\mathbf{U}_u(t)$.

Define synchronized node coordinates $\mathbf{g}_u := \mathbf{U}_u(t)\mathbf{h}_u$. Summing the edgewise identity over $(u,v)\in E$ yields
\[
\sum_{(u,v)\in E}\|\mathbf{h}_u - \mathbf{Q}_{uv}(t)\mathbf{h}_v\|^2
=
\sum_{(u,v)\in E}\|\mathbf{g}_u - \mathbf{g}_v\|^2
=
\mathbf{g}^\top (\mathbf{L}_{G} \otimes \mathbf{I}_d)\mathbf{g}
=
\mathbf{z}^\top \mathbf{S}^\top (\mathbf{L}_{G}\otimes \mathbf{I}_d)\mathbf{S}\mathbf{z},
\]
where we used $\mathbf{g} = \mathbf{S}\mathbf{z}$. On the other hand,
$\sum_{(u,v)\in E}\|\mathbf{h}_u - \mathbf{Q}_{uv}(t)\mathbf{h}_v\|^2 = \mathbf{z}^\top \mathbf{L}_{\mathcal{F}}(t)\mathbf{z}$.
Since these quadratic forms agree for every $\mathbf{z}$ and both matrices are symmetric,
$\mathbf{L}_{\mathcal{F}}(t) = \mathbf{S}^\top (\mathbf{L}_{G}\otimes \mathbf{I}_d)\mathbf{S}$.

We now descend to the normalized form. Decompose $\mathbf{S} = \blkdiag(\mathbf{U}_1,\ldots,\mathbf{U}_n)$ and $\boldsymbol{\Delta}_t = \mathbf{D}_G\otimes\mathbf{I}_d = \blkdiag(\bar d_1(t)\mathbf{I}_d,\ldots,\bar d_n(t)\mathbf{I}_d)$, where $\bar d_v(t)=\max\{|\mathcal{N}_v^{\mathrm{act}}(t)|,1\}$. Each $v$-th diagonal block of $\boldsymbol{\Delta}_t$ is the scalar $\bar d_v(t) \mathbf{I}_d$, and each $v$-th diagonal block of $\mathbf{S}$ is the orthogonal $\mathbf{U}_v$. Scalar multiples of the identity commute with every matrix, so the blocks commute block-wise, giving $\mathbf{S}\boldsymbol{\Delta}_t = \boldsymbol{\Delta}_t\mathbf{S}$ and hence
\[
\mathbf{S}^\top \boldsymbol{\Delta}_t^{-1} \mathbf{S} \;=\; \boldsymbol{\Delta}_t^{-1}\mathbf{S}^\top\mathbf{S} \;=\; \boldsymbol{\Delta}_t^{-1}.
\]
Left-multiplying the sheaf identity by $\boldsymbol{\Delta}_t^{-1}$ and using this commutation yields
\[
\widetilde{\mathbf{L}}_{\mathcal{F}}(t) \;:=\; \boldsymbol{\Delta}_t^{-1}\mathbf{L}_{\mathcal{F}}(t) \;=\; \mathbf{S}^\top (\widetilde{\mathbf{L}}_G \otimes \mathbf{I}_d)\mathbf{S},
\qquad \widetilde{\mathbf{L}}_G := \mathbf{D}_G^{-1}\mathbf{L}_G.
\]
Conjugating the sheaf identity $\mathbf{L}_{\mathcal{F}}(t) = \mathbf{S}^\top(\mathbf{L}_G\otimes\mathbf{I}_d)\mathbf{S}$ by $\boldsymbol{\Delta}_t^{-1/2}$ on both sides and using the same blockwise commutation $\boldsymbol{\Delta}_t^{-1/2}\mathbf{S}^\top = \mathbf{S}^\top\boldsymbol{\Delta}_t^{-1/2}$ (scalar-block diagonal commutes with block-diagonal orthogonal) gives the symmetric form
\[
\widetilde{\mathbf{L}}_{\mathcal{F}}^{\mathrm{sym}}(t) \;:=\; \boldsymbol{\Delta}_t^{-1/2}\mathbf{L}_{\mathcal{F}}(t)\boldsymbol{\Delta}_t^{-1/2} \;=\; \mathbf{S}^\top\!\big((\mathbf{D}_G^{-1/2}\mathbf{L}_G\mathbf{D}_G^{-1/2})\otimes\mathbf{I}_d\big)\mathbf{S} \;=\; \mathbf{S}^\top(\widetilde{\mathbf{L}}_G^{\mathrm{sym}}\otimes\mathbf{I}_d)\mathbf{S}.
\]

Next, we derive the synchronized \emph{normalized} diffusion update. Starting from Eq.~\eqref{eq:diffusion_local} with $\mathbf{D}_\theta=\mathbf{I}_d$ and the same $\bar d_u(t)=\max\{|\mathcal{N}_u^{\mathrm{act}}(t)|,1\}$ convention,
\[
\mathbf{h}_u^{(\ell+1)}
= \mathbf{h}_u^{(\ell)}
- \frac{\eta}{\bar d_u(t)}\sum_{v \in \mathcal{N}_u^{\mathrm{act}}(t)}\big(\mathbf{h}_u^{(\ell)} - \mathbf{Q}_{uv}(t)\mathbf{h}_v^{(\ell)}\big),
\]
left-multiplying by $\mathbf{U}_u(t)$ and using $\mathbf{U}_u(t)\mathbf{Q}_{uv}(t) = \mathbf{U}_v(t)$ yields
\[
\mathbf{g}_u^{(\ell+1)}
= \mathbf{g}_u^{(\ell)}
- \frac{\eta}{\bar d_u(t)}\sum_{v\in\mathcal{N}_u^{\mathrm{act}}(t)}\big(\mathbf{g}_u^{(\ell)} - \mathbf{g}_v^{(\ell)}\big),
\]
which is the ordinary random-walk graph diffusion in synchronized coordinates.
\end{proof}

\subsubsection{Proof of Corollary~\ref{cor:spectral_equivalence} (Spectral Equivalence)}
\label{app:proof_spectral}

\begin{corollary}[Spectral equivalence]
\label{cor:spectral_equivalence}
The combinatorial sheaf Laplacian
$\mathbf{L}_{\mathcal{F}}(t)$ is orthogonally similar to
$\mathbf{L}_{G}\otimes\mathbf{I}_d$. The random-walk
normalized operators
$\widetilde{\mathbf{L}}_{\mathcal{F}}(t)$ and
$\widetilde{\mathbf{L}}_{G}\otimes\mathbf{I}_d$
are similar, while the symmetric normalized operators
$\widetilde{\mathbf{L}}_{\mathcal{F}}^{\mathrm{sym}}(t)$ and
$\widetilde{\mathbf{L}}_{G}^{\mathrm{sym}}\otimes\mathbf{I}_d$
are orthogonally similar. Hence
\[
    \sigma\!\left(
        \widetilde{\mathbf{L}}_{\mathcal{F}}^{\mathrm{sym}}(t)
    \right)
    \subseteq [0,2],
    \qquad
    \sigma\!\left(
        \widetilde{\mathbf{L}}_{\mathcal{F}}(t)
    \right)
    \subseteq [0,2],
    \qquad
    \rho\!\left(
        \widetilde{\mathbf{L}}_{\mathcal{F}}(t)
    \right)
    \le 2,
\]
independently of the maximum active degree.
\end{corollary}

\begin{proof}
Proposition~\ref{prop:sync_equiv} shows
$\mathbf{L}_{\mathcal{F}}(t) = \mathbf{S}^\top (\mathbf{L}_{G}\otimes \mathbf{I}_d)\mathbf{S}$.
Since $\mathbf{S}$ is orthogonal, $\mathbf{L}_{\mathcal{F}}(t)$ and
$\mathbf{L}_{G}\otimes\mathbf{I}_d$ are orthogonally similar, so their
spectra agree up to the $d$-fold Kronecker multiplicity. For the
random-walk normalized form,
\[
    \widetilde{\mathbf{L}}_{\mathcal{F}}(t)
    =
    \mathbf{S}^\top
    (\widetilde{\mathbf{L}}_G\otimes\mathbf{I}_d)
    \mathbf{S},
\]
so $\widetilde{\mathbf{L}}_{\mathcal{F}}(t)$ and
$\widetilde{\mathbf{L}}_G\otimes\mathbf{I}_d$ are similar and therefore
isospectral. The symmetric normalized identity in
Proposition~\ref{prop:sync_equiv} gives an orthogonal similarity between
$\widetilde{\mathbf{L}}_{\mathcal{F}}^{\mathrm{sym}}(t)$ and
$\widetilde{\mathbf{L}}_G^{\mathrm{sym}}\otimes\mathbf{I}_d$.
The standard Rayleigh-quotient bound, with isolated vertices contributing zero
rows under our convention,
$\sigma(\widetilde{\mathbf{L}}_G^{\mathrm{sym}})\subseteq[0,2]$
therefore gives
$\sigma(\widetilde{\mathbf{L}}_{\mathcal{F}}^{\mathrm{sym}}(t))\subseteq[0,2]$.
Because $\mathbf{D}_G$ is positive diagonal under our normalization convention,
$\widetilde{\mathbf{L}}_G$ is similar to
$\widetilde{\mathbf{L}}_G^{\mathrm{sym}}$ via conjugation by
$\mathbf{D}_G^{1/2}$. Therefore
$\sigma(\widetilde{\mathbf{L}}_{\mathcal{F}}(t))\subseteq[0,2]$ and hence
$\rho(\widetilde{\mathbf{L}}_{\mathcal{F}}(t))\le 2$.
\end{proof}

\subsection{Energy Descent and Non-Expansiveness}
\label{app:energy_proofs}

\subsubsection{Proof of Theorem~\ref{thm:energy} (Aligned Energy Descent, Normalized-Only Case)}
\label{app:proof_energy}

\begin{theorem}[Aligned energy descent, normalized-only case]
\label{thm:energy}
Define the aligned Dirichlet energy
\[
\mathcal{E}_t(\mathbf{z})
:=
\tfrac12
\sum_{(u,v)\in E_t^{\mathrm{act}}}
\|\mathbf{h}_u-\mathbf{Q}_{uv}(t)\mathbf{h}_v\|^2
=
\tfrac12
\mathbf{z}^\top
\mathbf{L}_{\mathcal{F}}(t)
\mathbf{z}.
\]
Then
\[
    \nabla \mathcal{E}_t(\mathbf{z})
    =
    \mathbf{L}_{\mathcal{F}}(t)\mathbf{z}.
\]
Moreover, the normalized-only update
\[
    \mathbf{z}^+
    =
    \mathbf{z}
    -
    \eta\,\widetilde{\mathbf{L}}_{\mathcal{F}}(t)\mathbf{z}
\]
is the $\boldsymbol{\Delta}_t$-preconditioned gradient step on
$\mathcal{E}_t$, equivalently gradient descent in the weighted inner
product
\[
    \langle \mathbf{x},\mathbf{y}\rangle_{\boldsymbol{\Delta}_t}
    :=
    \mathbf{x}^\top\boldsymbol{\Delta}_t\mathbf{y}.
\]
\end{theorem}

\begin{proof}
By Proposition~\ref{prop:sync_equiv},
$\mathcal{E}_t(\mathbf{z}) = \tfrac12 \mathbf{z}^\top \mathbf{L}_{\mathcal{F}}(t)\mathbf{z}$, and since $\mathbf{L}_{\mathcal{F}}(t)$ is symmetric, $\nabla_{\mathbf{z}}\mathcal{E}_t(\mathbf{z}) = \mathbf{L}_{\mathcal{F}}(t)\mathbf{z}$. The normalized-only update $\mathbf{z}^+ = \mathbf{z} - \eta\widetilde{\mathbf{L}}_{\mathcal{F}}(t)\mathbf{z} = \mathbf{z} - \eta\boldsymbol{\Delta}_t^{-1}\nabla\mathcal{E}_t(\mathbf{z})$ is therefore the $\boldsymbol{\Delta}_t$-gradient step on $\mathcal{E}_t$, i.e.\ steepest descent of $\mathcal{E}_t$ in the inner product $\langle\mathbf{x},\mathbf{y}\rangle_{\boldsymbol{\Delta}_t}:=\mathbf{x}^\top\boldsymbol{\Delta}_t\mathbf{y}$.

In synchronized coordinates, Proposition~\ref{prop:sync_equiv} gives
$\mathcal{E}_t(\mathbf{z}) = \tfrac12 \sum_{(u,v)\in E}\|\mathbf{g}_u - \mathbf{g}_v\|^2$,
the ordinary graph Dirichlet energy applied to the synchronized states, so the same preconditioned-GD interpretation transfers to the plain (non-sheaf) degree-normalized Laplacian.
\end{proof}

\subsubsection{Proof of Theorem~\ref{thm:feature_scaled} (Feature-Scaled Sheaf Descent)}
\label{app:proof_feature_scaled}

\begin{theorem}[Feature-scaled sheaf descent]
\label{thm:feature_scaled}
Let
\[
    \mathbf{P}
    :=
    \mathbf{I}_n\otimes\mathbf{D}_\theta,
    \qquad
    \mathbf{D}_\theta\succ0.
\]
Consider the full active-graph update
\[
    \mathbf{z}^+
    =
    \mathbf{z}
    -
    \eta\,
    \mathbf{P}\,
    \widetilde{\mathbf{L}}_{\mathcal{F}}(t)
    \mathbf{z}.
\]
Define the metric
\[
    \mathbf{M}_t
    :=
    \mathbf{D}_G\otimes\mathbf{D}_\theta^{-1}
    =
    \boldsymbol{\Delta}_t
    (\mathbf{I}_n\otimes\mathbf{D}_\theta^{-1})
    \in \mathbb{R}^{nd\times nd},
\]
and the symmetrized preconditioned operator
\[
    \mathbf{B}_t
    :=
    \mathbf{M}_t^{-1/2}
    \mathbf{L}_{\mathcal{F}}(t)
    \mathbf{M}_t^{-1/2}.
\]
Let
\[
    \mathrm{grad}_{\mathbf{M}_t}\mathcal{E}_t(\mathbf{z})
    :=
    \mathbf{M}_t^{-1}\nabla \mathcal{E}_t(\mathbf{z})
\]
denote the gradient with respect to the weighted inner product
\[
    \langle\mathbf{x},\mathbf{y}\rangle_{\mathbf{M}_t}
    :=
    \mathbf{x}^\top\mathbf{M}_t\mathbf{y}.
\]
Then
\[
    \mathbf{z}^+
    =
    \mathbf{z}
    -
    \eta\,
    \mathrm{grad}_{\mathbf{M}_t}\mathcal{E}_t(\mathbf{z}),
\]
so the full active-graph update is the $\mathbf{M}_t$-gradient step on the
combinatorial sheaf energy $\mathcal{E}_t$. Furthermore,
\[
\mathcal{E}_t(\mathbf{z}^+)
\le
\mathcal{E}_t(\mathbf{z})
-
\eta
\left(
    1-\tfrac{\eta}{2}\lambda_{\max}(\mathbf{B}_t)
\right)
\|\mathbf{L}_{\mathcal{F}}(t)\mathbf{z}\|_{\mathbf{M}_t^{-1}}^2.
\]
Since
\[
    \lambda_{\max}(\mathbf{B}_t)
    \le
    2\,\lambda_{\max}(\mathbf{D}_\theta),
\]
the sufficient condition
\[
    0<\eta\le
    \frac{1}{\lambda_{\max}(\mathbf{D}_\theta)}
\]
guarantees monotone energy non-increase and non-expansiveness of the
update in the $\mathbf{M}_t$-norm. The descent is strict at every
non-stationary point ($\mathbf{L}_{\mathcal{F}}(t)\mathbf{z}\neq\mathbf{0}$)
whenever $\eta<2/\lambda_{\max}(\mathbf{B}_t)$, which holds in particular
for any $\eta<1/\lambda_{\max}(\mathbf{D}_\theta)$ or whenever the
spectral bound $\lambda_{\max}(\mathbf{B}_t)\le 2\lambda_{\max}(\mathbf{D}_\theta)$
is non-tight.
\end{theorem}

\begin{proof}
Set $\mathbf{P}:=\mathbf{I}_n\otimes\mathbf{D}_\theta$ and recall $\mathbf{M}_t = \mathbf{D}_G\otimes\mathbf{D}_\theta^{-1} \succ 0$. Using the Kronecker identity $(\mathbf{D}_G^{-1}\otimes\mathbf{D}_\theta) = (\mathbf{I}_n\otimes\mathbf{D}_\theta)(\mathbf{D}_G^{-1}\otimes\mathbf{I}_d) = \mathbf{P}\,\boldsymbol{\Delta}_t^{-1}$, we have
\[
\mathbf{M}_t^{-1} \;=\; \mathbf{D}_G^{-1}\otimes\mathbf{D}_\theta \;=\; \mathbf{P}\,\boldsymbol{\Delta}_t^{-1}.
\]
Since $\nabla\mathcal{E}_t(\mathbf{z}) = \mathbf{L}_{\mathcal{F}}(t)\mathbf{z}$, the full active-graph update satisfies
\[
\mathbf{z}^+ \;=\; \mathbf{z} - \eta\,\mathbf{P}\,\widetilde{\mathbf{L}}_{\mathcal{F}}(t)\,\mathbf{z} \;=\; \mathbf{z} - \eta\,\mathbf{P}\,\boldsymbol{\Delta}_t^{-1}\,\mathbf{L}_{\mathcal{F}}(t)\,\mathbf{z} \;=\; \mathbf{z} - \eta\,\mathbf{M}_t^{-1}\,\nabla\mathcal{E}_t(\mathbf{z}),
\]
the $\mathbf{M}_t$-gradient step.

Let $\mathbf{y} := \mathbf{M}_t^{1/2}\mathbf{z}$. In these whitened coordinates the update is $\mathbf{y}^+ = (\mathbf{I} - \eta\mathbf{B}_t)\mathbf{y}$, and
$\mathcal{E}_t(\mathbf{z}) = \tfrac12 \mathbf{y}^\top\mathbf{B}_t\mathbf{y}$. Because $\mathbf{B}_t$ is symmetric PSD, its eigendecomposition $\mathbf{B}_t = \mathbf{U}\boldsymbol{\Lambda}\mathbf{U}^\top$ with $\mathbf{y} = \sum_i \alpha_i\mathbf{u}_i$ gives
\[
\mathcal{E}_t(\mathbf{z}^+) - \mathcal{E}_t(\mathbf{z})
= -\eta\sum_i \lambda_i(\mathbf{B}_t)^2\Big(1 - \tfrac{\eta}{2}\lambda_i(\mathbf{B}_t)\Big)\alpha_i^2
\le
-\eta\Big(1 - \tfrac{\eta}{2}\lambda_{\max}(\mathbf{B}_t)\Big)\,\|\mathbf{L}_{\mathcal{F}}(t)\mathbf{z}\|_{\mathbf{M}_t^{-1}}^2,
\]
where we used $\|\mathbf{L}_{\mathcal{F}}(t)\mathbf{z}\|_{\mathbf{M}_t^{-1}}^2 = \sum_i \lambda_i(\mathbf{B}_t)^2\alpha_i^2$. Hence $\mathcal{E}_t(\mathbf{z}^+)\le\mathcal{E}_t(\mathbf{z})$ whenever $0<\eta\le 2/\lambda_{\max}(\mathbf{B}_t)$. For $0<\eta<2/\lambda_{\max}(\mathbf{B}_t)$ the factor $1-\tfrac{\eta}{2}\lambda_i(\mathbf{B}_t)$ is strictly positive for every $i$, so the descent is strict at every non-stationary point ($\mathbf{L}_{\mathcal{F}}(t)\mathbf{z}\neq\mathbf{0}$); at the boundary $\eta=2/\lambda_{\max}(\mathbf{B}_t)$ the descent is strict unless $\mathbf{L}_{\mathcal{F}}(t)\mathbf{z}$ lies entirely in the top eigenspace of $\mathbf{B}_t$, a measure-zero condition.

For the spectral bound, factor
$\mathbf{B}_t = (\mathbf{I}_n\otimes\mathbf{D}_\theta^{1/2})\,\widetilde{\mathbf{L}}_{\mathcal{F}}^{\mathrm{sym}}(t)\,(\mathbf{I}_n\otimes\mathbf{D}_\theta^{1/2})$ where $\widetilde{\mathbf{L}}_{\mathcal{F}}^{\mathrm{sym}}(t) = \boldsymbol{\Delta}_t^{-1/2}\mathbf{L}_{\mathcal{F}}(t)\boldsymbol{\Delta}_t^{-1/2}$. From Corollary~\ref{cor:spectral_equivalence}, $\widetilde{\mathbf{L}}_{\mathcal{F}}^{\mathrm{sym}}(t)\preceq 2\mathbf{I}$, hence
$\mathbf{B}_t\preceq 2(\mathbf{I}_n\otimes\mathbf{D}_\theta)\preceq 2\lambda_{\max}(\mathbf{D}_\theta)\mathbf{I}$.
Therefore $\eta \le 1/\lambda_{\max}(\mathbf{D}_\theta)$ is sufficient for
monotone energy non-increase.

The same eigendecomposition gives non-expansiveness in the $\mathbf{M}_t$-norm.
For any input difference vector $\mathbf{w}$, let $\mathbf{w}^+$ denote the
corresponding output difference after one update step, set
$\mathbf{r}:=\mathbf{M}_t^{1/2}\mathbf{w}$ and write
$\mathbf{r}=\sum_i \beta_i\mathbf{u}_i$ in an eigenbasis of
$\mathbf{B}_t$. Then
\[
\|\mathbf{w}^+ \|_{\mathbf{M}_t}^2
=
\|(\mathbf{I}-\eta\mathbf{B}_t)\mathbf{r}\|_2^2
=
\sum_i (1-\eta\lambda_i(\mathbf{B}_t))^2\beta_i^2 .
\]
If $0<\eta\le 2/\lambda_{\max}(\mathbf{B}_t)$, every multiplier satisfies
$|1-\eta\lambda_i(\mathbf{B}_t)|\le1$, so the update is non-expansive.
If $\lambda_{\max}(\mathbf{B}_t)=0$ this is immediate; otherwise the displayed
spectral bound implies
$1/\lambda_{\max}(\mathbf{D}_\theta)\le 2/\lambda_{\max}(\mathbf{B}_t)$.
Hence the sufficient condition in the theorem guarantees non-expansiveness.
\end{proof}

\begin{corollary}[Monotone energy descent in the normalized-only case]
\label{cor:energy_descent}
Under the assumptions of Theorems~\ref{thm:energy} and~\ref{thm:feature_scaled} with $\mathbf{D}_\theta=\mathbf{I}_d$, and for $0<\eta\le 1$,
\[
\mathcal{E}_t(\mathbf{z}^+)
\;\le\;
\mathcal{E}_t(\mathbf{z}) \;-\; \eta\Big(1 - \tfrac{\eta}{2}\,\lambda_{\max}(\widetilde{\mathbf{L}}_{\mathcal{F}}^{\mathrm{sym}}(t))\Big)\,\|\mathbf{L}_{\mathcal{F}}(t)\mathbf{z}\|_{\boldsymbol{\Delta}_t^{-1}}^2.
\]
In particular $\mathcal{E}_t(\mathbf{z}^+) \le \mathcal{E}_t(\mathbf{z})$, with strict inequality whenever $\mathbf{L}_{\mathcal{F}}(t)\mathbf{z} \neq \mathbf{0}$ and $0<\eta<1$.
\end{corollary}

\begin{proof}
This is the $\mathbf{D}_\theta=\mathbf{I}_d$ specialization of
Theorem~\ref{thm:feature_scaled}: substituting
$\mathbf{D}_\theta=\mathbf{I}_d$ in that theorem yields
$\mathbf{M}_t = \boldsymbol{\Delta}_t$ and
$\mathbf{B}_t = \boldsymbol{\Delta}_t^{-1/2}\mathbf{L}_{\mathcal{F}}(t)
\boldsymbol{\Delta}_t^{-1/2}
= \widetilde{\mathbf{L}}_{\mathcal{F}}^{\mathrm{sym}}(t)$.
The descent inequality from Theorem~\ref{thm:feature_scaled} gives the stated
bound, and Corollary~\ref{cor:spectral_equivalence} provides
$\lambda_{\max}(\widetilde{\mathbf{L}}_{\mathcal{F}}^{\mathrm{sym}}(t))\le 2$,
so the descent coefficient is nonnegative for $0<\eta\le1$ and strictly
positive at non-stationary points when $0<\eta<1$.
\end{proof}

\subsubsection{Proof of Theorem~\ref{thm:nonexpansive} (Firm Non-Expansiveness, Normalized-Only Case)}
\label{app:proof_nonexpansive}

\begin{theorem}[Firm non-expansiveness, normalized-only case]
\label{thm:nonexpansive}
Set $\mathbf{D}_\theta=\mathbf{I}_d$ and define
\[
    \mathbf{T}_t(\mathbf{z})
    :=
    \mathbf{z}
    -
    \eta\,\widetilde{\mathbf{L}}_{\mathcal{F}}(t)\mathbf{z}.
\]
For any $0<\eta\le\tfrac12$, $\mathbf{T}_t$ is firmly non-expansive in
the $\boldsymbol{\Delta}_t$-inner product:
\begin{equation}
    \|\mathbf{T}_t(\mathbf{z}_1)-\mathbf{T}_t(\mathbf{z}_2)\|_{\boldsymbol{\Delta}_t}^2
    +
    \|(\mathbf{I}-\mathbf{T}_t)(\mathbf{z}_1)
      -(\mathbf{I}-\mathbf{T}_t)(\mathbf{z}_2)\|_{\boldsymbol{\Delta}_t}^2
    \le
    \|\mathbf{z}_1-\mathbf{z}_2\|_{\boldsymbol{\Delta}_t}^2.
    \label{eq:firm_nonexpansive}
\end{equation}
Firm non-expansiveness implies ordinary non-expansiveness and convergence of
the diffusion iterates to the kernel of $\widetilde{\mathbf{L}}_{\mathcal{F}}(t)$.
\end{theorem}

\begin{proof}
Set $\mathbf{D}_\theta=\mathbf{I}_d$ and write $\widetilde{\mathbf{L}}_{\mathcal{F}}^{\mathrm{sym}} := \boldsymbol{\Delta}_t^{-1/2}\mathbf{L}_{\mathcal{F}}(t)\boldsymbol{\Delta}_t^{-1/2}$, which is symmetric positive semidefinite with eigenvalues in $[0,2]$ by Corollary~\ref{cor:spectral_equivalence}. Define $\mathbf{T}_t := \mathbf{I} - \eta\widetilde{\mathbf{L}}_{\mathcal{F}}(t)$ and let $\mathbf{w} := \mathbf{z}_1 - \mathbf{z}_2$. Whiten to $\mathbf{u} := \boldsymbol{\Delta}_t^{1/2}\mathbf{w}$, so $\|\mathbf{w}\|_{\boldsymbol{\Delta}_t}^2 = \|\mathbf{u}\|_2^2$. Conjugating by $\boldsymbol{\Delta}_t^{1/2}$ gives
\[
\boldsymbol{\Delta}_t^{1/2}\mathbf{T}_t\boldsymbol{\Delta}_t^{-1/2}
\;=\;
\mathbf{I} - \eta\,\boldsymbol{\Delta}_t^{1/2}\widetilde{\mathbf{L}}_{\mathcal{F}}(t)\boldsymbol{\Delta}_t^{-1/2}
\;=\;
\mathbf{I} - \eta\,\widetilde{\mathbf{L}}_{\mathcal{F}}^{\mathrm{sym}},
\]
and therefore
\[
\|\mathbf{T}_t\mathbf{w}\|_{\boldsymbol{\Delta}_t}
= \|(\mathbf{I}-\eta\widetilde{\mathbf{L}}_{\mathcal{F}}^{\mathrm{sym}})\mathbf{u}\|_2,
\qquad
\|(\mathbf{I}-\mathbf{T}_t)\mathbf{w}\|_{\boldsymbol{\Delta}_t} = \|\eta\widetilde{\mathbf{L}}_{\mathcal{F}}^{\mathrm{sym}}\mathbf{u}\|_2.
\]

Let $\{\mathbf{v}_i\}$ be an orthonormal eigenbasis of $\widetilde{\mathbf{L}}_{\mathcal{F}}^{\mathrm{sym}}$ with eigenvalues $0\le \lambda_1\le\cdots\le \lambda_m\le 2$, and write $\mathbf{u} = \sum_i \alpha_i\mathbf{v}_i$. Then
\[
(\mathbf{I} - \eta\widetilde{\mathbf{L}}_{\mathcal{F}}^{\mathrm{sym}})\mathbf{u}
= \sum_i (1-\eta\lambda_i)\alpha_i\mathbf{v}_i,
\qquad
\eta\widetilde{\mathbf{L}}_{\mathcal{F}}^{\mathrm{sym}}\mathbf{u}
= \sum_i \eta\lambda_i\alpha_i\mathbf{v}_i.
\]
Orthogonality of the eigenbasis $\{\mathbf{v}_i\}$ gives
\begin{align*}
\|\mathbf{T}_t\mathbf{w}\|_{\boldsymbol{\Delta}_t}^2 + \|(\mathbf{I}-\mathbf{T}_t)\mathbf{w}\|_{\boldsymbol{\Delta}_t}^2
&= \sum_i \big((1-\eta\lambda_i)^2 + \eta^2\lambda_i^2\big)\alpha_i^2 \\
&= \sum_i \big(1 - 2\eta\lambda_i + 2\eta^2\lambda_i^2\big)\alpha_i^2.
\end{align*}
For $0 < \eta \le \tfrac12$ and $\lambda_i\in[0,2]$ we have $\eta\lambda_i\in[0,1]$, hence $1 - 2\eta\lambda_i + 2\eta^2\lambda_i^2 \le 1$. Therefore
\[
\|\mathbf{T}_t\mathbf{w}\|_{\boldsymbol{\Delta}_t}^2 + \|(\mathbf{I}-\mathbf{T}_t)\mathbf{w}\|_{\boldsymbol{\Delta}_t}^2
\;\le\; \sum_i \alpha_i^2
\;=\; \|\mathbf{u}\|_2^2
\;=\; \|\mathbf{w}\|_{\boldsymbol{\Delta}_t}^2,
\]
which is Eq.~\eqref{eq:firm_nonexpansive}. Firm non-expansiveness in the $\boldsymbol{\Delta}_t$-metric implies ordinary non-expansiveness $\|\mathbf{T}_t\mathbf{w}\|_{\boldsymbol{\Delta}_t} \le \|\mathbf{w}\|_{\boldsymbol{\Delta}_t}$. Iterating $\mathbf{T}_t^k$ converges to the $\boldsymbol{\Delta}_t$-orthogonal projector onto $\ker\widetilde{\mathbf{L}}_{\mathcal{F}}(t) = \ker\mathbf{L}_{\mathcal{F}}(t)$.
\end{proof}

\subsection{Stored Frames and Drift Stability}
\label{app:stored_frame_proofs}

\subsubsection{Proof of Theorem~\ref{thm:stale} (Stored-Frame Transport and Frame-Drift Stability)}
\label{app:proof_stale}

\begin{theorem}[Stored-frame transport and frame-drift stability]
\label{thm:stale}
For each node $w$, let $s_w(t)\le t$ denote its last update time, and
set
\[
    \tilde{\mathbf{U}}_w(t)
    :=
    \mathbf{U}_w(s_w(t)),
    \qquad
    \tilde{\mathbf{Q}}_{ab}(t)
    :=
    \tilde{\mathbf{U}}_a(t)^\top\tilde{\mathbf{U}}_b(t).
\]
Then the stored-frame transport is still node-induced and
path-independent, with random-walk normalized operator
\[
\widetilde{\mathbf{L}}_{\mathcal{F}}^{\mathrm{stored}}(t)
=
\tilde{\mathbf{S}}(t)^\top
\big(
    \widetilde{\mathbf{L}}_G
    \otimes\mathbf{I}_d
\big)
\tilde{\mathbf{S}}(t).
\]
For any reference orthogonal frame family
$\{\hat{\mathbf{U}}_w(t)\}_{w\in V_t^{\mathrm{act}}}\subset O(d)$, let
$\widehat{\widetilde{\mathbf{L}}}_{\mathcal{F}}(t)$ denote the
random-walk normalized sheaf Laplacian induced by
$\{\hat{\mathbf{U}}_w(t)\}$, and define
\[
    \delta_t
    :=
    \max_{w\in V_t^{\mathrm{act}}}
    \|\tilde{\mathbf{U}}_w(t)-\hat{\mathbf{U}}_w(t)\|_2.
\]
Then
\[
\left\|
\widetilde{\mathbf{L}}_{\mathcal{F}}^{\mathrm{stored}}(t)
-
\widehat{\widetilde{\mathbf{L}}}_{\mathcal{F}}(t)
\right\|_{\boldsymbol{\Delta}_t\to\boldsymbol{\Delta}_t}
\le
2\,\delta_t.
\]
Moreover, for
\[
    \widetilde{\mathbf{T}}_t(\mathbf{z})
    :=
    \mathbf{z}
    -
    \eta\,
    \mathbf{P}\,
    \widetilde{\mathbf{L}}_{\mathcal{F}}^{\mathrm{stored}}(t)
    \mathbf{z},
\]
if $0<\eta\le 1/\lambda_{\max}(\mathbf{D}_\theta)$, then
\[
\|\widetilde{\mathbf{T}}_t\mathbf{x}
-
\widetilde{\mathbf{T}}_t\mathbf{y}\|_{\mathbf{M}_t}
\le
\big(
    1+
    2\eta\,\lambda_{\max}(\mathbf{D}_\theta)\,\delta_t
\big)
\|\mathbf{x}-\mathbf{y}\|_{\mathbf{M}_t}.
\]
\end{theorem}

\begin{proof}
Define the stored and refreshed synchronization operators and update maps
\[
\tilde{\mathbf{S}}(t) := \blkdiag(\tilde{\mathbf{U}}_w(t))_{w\in V_t^{\mathrm{act}}},
\qquad
\hat{\mathbf{S}}(t) := \blkdiag(\hat{\mathbf{U}}_w(t))_{w\in V_t^{\mathrm{act}}},
\]
\[
\tilde{\mathbf{T}}_t(\mathbf{z}) := \mathbf{z} - \eta\,\mathbf{P}\,\widetilde{\mathbf{L}}_{\mathcal{F}}^{\mathrm{stored}}(t)\mathbf{z},
\qquad
\hat{\mathbf{T}}_t(\mathbf{z}) := \mathbf{z} - \eta\,\mathbf{P}\,\widehat{\widetilde{\mathbf{L}}}_{\mathcal{F}}(t)\mathbf{z}.
\]
The stored-frame factorization then follows from the argument of Proposition~\ref{prop:sync_equiv} with $\mathbf{S}$ replaced by $\tilde{\mathbf{S}}(t)$: since each $\tilde{\mathbf{U}}_w(t)$ is orthogonal, the derivation $\mathbf{L}_{\mathcal{F}} = \tilde{\mathbf{S}}^\top(\mathbf{L}_G\otimes\mathbf{I}_d)\tilde{\mathbf{S}}$ is unchanged, and the blockwise commutation of $\boldsymbol{\Delta}_t$ with $\tilde{\mathbf{S}}$ gives $\widetilde{\mathbf{L}}_{\mathcal{F}}^{\mathrm{stored}} = \tilde{\mathbf{S}}^\top(\widetilde{\mathbf{L}}_G\otimes\mathbf{I}_d)\tilde{\mathbf{S}}$.

We prove the two bounds in turn.

\textbf{Operator bound.} Let
$\mathbf{N}_G := \mathbf{D}_G^{-1/2}\mathbf{A}_G\mathbf{D}_G^{-1/2}$
be the symmetric normalized adjacency of the active graph.
Let $\boldsymbol{\Pi}_G:=\diag(\mathbf{1}_{\{d_a^0>0\}})_{a\in V}$. Under the
isolated-node convention,
$\widetilde{\mathbf{L}}_G^{\mathrm{sym}}=\boldsymbol{\Pi}_G-\mathbf{N}_G$:
the $\boldsymbol{\Pi}_G$ term is block-diagonal with scalar blocks and therefore
commutes with both $\tilde{\mathbf{S}}(t)$ and $\hat{\mathbf{S}}(t)$, so it
cancels in the difference between the stored and refreshed operators. The
standard edgewise bound gives $\|\mathbf{N}_G\|_2\le1$. Using
Proposition~\ref{prop:sync_equiv},
\[
\widetilde{\mathbf{L}}_{\mathcal{F}}^{\mathrm{stored,sym}}(t) - \widehat{\widetilde{\mathbf{L}}}_{\mathcal{F}}^{\mathrm{sym}}(t)
=
\hat{\mathbf{S}}(t)^\top(\mathbf{N}_G\otimes\mathbf{I}_d)\hat{\mathbf{S}}(t)
\;-\;
\tilde{\mathbf{S}}(t)^\top(\mathbf{N}_G\otimes\mathbf{I}_d)\tilde{\mathbf{S}}(t),
\]
and $\|\tilde{\mathbf{S}}(t) - \hat{\mathbf{S}}(t)\|_2 = \delta_t$ because $\tilde{\mathbf{S}}, \hat{\mathbf{S}}$ are block-diagonal. A standard $a b - c d = (a-c)b + c(b-d)$ split therefore yields
\[
\|\widetilde{\mathbf{L}}_{\mathcal{F}}^{\mathrm{stored,sym}}(t) - \widehat{\widetilde{\mathbf{L}}}_{\mathcal{F}}^{\mathrm{sym}}(t)\|_2
\le 2\,\|\mathbf{N}_G\otimes\mathbf{I}_d\|_2\,\delta_t \le 2\delta_t.
\]
Since
$\|\mathbf{A}\|_{\boldsymbol{\Delta}_t\to\boldsymbol{\Delta}_t}
= \|\boldsymbol{\Delta}_t^{1/2}\mathbf{A}\boldsymbol{\Delta}_t^{-1/2}\|_2$,
this is equivalent to
$\|\widetilde{\mathbf{L}}_{\mathcal{F}}^{\mathrm{stored}}(t)
- \widehat{\widetilde{\mathbf{L}}}_{\mathcal{F}}(t)\|_{\boldsymbol{\Delta}_t\to\boldsymbol{\Delta}_t}
\le 2\delta_t$.

\textbf{Lipschitz bound.} Let
$\mathbf{E}_t := \widetilde{\mathbf{L}}_{\mathcal{F}}^{\mathrm{stored}}(t)
- \widehat{\widetilde{\mathbf{L}}}_{\mathcal{F}}(t)$ and
$\mathbf{w} := \mathbf{x}-\mathbf{y}$, so
$\tilde{\mathbf{T}}_t\mathbf{w}
= \hat{\mathbf{T}}_t\mathbf{w} - \eta\,\mathbf{P}\mathbf{E}_t\mathbf{w}$.
Theorem~\ref{thm:feature_scaled} gives
$\|\hat{\mathbf{T}}_t\mathbf{w}\|_{\mathbf{M}_t}\le\|\mathbf{w}\|_{\mathbf{M}_t}$
under the stated step-size condition. For the perturbation term, set
$\mathbf{E}_t^{\mathrm{sym}} := \boldsymbol{\Delta}_t^{1/2}\mathbf{E}_t\boldsymbol{\Delta}_t^{-1/2}$
with $\|\mathbf{E}_t^{\mathrm{sym}}\|_2\le 2\delta_t$; then
\[
\mathbf{M}_t^{1/2}\mathbf{P}\mathbf{E}_t\mathbf{M}_t^{-1/2}
=
(\mathbf{I}_n\otimes\mathbf{D}_\theta^{1/2})\,\mathbf{E}_t^{\mathrm{sym}}\,(\mathbf{I}_n\otimes\mathbf{D}_\theta^{1/2}),
\]
hence $\|\mathbf{P}\mathbf{E}_t\|_{\mathbf{M}_t\to\mathbf{M}_t} \le \lambda_{\max}(\mathbf{D}_\theta)\,\|\mathbf{E}_t^{\mathrm{sym}}\|_2 \le 2\lambda_{\max}(\mathbf{D}_\theta)\,\delta_t$. Combining,
\[
\|\tilde{\mathbf{T}}_t\mathbf{x} - \tilde{\mathbf{T}}_t\mathbf{y}\|_{\mathbf{M}_t}
\le
\|\mathbf{w}\|_{\mathbf{M}_t} + 2\eta\,\lambda_{\max}(\mathbf{D}_\theta)\,\delta_t\,\|\mathbf{w}\|_{\mathbf{M}_t}.
\]
\end{proof}

\begin{remark}[Frame-drift control by $\mathcal{L}_{\text{smooth}}$]
\label{rem:drift_control}
Choosing the always-refreshed reference $\hat{\mathbf{U}}_w(t):=\mathbf{U}_w(t)$
in Theorem~\ref{thm:stale}, Lemma~\ref{lem:houshlip} gives
$\|\mathbf{U}_w(t)-\mathbf{U}_w(t^-)\|_F
\le (8\sqrt{k}/\tau_{\mathrm{H}})\|\Delta\mathbf{F}_w(t)\|_F$ between consecutive event updates.
Thus $\mathcal{L}_{\text{smooth}}$ controls $\|\mathbf{U}_w(t)-\mathbf{U}_w(t^-)\|_F^2$ via the explicit constant $64k/\tau_{\mathrm{H}}^2$.
\end{remark}

\begin{remark}[Step-size bound and scope]
\label{rem:spectral_bound}
The normalized-only condition $\eta\le\tfrac12$ and the feature-scaled condition
$\eta\le1/\lambda_{\max}(\mathbf{D}_\theta)$ are independent of the maximum
active degree. These guarantees apply to the frozen-frame diffusion operator;
the full \tsnn{} model also includes evolving frames, carry-over across frame
changes, and nonlinear message and recurrent updates in local coordinates.
\end{remark}

\subsubsection{Step-Size Derivation for Remark~\ref{rem:spectral_bound}}
\label{app:proof_step_size}

\begin{proof}[Derivation]
We justify the step-size bound. By Corollary~\ref{cor:spectral_equivalence},
$\sigma(\widetilde{\mathbf{L}}_{\mathcal{F}}^{\mathrm{sym}}(t))\subseteq[0,2]$
independent of $d_{\max}^{\mathrm{act}}$. The bound follows directly from the
Rayleigh quotient of the symmetric-normalized graph Laplacian under our
isolated-node convention. For any $\mathbf{x}\neq\mathbf{0}$, set
$\mathbf{y}=\mathbf{D}_G^{-1/2}\mathbf{x}$. Then
\[
\mathbf{x}^\top\widetilde{\mathbf{L}}_G^{\mathrm{sym}}\mathbf{x}
=
\sum_{(u,v)\in E}(y_u-y_v)^2
\ge 0,
\]
and, using $(a-b)^2\le2a^2+2b^2$ edgewise,
\[
\mathbf{x}^\top\widetilde{\mathbf{L}}_G^{\mathrm{sym}}\mathbf{x}
\le
2\sum_v d_v^0y_v^2
\le
2\sum_v d_v y_v^2
=2\mathbf{x}^\top\mathbf{x}.
\]
Thus $\sigma(\widetilde{\mathbf{L}}_G^{\mathrm{sym}})\subseteq[0,2]$ and, by
Corollary~\ref{cor:spectral_equivalence},
$\sigma(\widetilde{\mathbf{L}}_{\mathcal{F}}^{\mathrm{sym}}(t))\subseteq[0,2]$.
Similarly, for $\mathbf{N}_G=\mathbf{D}_G^{-1/2}\mathbf{A}_G\mathbf{D}_G^{-1/2}$,
applying $|2 y_u y_v| \le y_u^2 + y_v^2$ edgewise gives
\[
\left|\mathbf{x}^\top\mathbf{D}_G^{-1/2}\mathbf{A}_G\mathbf{D}_G^{-1/2}\mathbf{x}\right|
=|\mathbf{y}^\top\mathbf{A}_G\mathbf{y}|
\le \sum_v d_v^0 y_v^2
\le \sum_v d_v y_v^2
= \mathbf{x}^\top\mathbf{x}.
\]
Hence $\|\mathbf{N}_G\|_2\le1$, as used in
Appendix~\ref{app:proof_stale}. Theorem~\ref{thm:nonexpansive} then gives the
normalized-only step-size bound
\[
0 < \eta \le \tfrac12.
\]
With the learned feature gain $\mathbf{D}_\theta = \softplus(\boldsymbol{\theta})$, Theorem~\ref{thm:feature_scaled} gives the sufficient feature-scaled condition
\[
0 < \eta \le \frac{1}{\lambda_{\max}(\mathbf{D}_\theta)}.
\]
Because we initialize $\boldsymbol{\theta}=\mathbf{0}$, giving $\mathbf{D}_\theta = \softplus(0)\mathbf{I}_d = \log(2)\,\mathbf{I}_d \approx 0.693\,\mathbf{I}_d$ at step~0, the bound starts as $\eta \le 1/\log(2)\approx 1.44$. Our choices $\eta\in\{0.005,0.01\}$, capped at $\eta\le0.02$, therefore lie well inside the guaranteed regime, with a margin that covers any subsequent growth of $\mathbf{D}_\theta$ during training.
\end{proof}

\subsection{Scope of Nonlinear Local Modules}
\label{app:nonlinear_scope}

\subsubsection{Why the Nonlinear Modules Do Not Reduce to Shared-Space Operations}
\label{app:proof_nonreducible}

We rigorously characterize the class of functions that admit a shared-space
reduction and prove that \tsnn{}'s message and recurrent modules lie strictly
outside this class.

Consider a generic local message function
\[
\mathbf{m}_{u\leftarrow v} = \Psi\!\big(\mathbf{h}_u,\,\mathbf{Q}_{uv}\mathbf{h}_v,\,\xi_{uv}\big),
\]
where $\xi_{uv}$ bundles event features, time encodings, and relation embeddings. Writing $\mathbf{g}_u = \mathbf{U}_u \mathbf{h}_u$ and $\mathbf{g}_v = \mathbf{U}_v \mathbf{h}_v$, we have $\mathbf{h}_u = \mathbf{U}_u^\top \mathbf{g}_u$ and $\mathbf{Q}_{uv}\mathbf{h}_v = \mathbf{U}_u^\top \mathbf{g}_v$, so the message becomes
\[
\mathbf{m}_{u\leftarrow v} = \Psi\!\big(\mathbf{U}_u^\top \mathbf{g}_u,\,\mathbf{U}_u^\top \mathbf{g}_v,\,\xi_{uv}\big).
\]
The synchronized-coordinate reduction asks for a shared-space function $\bar\Psi$ satisfying
\begin{equation}
\Psi\!\big(\mathbf{U}^\top \mathbf{a},\mathbf{U}^\top \mathbf{b},\xi\big)
=
\mathbf{U}^\top \bar\Psi(\mathbf{a},\mathbf{b},\xi)
\qquad
\text{for all } \mathbf{U}\in O(d).
\label{eq:psi_equiv}
\end{equation}

\paragraph{Characterization.} We characterize exactly which $\Psi$
satisfy Eq.~\eqref{eq:psi_equiv}.

\begin{proposition}[Equivariant message functions]
\label{prop:psi_equiv}
A continuous map $\Psi:\RR^d\!\times\!\RR^d\!\times\!\Xi \to \RR^d$ satisfies Eq.~\eqref{eq:psi_equiv} if and only if for $d\ge 2$ it admits the representation
\begin{equation}
\Psi(\mathbf{a},\mathbf{b},\xi) = \alpha\big(\mathcal{I}(\mathbf{a},\mathbf{b}),\xi\big)\,\mathbf{a} + \beta\big(\mathcal{I}(\mathbf{a},\mathbf{b}),\xi\big)\,\mathbf{b},
\label{eq:invariant_form}
\end{equation}
where $\mathcal{I}(\mathbf{a},\mathbf{b}) := (\|\mathbf{a}\|^2,\|\mathbf{b}\|^2,\langle\mathbf{a},\mathbf{b}\rangle)$ is the
$O(d)$-invariant scalar triple of $(\mathbf{a},\mathbf{b})$ and $\alpha,\beta$
are continuous scalar functions on the Gram cone
$\{(r,s,c):r\ge0,\ s\ge0,\ c^2\le rs\}$ (with any continuous extension to
$\RR^3$ when convenient).
\end{proposition}

\begin{proof}
($\Leftarrow$) For $\Psi$ of the form \eqref{eq:invariant_form},
$\mathcal{I}(\mathbf{U}^\top \mathbf{a},\mathbf{U}^\top \mathbf{b}) = \mathcal{I}(\mathbf{a},\mathbf{b})$ since each entry is an
$O(d)$-invariant. Hence $\alpha,\beta$ depend only on $\mathcal{I}(\mathbf{a},\mathbf{b})$ even
under the substitution $(\mathbf{a},\mathbf{b})\to(\mathbf{U}^\top \mathbf{a},\mathbf{U}^\top \mathbf{b})$, and
$\Psi(\mathbf{U}^\top \mathbf{a},\mathbf{U}^\top \mathbf{b},\xi)=\alpha\,\mathbf{U}^\top \mathbf{a}+\beta\,\mathbf{U}^\top \mathbf{b}=\mathbf{U}^\top(\alpha\mathbf{a}+\beta\mathbf{b})=\mathbf{U}^\top\bar\Psi(\mathbf{a},\mathbf{b},\xi)$ for
$\bar\Psi=\Psi$.

($\Rightarrow$) Conversely, assume Eq.~\eqref{eq:psi_equiv}. Fix
generic $(\mathbf{a},\mathbf{b})$ with $\mathbf{a}\not\parallel\mathbf{b}$. The stabilizer
$\mathrm{Stab}(\mathbf{a},\mathbf{b}) = \{\mathbf{U}\in O(d):\mathbf{U}\mathbf{a}=\mathbf{a},\mathbf{U}\mathbf{b}=\mathbf{b}\}$ acts trivially on
$\mathrm{span}(\mathbf{a},\mathbf{b})$ and as $O(d-2)$ on its orthogonal complement.
Eq.~\eqref{eq:psi_equiv} forces $\mathbf{U}^\top\bar\Psi(\mathbf{a},\mathbf{b},\xi) = \bar\Psi(\mathbf{U}^\top \mathbf{a},\mathbf{U}^\top \mathbf{b},\xi) = \bar\Psi(\mathbf{a},\mathbf{b},\xi)$
for $\mathbf{U}\in\mathrm{Stab}(\mathbf{a},\mathbf{b})$, so $\bar\Psi(\mathbf{a},\mathbf{b},\xi)$ is fixed by $O(d-2)$ on
$\mathrm{span}(\mathbf{a},\mathbf{b})^\perp$. The only $O(d-2)$-fixed vector is the zero
vector, so $\bar\Psi(\mathbf{a},\mathbf{b},\xi)\in\mathrm{span}(\mathbf{a},\mathbf{b})$, i.e.\
$\bar\Psi(\mathbf{a},\mathbf{b},\xi) = \alpha\,\mathbf{a}+\beta\,\mathbf{b}$ for scalars
$\alpha,\beta$ depending on $(\mathbf{a},\mathbf{b},\xi)$. Equivariance then forces
	$(\alpha,\beta)$ to depend only on $O(d)$-orbits of $(\mathbf{a},\mathbf{b})$.
	The Gram map
	$\mathcal{I}(\mathbf{a},\mathbf{b})=(\|\mathbf{a}\|^2,\|\mathbf{b}\|^2,\langle\mathbf{a},\mathbf{b}\rangle)$
	is a complete orbit invariant for pairs of vectors: two pairs have the
	same Gram triple if and only if an orthogonal map sends one pair to the
	other, by extending the induced isometry between their spans to all of
	$\RR^d$. Because $O(d)$ is compact, the orbit quotient is Hausdorff and
	this Gram map induces a homeomorphism from
	$(\RR^d\times\RR^d)/O(d)$ onto the closed Gram cone
	$\{(r,s,c): r\ge0,\ s\ge0,\ c^2\le rs\}$. Hence every continuous
	orbit-invariant scalar, including the coefficient functions
	$\alpha$ and $\beta$ on the non-collinear stratum, factors continuously
	through $\mathcal{I}$. The parallel and zero-vector cases follow from the
	same Gram-orbit characterization and continuity. This yields
	Eq.~\eqref{eq:invariant_form}.
\end{proof}

\paragraph{\tsnn{}'s modules lie outside this class.} Eq.~\eqref{eq:invariant_form}
is restrictive: every output coordinate must be an invariant-modulated
linear combination of the two vector inputs. \tsnn{}'s message function
$\MLP_{\text{msg}}$ (Eq.~\eqref{eq:message}) and the GRU update
(Eq.~\eqref{eq:gru_update}) are unconstrained learnable maps; in
particular, the affine layer $\mathbf{a}\mapsto \mathbf{W}\mathbf{a}$ inside an MLP satisfies
Eq.~\eqref{eq:invariant_form} only if $\mathbf{W} = c\,\mathbf{I}_d$ for some
$c\in\RR$ (taking $\mathbf{b}=\mathbf{0}$ in Eq.~\eqref{eq:invariant_form} forces
$\mathbf{W}\mathbf{a}=\alpha(\|\mathbf{a}\|^2)\mathbf{a}$ for all $\mathbf{a}\in\mathbb{R}^d$, which by linearity in
$\mathbf{a}$ requires $\alpha$ constant and $\mathbf{W}=\alpha\mathbf{I}_d$). Since
$\MLP_{\text{msg}}$ and the GRU's input/hidden weight matrices are
arbitrary trained parameters, with overwhelming probability $\mathbf{W}\notin\{c\mathbf{I}_d:c\in\RR\}$, and the modules fail Eq.~\eqref{eq:invariant_form}.
Therefore the synchronized-coordinate reduction does not propagate
through them: the formal guarantees of \S\ref{sec:theory} are
operator-, transport-, and decoder-level, while the message and
recurrent updates remain consistent with the local-frame view of LoCS
\citep{kofinas2021locs} rather than fully gauge-equivariant.

\subsection{Idealized Temporal-WL Correspondence}
\label{app:proof_twl}

\begin{proposition}[Idealized temporal-WL correspondence]
\label{prop:twl}
Consider an idealized variant of \tsnn{} in which (i) the initial state map is
injective in the initial node/time attributes used by temporal 1-WL, (ii) the
full temporal neighborhood is aggregated without truncation, (iii) the message
map is injective in
$(\mathbf{h}_u^{(\ell)}, \mathbf{Q}_{uv}, r_{uv}, t_{uv}, \mathbf{x}_{uv})$
under the equality relation used by temporal 1-WL, (iv) the multiset
aggregation is injective, and (v) the node update map is injective in its own
state and aggregated incoming multiset. Then after $L$ message-passing rounds,
the hidden state of any node is at least as discriminative as the depth-$L$
temporal Weisfeiler--Leman (temporal 1-WL) color on the corresponding temporal
computation tree \citep{souza2022pint}.
\end{proposition}

\begin{proof}
Let $c_v^{(\ell)}(t)$ denote the temporal 1-WL color of node $v$ at time $t$ after $\ell$ refinement rounds, defined recursively via
\[
c_v^{(\ell+1)}(t)
=
\textsc{Hash}\!\Big(
c_v^{(\ell)}(t),\;
\{\!\!\{\, (c_u^{(\ell)}(t_{uv}), r_{uv}, t_{uv}, \mathbf{x}_{uv}) : (u,v,t_{uv}) \in \mathcal{E}_{<t}\,\}\!\!\}
\Big),
\]
where $\textsc{Hash}$ is an injection on its argument \citep{souza2022pint}.

Write $\Phi^{(\ell)}(v,t)$ for the hidden state of node $v$ at round $\ell$ produced by the idealized \tsnn{}. The desired refinement statement
\[
c_v^{(\ell)}(t) \neq c_{v'}^{(\ell)}(t') \;\Longrightarrow\; \Phi^{(\ell)}(v,t) \neq \Phi^{(\ell)}(v',t')
\]
is the contrapositive of
\begin{equation}
\Phi^{(\ell)}(v,t) = \Phi^{(\ell)}(v',t') \;\Longrightarrow\; c_v^{(\ell)}(t) = c_{v'}^{(\ell)}(t'),
\label{eq:twl_induction}
\end{equation}
which we establish by induction on $\ell$, mirroring the GIN-style argument of \citet{xu2019gin} and its temporal extension in \citet{souza2022pint}. The base case $\ell=0$ follows from the injective initial-state assumption, so $\Phi^{(0)}(v,t)=\Phi^{(0)}(v',t')$ implies equality of the initial attributes used by temporal 1-WL, hence $c_v^{(0)}(t)=c_{v'}^{(0)}(t')$.

For the inductive step, assumptions (ii)--(v) imply that $\Phi^{(\ell+1)}(v,t)$ is an injective function of the pair
\[
\big(\Phi^{(\ell)}(v,t),\;\{\!\!\{(\Phi^{(\ell)}(u,t_{uv}), \mathbf{Q}_{uv}, r_{uv}, t_{uv}, \mathbf{x}_{uv}) : (u,v,t_{uv}) \in \mathcal{E}_{<t}\}\!\!\}\big).
\]
Hence $\Phi^{(\ell+1)}(v,t) = \Phi^{(\ell+1)}(v',t')$ forces equality of these two pairs: the own-states agree, $\Phi^{(\ell)}(v,t)=\Phi^{(\ell)}(v',t')$, and the neighbor multisets coincide elementwise under some matching of edges. By the inductive hypothesis applied to the own-states, $c_v^{(\ell)}(t)=c_{v'}^{(\ell)}(t')$. Applying the inductive hypothesis to each matched neighbor pair, equality of the $\Phi^{(\ell)}$-multiset forces equality of the projected WL multiset
\[
\{\!\!\{(c_u^{(\ell)}(t_{uv}), r_{uv}, t_{uv}, \mathbf{x}_{uv})\}\!\!\}.
\]
(Note that the \tsnn{} multiset carries the additional coordinate $\mathbf{Q}_{uv}$, so its equality is strictly stronger than equality of the WL multiset; we use only the implied projection.) Since $\textsc{Hash}$ is injective, equality of its inputs forces $c_v^{(\ell+1)}(t)=c_{v'}^{(\ell+1)}(t')$, completing the induction. The contrapositive yields the refinement inequality, so any two temporal computation trees that temporal 1-WL separates are also separated by the idealized \tsnn{}.
\end{proof}

\begin{remark}[Scope of Proposition~\ref{prop:twl}]
The proposition does not apply verbatim to the experimental model. The practical architecture truncates neighborhoods to $K_{\text{nbr}}$ most recent neighbors and trains with truncated BPTT. We interpret Proposition~\ref{prop:twl} as an idealized capacity statement showing that the transport-augmented message-passing framework is sufficiently expressive in principle.
\end{remark}

\section{Algorithmic Details}
\label{app:algorithms}

Algorithm~\ref{alg:tsnn} in the main text states the per-event pipeline at the level of equation references. Several of its lines, however, abbreviate non-trivial procedures: the sequential Householder reflections behind every transport call (\S\ref{app:alg_transport}), the coordinate-style realization of the sheaf-diffusion update of Eq.~\eqref{eq:diffusion_local} (\S\ref{app:alg_diffusion}), the optional sheaf-consistent neighbor attention used in some ablation rows (\S\ref{app:alg_nbrattn}), the EdgeBank sliding-window data structure that supplies the auxiliary score in Eq.~\eqref{eq:edgebank_mixture} (\S\ref{app:alg_edgebank}), and the streaming training epoch with truncated backpropagation-through-time (TBPTT) that wraps the per-event update (\S\ref{app:alg_train}). \S\ref{app:complexity} then aggregates the per-event, per-epoch, and per-pool memory costs of these procedures and ties them to the wall-clock and peak-VRAM measurements reported in \S\ref{app:scalability}.

Throughout this section, vectors are written in the local frame of their owning node. We write $\textsc{SheafTransport}(\mathbf{F}_a, \mathbf{F}_b, \mathbf{h}_a)$ for the procedure that takes a vector $\mathbf{h}_a$ expressed in source node $a$'s local frame and re-expresses it in target node $b$'s local frame. In the main notation this source-to-target call applies $\mathbf{Q}_{ba}(t) = \mathbf{U}_b(t)^{\transpose} \mathbf{U}_a(t)$; equivalently, the map $\mathbf{Q}_{uv}$ transports from $v$ to $u$, and $\mathbf{Q}_{ba} = \mathbf{Q}_{ab}^{\transpose}$. The chunk-local differentiable cache $\mathcal{C}$ and the persistent host store $\mathcal{S}$ are introduced in \S\ref{app:alg_train} and used throughout.

\subsection{Sheaf Transport via Sequential Householder Reflections}
\label{app:alg_transport}

For a source-to-target call $\textsc{SheafTransport}(\mathbf{F}_a,\mathbf{F}_b,\mathbf{h}_a)$, the transport $\mathbf{Q}_{ba} = \mathbf{U}_b^{\transpose} \mathbf{U}_a$ is composed of $2k$ Householder reflections, one per row of $\mathbf{F}_a$ and $\mathbf{F}_b$. Each reflection $H(\mathbf{f}) = \mathbf{I} - 2\hat{\mathbf{f}}\hat{\mathbf{f}}^{\transpose}$ acts on a vector in $\mathcal{O}(d)$ time without ever materializing the $d \times d$ Householder matrix; by composing reflections sequentially, Algorithm~\ref{alg:transport} applies $\mathbf{Q}_{ba}$ in $\mathcal{O}(kd)$ time and $\mathcal{O}(d)$ working memory. Under the convention of Eq.~\eqref{eq:householder_product}, applying row reflectors in the order $1,\ldots,k$ realizes $\mathbf{U}$, while applying them in the reverse order realizes $\mathbf{U}^{\transpose}$ because each reflector is symmetric. The conditional on line~\ref{alg:transport:guard} is the numerical safeguard $\varepsilon$ from $H_\varepsilon$ defined in \S\ref{sec:sheaf_construction}: a Householder vector with near-zero norm yields the identity reflection and is silently skipped. As discussed in \S\ref{sec:sheaf_construction}, random initialization combined with the additive frame update keeps $\min_{w,i}\|\mathbf{f}_w^{(i)}\|$ comfortably above $\varepsilon = 10^{-6}$ throughout training, so the guard is a precaution against degenerate edge cases rather than the operating regime.

\begin{algorithm}[htbp]
\caption{Sheaf transport via sequential Householder reflections.}
\label{alg:transport}
\begin{algorithmic}[1]
\REQUIRE Source frame $\mathbf{F}_a \in \RR^{k \times d}$, target frame $\mathbf{F}_b \in \RR^{k \times d}$, vector $\mathbf{h} \in \RR^d$, threshold $\varepsilon$.
\ENSURE $\mathbf{h}'$ expressing $\mathbf{h}$ in $b$'s local frame, equal to $\mathbf{Q}_{ba}\,\mathbf{h} = \mathbf{U}_b^{\transpose} \mathbf{U}_a \mathbf{h}$.
\STATE $\mathbf{y} \leftarrow \mathbf{h}$
\FOR{$i = 1,\ldots,k$}
    \STATE $\mathbf{f} \leftarrow (\mathbf{F}_a)_{i,:}$;\quad $n \leftarrow \|\mathbf{f}\|$ \hfill $\triangleright\ \mathbf{y}\leftarrow\mathbf{U}_a\mathbf{h}=H_\varepsilon(\mathbf{f}_a^{(k)})\cdots H_\varepsilon(\mathbf{f}_a^{(1)})\mathbf{h}$
    \IF{$n > \varepsilon$} \label{alg:transport:guard}
        \STATE $\mathbf{y} \leftarrow \mathbf{y} - \tfrac{2(\mathbf{f}^{\transpose}\mathbf{y})}{n^2}\,\mathbf{f}$ \hfill $\triangleright\ \mathbf{y}\leftarrow H_\varepsilon(\mathbf{f})\,\mathbf{y},\ \mathcal{O}(d)$
    \ENDIF
\ENDFOR
\FOR{$i = k,\ldots,1$}
    \STATE $\mathbf{f} \leftarrow (\mathbf{F}_b)_{i,:}$;\quad $n \leftarrow \|\mathbf{f}\|$ \hfill $\triangleright\ \mathbf{y}\leftarrow\mathbf{U}_b^{\transpose}\mathbf{y}=H_\varepsilon(\mathbf{f}_b^{(1)})\cdots H_\varepsilon(\mathbf{f}_b^{(k)})\mathbf{y}$
    \IF{$n > \varepsilon$}
        \STATE $\mathbf{y} \leftarrow \mathbf{y} - \tfrac{2(\mathbf{f}^{\transpose}\mathbf{y})}{n^2}\,\mathbf{f}$
    \ENDIF
\ENDFOR
\STATE \textbf{return} $\mathbf{y}\;=\;\mathbf{Q}_{ba}\mathbf{h}=\mathbf{U}_b^{\transpose}\mathbf{U}_a\mathbf{h}$
\end{algorithmic}
\end{algorithm}

In a vectorized implementation the outer loop over reflections is the only sequential dependency; the inner Householder step is fully parallel across the batch dimension, which we exploit when scoring $K_{\text{cand}} = K_{\text{neg}}{+}1$ candidates simultaneously. Gradients flow through $\hat{\mathbf{f}}_i$ and the scalar projection $\hat{\mathbf{f}}_i^{\transpose} \mathbf{y}$, so the frame parameters $\mathbf{F}_v$ receive a direct geometric signal from every transport-aligned distance and dot product in Eq.~\eqref{eq:score_features}; this is the differentiability premise underlying the smoothness analysis of Lemma~\ref{lem:houshlip}.

\subsection{Coordinate-Style Sheaf Laplacian Diffusion}
\label{app:alg_diffusion}

Equation~\eqref{eq:diffusion_local} prescribes one step of feature-scaled normalized sheaf diffusion at each active node of the event-local active graph. Globalizing the equation by materializing a $|\mathcal{A}_t|d \times |\mathcal{A}_t|d$ block sheaf Laplacian is unnecessary: Algorithm~\ref{alg:diffuse} realizes the same full-active operator by sparse edge-style accumulation over the active adjacency lists $\{\mathcal{N}^{\mathrm{act}}_a(t):a\in\mathcal{A}_t\}$. Thus every stalk in $\mathcal{A}_t$ is diffused and committed after the event, while only the current endpoints receive frame, GRU, last-time, and neighbor-buffer updates.

A subtlety arises during training. Some active nodes may have been updated earlier in the same TBPTT chunk, while others enter $\mathcal{A}_t$ only as historical neighbors of the current endpoints. If every active stalk is read from the host store $\mathcal{S}$, the in-chunk gradient route through full-active diffusion is lost. Algorithm~\ref{alg:diffuse} therefore uses the chunk-local cache $\mathcal{C}$ whenever an active node already has a differentiable in-chunk state, and lazily inserts missing active nodes from $\mathcal{S}$ with detached initial states. This is the implementation-level realization of the full-active metric-gradient update analyzed in Theorem~\ref{thm:feature_scaled}.

\begin{algorithm}[htbp]
\caption{Full-active coordinate-style sheaf Laplacian diffusion.}
\label{alg:diffuse}
\begin{algorithmic}[1]
\REQUIRE Active vertex set $\mathcal{A}_t$; active adjacency lists $\mathcal{N}^{\mathrm{act}}_a(t)$ for $a\in\mathcal{A}_t$; cache/store states $(\mathbf{h}_a,\mathbf{F}_a)$ with post-GRU endpoint stalks already installed in $\mathcal{C}$; learned diagonal $\mathbf{D}_\theta = \softplus(\boldsymbol{\theta}_{\diag}) \in \RR^{d}_{>0}$; step size $\eta$; iterations $K$; chunk cache $\mathcal{C}$ (training only); store $\mathcal{S}$.
\ENSURE Diffused stalks $\{\mathbf{h}_a^{(K)}:a\in\mathcal{A}_t\}$ in their owning local frames.
\FOR{each $a\in\mathcal{A}_t$}
    \IF{$a \notin \mathcal{C}$}
        \STATE $(\mathbf{h}_a,\mathbf{F}_a) \leftarrow \mathrm{detach}(\mathcal{S}[a])$;\quad insert $(\mathbf{h}_a,\mathbf{F}_a)$ into $\mathcal{C}$
    \ENDIF
    \STATE $\mathbf{h}_a^{(0)} \leftarrow \mathcal{C}[a].\mathbf{h}$;\quad $\mathbf{F}_a \leftarrow \mathcal{C}[a].\mathbf{F}$
\ENDFOR
\FOR{$\ell = 0, 1, \ldots, K-1$}
    \FOR{each active node $a \in \mathcal{A}_t$}
        \STATE $\mathcal{N} \leftarrow \mathcal{N}^{\mathrm{act}}_a(t)$
        \IF{$\mathcal{N} = \emptyset$}
            \STATE $\mathbf{h}_a^{(\ell+1)} \leftarrow \mathbf{h}_a^{(\ell)}$;\quad \textbf{continue} \hfill $\triangleright$ no neighbors $\Rightarrow$ identity step
        \ENDIF
        \STATE $\mathbf{g}_a \leftarrow \mathbf{0} \in \RR^d$
        \FOR{each $b \in \mathcal{N}$}
            \STATE $\mathbf{g}_a \leftarrow \mathbf{g}_a + \bigl(\mathbf{h}_a^{(\ell)} - \mathbf{U}_a^{\transpose}\mathbf{U}_b\,\mathbf{h}_b^{(\ell)}\bigr)$ \hfill $\triangleright\ \mathbf{Q}_{ab}\mathbf{h}_b^{(\ell)}$ via Alg.~\ref{alg:transport}, residual in $a$'s frame
        \ENDFOR
        \STATE $\mathbf{h}_a^{(\ell+1)} \leftarrow \mathbf{h}_a^{(\ell)} - \tfrac{\eta}{|\mathcal{N}|}\,\mathbf{D}_\theta\,\mathbf{g}_a$ \hfill $\triangleright$ Eq.~\eqref{eq:diffusion_local}
    \ENDFOR
\ENDFOR
\STATE \textbf{return} $\{\mathbf{h}_a^{(K)}:a\in\mathcal{A}_t\}$
\end{algorithmic}
\end{algorithm}

The inner loop runs once per directed active adjacency entry. Because $\mathcal{A}_t$ is built from the two capped endpoint histories, the active edge multiset has $\mathcal{O}(K_{\text{nbr}})$ directed entries even though up to $2K_{\text{nbr}}+2$ vertices may be touched. Each transport costs $\mathcal{O}(kd)$, so a full diffusion costs $\mathcal{O}(K\,K_{\text{nbr}}\,k\,d)$ per event. In practice, we set $K \in \{1, 2, 3, 4\}$. The eval-time path differs only in that the chunk cache $\mathcal{C}$ is initialized from detached store reads, mirroring the deployment-time setting and preserving the predict-then-update causality of Algorithm~\ref{alg:tsnn}.

\subsection{Sheaf-Consistent Multi-Head Neighbor Attention}
\label{app:alg_nbrattn}

When the optional neighbor-attention head is enabled (used in Track-B datasets where structural neighborhoods carry strong signal beyond the geometric residual), the source representation $\mathbf{h}_u$ is augmented before scoring by attending over the up-to-$K$ most recent transport-aligned neighbors. The key sheaf-respecting design choice is that transport is applied before attention rather than after: each neighbor's stalk $\mathbf{h}_w$ is first re-expressed in $u$'s local frame via $\mathbf{Q}_{uw}\,\mathbf{h}_w$, and only then does standard multi-head dot-product attention compare vectors that already live in a single coordinate system. Without this transport, the attention key/value projections would mix coordinates from different fibers and defeat the purpose of node-local frames.

\begin{algorithm}[htbp]
\caption{Sheaf-consistent multi-head neighbor attention.}
\label{alg:nbrattn}
\begin{algorithmic}[1]
\REQUIRE Query state $(\mathbf{h}_u, \mathbf{F}_u)$; up to $K$ neighbors $\{(\mathbf{h}_{w_i}, \mathbf{F}_{w_i}, \delta t_i, r_i)\}_{i=1}^{|\mathcal{N}|}$ from $\mathcal{N}^{\mathrm{act}}_u(t)$; learned $\mathbf{W}_Q, \mathbf{W}_K, \mathbf{W}_V, \mathbf{W}_O \in \RR^{d \times d}$; relation embedding $\mathbf{R}$; temporal projection $\mathbf{W}_t$; head count $H$ with $d_H = d/H$; null token $\boldsymbol{\nu}_0$.
\ENSURE Refined query $\mathbf{h}_u'$.
\IF{$|\mathcal{N}| = 0$}
    \STATE \textbf{return} $\mathrm{LayerNorm}\bigl(\mathbf{h}_u + \mathbf{W}_O \mathbf{W}_V \boldsymbol{\nu}_0\bigr)$ \hfill $\triangleright$ degenerate path
\ENDIF
\FOR{$i = 1, \ldots, |\mathcal{N}|$}
    \STATE $\tilde{\mathbf{h}}_i \leftarrow \mathbf{U}_u^{\transpose}\,\mathrm{detach}(\mathbf{U}_{w_i})\,\mathrm{detach}(\mathbf{h}_{w_i})$ \hfill $\triangleright\ \mathbf{Q}_{u w_i}\mathbf{h}_{w_i}$ via Alg.~\ref{alg:transport}, aligned to $u$'s frame
    \STATE $\tilde{\mathbf{h}}_i \leftarrow \tilde{\mathbf{h}}_i + \mathbf{W}_t\,\phi(\delta t_i) + \mathbf{R}[r_i]$ \hfill $\triangleright$ temporal $+$ relational
\ENDFOR
\STATE Stack into $\tilde{\mathbf{H}} \in \RR^{|\mathcal{N}| \times d}$
\STATE $\mathbf{Q} \leftarrow \mathbf{W}_Q \mathbf{h}_u$;\quad $\mathbf{K} \leftarrow \tilde{\mathbf{H}} \mathbf{W}_K^{\transpose}$;\quad $\mathbf{V} \leftarrow \tilde{\mathbf{H}} \mathbf{W}_V^{\transpose}$
\STATE Reshape $\mathbf{Q}, \mathbf{K}, \mathbf{V}$ into $H$ heads of width $d_H$
\STATE $\boldsymbol{\alpha}^{(h)} \leftarrow \softmx_i\!\Bigl(d_H^{-1/2}\, \mathbf{Q}^{(h)}\,(\mathbf{K}^{(h)}_{i,:})^{\transpose}\Bigr)$ for each head $h$
\STATE $\mathbf{a} \leftarrow \mathrm{Concat}_{h=1}^{H}\!\Bigl( \textstyle\sum_i \alpha^{(h)}_i\,\mathbf{V}^{(h)}_{i,:} \Bigr)$
\STATE \textbf{return} $\mathrm{LayerNorm}\bigl(\mathbf{h}_u + \mathbf{W}_O \mathbf{a}\bigr)$
\end{algorithmic}
\end{algorithm}

The temporal modulation $\mathbf{W}_t \phi(\delta t_i)$ uses the same sinusoidal encoding as Eq.~\eqref{eq:time_encoding} but with half the dimension; the relational modulation $\mathbf{R}[r_i]$ is shared with the per-relation embedding used elsewhere in the model (\S\ref{sec:scoring}). The constraint $H \mid d$ is enforced at construction time. Attention dropout and the residual$+$LayerNorm wrap are standard transformer-block components; we omit them from the pseudocode for brevity but apply them in the implementation. Crucially, every neighbor state read in line~6 detaches gradients (neighbor states are not reshaped by the current event's loss); only the target-frame part of the transport contributes to $\nabla_{\mathbf{F}_u}\mathcal{L}$.

\subsection{EdgeBank Sliding-Window Memory}
\label{app:alg_edgebank}

EdgeBank \citep{poursafaei2022edgebank} maintains, for each oriented relation triple $(u, v, r)$, the count of past interactions and the timestamp of the most recent occurrence. Despite being a purely memory-based heuristic with no learnable parameters, it is competitive with several learned temporal models on standard benchmarks \citep{poursafaei2022edgebank}; this motivates its inclusion as the auxiliary score in Eq.~\eqref{eq:edgebank_mixture}. Algorithm~\ref{alg:edgebank} states the three operations the model invokes per event: \textsc{Score} (read at prediction time), \textsc{Observe} (called only after the per-event update commits, preserving causality), and the internal \textsc{Expire} that maintains the sliding-window invariant. Two memory modes are supported: \textsc{unlimited}, which retains every past edge, and \textsc{fixed\_time\_window}, which expires entries older than $\tau_{\mathrm{w}}$ relative to the current event time. The mode is selected per dataset on validation. On Track~A, tgbl-wiki and tgbl-review use \textsc{unlimited} mode and thgl-software disables EdgeBank entirely; on Track~B, the default mode is \textsc{unlimited}, which the per-dataset \emph{EdgeBank window ratio} row optionally overrides: a value of $0$ means \textsc{unlimited}, while a positive value $\rho>0$ activates \textsc{fixed\_time\_window} with $\tau_{\mathrm{w}}=\rho\cdot\bar{\delta t}$, where $\bar{\delta t}$ is the mean inter-event gap on the training stream.

\begin{algorithm}[htbp]
\caption{EdgeBank sliding-window memory.}
\label{alg:edgebank}
\begin{algorithmic}[1]
\REQUIRE Counts $C[\,\cdot\,] = 0$, last-time map $T[\,\cdot\,]$, deque $\mathcal{Q}$ of $(t, k)$ entries (FIFO by $t$), window $\tau_{\mathrm{w}}$, mode $m \in \{\textsc{unlimited}, \textsc{fixed\_time\_window}\}$.
\STATE \textbf{procedure} \textsc{Expire}($t$) \hfill $\triangleright$ amortized $\mathcal{O}(1)$ per event
\IF{$m = \textsc{fixed\_time\_window}$}
    \WHILE{$\mathcal{Q}$ non-empty \textbf{and} $t - \mathcal{Q}.\mathrm{front}.t > \tau_{\mathrm{w}}$}
        \STATE $(t', k') \leftarrow \mathcal{Q}.\mathrm{pop\_front}()$;\quad $C[k'] \leftarrow C[k'] - 1$
        \IF{$C[k'] \leq 0$} \STATE remove $k'$ from $C$ and $T$ \ENDIF
    \ENDWHILE
\ENDIF
\STATEx
\STATE \textbf{procedure} \textsc{Observe}($u, v, r, t$) \hfill $\triangleright$ called after per-event update commits
\STATE $k \leftarrow (u, v, r)$
\IF{$m = \textsc{fixed\_time\_window}$} \STATE \textsc{Expire}($t$);\quad push $(t, k)$ onto $\mathcal{Q}$ \ENDIF
\STATE $C[k] \leftarrow C[k] + 1$;\quad $T[k] \leftarrow t$
\STATEx
\STATE \textbf{procedure} \textsc{Score}($u, v, r, t$) \hfill $\triangleright$ called before update from Eq.~\eqref{eq:edgebank_mixture}
\STATE $k \leftarrow (u, v, r)$;\quad \textsc{Expire}($t$)
\STATE \textbf{return} $\log\!\bigl(1 + C[k]\bigr)$ \hfill $\triangleright$ returns $0$ if $C[k]=0$
\end{algorithmic}
\end{algorithm}

All three operations are amortized $\mathcal{O}(1)$ per event (deque push/pop, hash insert/lookup), with bursts of up to $\mathcal{O}(|\mathcal{Q}_{\mathrm{old}}|)$ during \textsc{Expire} that amortize across the rest of the stream. Two ordering invariants are essential. First, \textsc{Score} is called before \textsc{Observe} commits the current event, so the EdgeBank mixture in Eq.~\eqref{eq:edgebank_mixture} cannot peek at the present interaction, which preserves the predict-then-update causality of Algorithm~\ref{alg:tsnn}. Second, the optional recency feature in Eq.~\eqref{eq:score_features}, $\log(1 + (t - T[k]))$, is computed from the same $T$ map and provides a recency signal independent of count magnitude.

\subsection{Streaming Training Epoch with Truncated BPTT}
\label{app:alg_train}

Algorithm~\ref{alg:train} states the outer loop that drives Algorithm~\ref{alg:tsnn} over a full training stream. Two design pressures conflict: gradients should flow across consecutive events so that the GRU, frame-update, and full-active diffusion operators learn long-range dependencies, but storing the full computation graph for a multi-million-edge sequence is infeasible. We resolve this with a truncated BPTT scheme: events are processed in fixed-length chunks of $L_c$ consecutive interactions, gradient computation is bounded to one chunk, and the chunk-final states are detached and pushed back to the host store before the next chunk begins. Inside a chunk, endpoint states are cached before scoring, and the cache is lazily expanded with every node in each event-local active graph $\mathcal{A}_t$ before diffusion. Nodes already in $\mathcal{C}$ retain differentiable in-chunk state; newly inserted active nodes are pulled from the host store $\mathcal{S}$ with detached initial states. This keeps the autograd graph proportional to $L_c$ events rather than the full epoch, while still allowing every loss term in $\mathcal{L}_{\text{CE}} + \lambda_{\text{geo}}\mathcal{L}_{\text{geo}} + \lambda_{\text{bias}} \mathcal{L}_{\text{bias}} + \lambda_e \mathcal{L}_{\text{energy}} + \lambda_s \mathcal{L}_{\text{smooth}}$ (Eqs.~\eqref{eq:loss_ce}--\eqref{eq:loss_energy}) to backpropagate through transports, frame updates, and full-active in-chunk diffusion.

\begin{algorithm}[htbp]
\caption{One streaming training epoch with truncated BPTT.}
\label{alg:train}
\begin{algorithmic}[1]
\REQUIRE Sorted training stream $\mathcal{E}_{\mathrm{tr}}$; chunk length $L_c$; negatives per positive $K_{\text{neg}}$; loss weights $(\lambda_{\text{geo}}, \lambda_{\text{bias}}, \lambda_e, \lambda_s)$; optimizer $\mathrm{Opt}$, AMP scaler $\mathcal{S}_{\mathrm{amp}}$, gradient clip $g_{\max}$.
\STATE Reset state store $\mathcal{S}$, neighbor buffers $\mathcal{N}^{\mathrm{act}}$, EdgeBank $\mathcal{E}\mathcal{B}$, sampler $\mathcal{N}_{-}$
\FOR{each chunk $[a, b)$ of $L_c$ consecutive events in $\mathcal{E}_{\mathrm{tr}}$}
    \STATE $\mathcal{A}_{\mathrm{end}} \leftarrow \bigcup_{i \in [a,b)} \{u_i, v_i\}$ \hfill $\triangleright$ endpoint seed set for the chunk
    \STATE Pull $(\mathbf{h}, \mathbf{F}, t_{\mathrm{last}})_a$ from $\mathcal{S}$ for $a \in \mathcal{A}_{\mathrm{end}}$; build cache $\mathcal{C}$ as differentiable tensors
    \STATE $\mathcal{L}_{\mathrm{chunk}} \leftarrow 0$
    \FOR{$i \in [a, b)$}
        \STATE $(u, v, t, r, \mathbf{x}) \leftarrow \mathcal{E}_{\mathrm{tr}}[i]$
        \STATE $\mathcal{N}_{-}^{(i)} \leftarrow \mathcal{N}_{-}.\textsc{sample}(u, v, r, K_{\text{neg}})$ \hfill $\triangleright$ track-specific sampler (\S\ref{sec:training})
        \STATE Assemble candidate states for $\{v\} \cup \mathcal{N}_{-}^{(i)}$ from $\mathcal{C}$ (active, differentiable) or $\mathcal{S}$ (detached); detach all negatives' rows
        \STATE $(s_{\text{EB}}, \delta t_{\text{EB}}) \leftarrow \mathcal{E}\mathcal{B}.\textsc{features}(u, \cdot, r, t)$;\quad compute CN features
        \STATE $\boldsymbol{s} \leftarrow \textsc{ScoreCandidates}$ via Eqs.~\eqref{eq:rel_frame_bias}--\eqref{eq:edgebank_mixture}
        \STATE $\mathcal{L}_{\mathrm{chunk}} \mathrel{+}= \mathcal{L}_{\text{CE}}(\boldsymbol{s}) + \lambda_{\text{geo}}\mathcal{L}_{\text{geo}}(\boldsymbol{s}) + \lambda_{\text{bias}} \mathcal{L}_{\text{bias}}(\boldsymbol{s})$
        \STATE Run Algorithm~\ref{alg:tsnn} Steps~2--4 on $\mathcal{C}[u], \mathcal{C}[v]$ \hfill $\triangleright$ frame update, carry-over, message, GRU
        \STATE Build $\mathcal{A}_t$ and insert any missing active nodes into $\mathcal{C}$ from detached $\mathcal{S}$ reads
        \STATE $\{\mathbf{h}_a:a\in\mathcal{A}_t\} \leftarrow \textsc{Diffuse}(\mathcal{A}_t,\ldots)$ via Algorithm~\ref{alg:diffuse}
        \STATE $\mathcal{L}_{\mathrm{chunk}} \mathrel{+}= \lambda_e \mathcal{L}_{\text{energy}} + \lambda_s \mathcal{L}_{\text{smooth}}$
        \STATE $\mathcal{C}[\mathcal{A}_t] \leftarrow$ post-diffusion active states via autograd-preserving \textsc{IndexCopy} \label{alg:train:idxcopy}
        \STATE $\mathcal{N}^{\mathrm{act}}_u \mathrel{{+}{=}} (v, t, r, \mathbf{x})$;\quad $\mathcal{N}^{\mathrm{act}}_v \mathrel{{+}{=}} (u, t, r, \mathbf{x})$;\quad $\mathcal{E}\mathcal{B}.\textsc{Observe}(u, v, r, t)$;\quad $\mathcal{N}_{-}.\textsc{update}(u, v, r)$
    \ENDFOR
    \STATE $\mathcal{L}_{\mathrm{chunk}} \leftarrow \mathcal{L}_{\mathrm{chunk}} / L_c$
    \STATE $\mathrm{Opt}.\mathrm{zero\_grad}()$;\quad $\mathcal{S}_{\mathrm{amp}}.\mathrm{scale}(\mathcal{L}_{\mathrm{chunk}}).\mathrm{backward}()$
    \STATE Unscale gradients;\quad clip $\|\nabla\|_2 \le g_{\max}$;\quad step $\mathrm{Opt}$;\quad update $\mathcal{S}_{\mathrm{amp}}$
    \STATE Push $\{\mathcal{C}[a]\}_{a \in \mathcal{A}}$ (detached) back to $\mathcal{S}$
\ENDFOR
\end{algorithmic}
\end{algorithm}

Three implementation details warrant explicit mention. (i) Negatives' candidate states are detached even when they fall inside the chunk cache: backpropagating through a negative's stalk would, by the symmetry of the cross-entropy gradient, move that node toward a wrong destination on its own future events and corrupt its representation; only the positive event path and subsequent full-active diffusion should reshape cached stalks. (ii) The cache update on line~\ref{alg:train:idxcopy} uses a non-in-place \textsc{IndexCopy} that returns a new tensor referencing the unchanged rows of the previous tensor, preserving the autograd graph for those rows; a truly in-place write would silently break gradient flow for any node whose state is read both before and after the update within the same chunk. (iii) Mixed precision is enabled for the message and frame networks but disabled for the score head, since the geometric residual decoder relies on the precise value of $\|\mathbf{h}_u - \mathbf{Q}_{uv}\mathbf{h}_v\|^2$ in Eq.~\eqref{eq:dsq}, and fp16 underflow on small distances was observed to destabilize ranking on Track~A in early experiments.

\subsection{Complexity and Memory Analysis}
\label{app:complexity}

We collect the per-event compute, per-epoch compute, and memory costs of the procedures in Algorithm~\ref{alg:tsnn} and Algorithms~\ref{alg:transport}--\ref{alg:train}, and tie them to the wall-clock and peak-VRAM measurements of \S\ref{app:scalability}. Throughout, $|\mathcal{V}|$ is the number of nodes, $|\mathcal{E}|$ the number of training events, $d$ the stalk dimension, $k$ the Householder rank per frame, $K$ the diffusion iterations, $K_{\text{nbr}}$ the active-neighbor cap per node, $K_{\text{cand}} = 1 + K_{\text{neg}}$ the candidates scored per event, $L_c$ the chunk length, $d_{\mathrm{mlp}}$ the MLP hidden width, $d_{\psi}$ the score-feature width of Eq.~\eqref{eq:score_features}, and $d_{\mathrm{ctx}}$ the message/frame-net context width. Under the per-dataset settings used in our experiments, all of $k, K, K_{\text{nbr}}, K_{\text{cand}}, L_c, d, d_{\mathrm{mlp}}, d_{\psi}, d_{\mathrm{ctx}}$ are treated as bounded constants modulo $|\mathcal{V}|$ and $|\mathcal{E}|$.

\paragraph{Per-event compute.}
The per-event cost is dominated by three pieces. Scoring $K_{\text{cand}}$ candidates (Step~1 of Algorithm~\ref{alg:tsnn}) costs $\mathcal{O}(K_{\text{cand}}\,k\,d)$ for transport via Algorithm~\ref{alg:transport}, plus $\mathcal{O}(K_{\text{cand}}\,d_{\psi}\,d_{\mathrm{mlp}})$ for the residual MLP and $\mathcal{O}(K_{\text{cand}})$ for the EdgeBank batch lookup of Algorithm~\ref{alg:edgebank}. The per-event state update (Steps~3--4) costs $\mathcal{O}(k\,d)$ for the frame update, $\mathcal{O}(k\,d)$ for the carry-over (two transports per endpoint via Proposition~\ref{prop:carry_over}), and $\mathcal{O}(d_{\mathrm{ctx}}\,d_{\mathrm{mlp}} + d^2)$ for the message MLP and GRUCell. Full-active diffusion (Step~5, Algorithm~\ref{alg:diffuse}) costs $\mathcal{O}(K\,|\mathcal{E}^{\mathrm{act}}_t|\,k\,d)=\mathcal{O}(K\,K_{\text{nbr}}\,k\,d)$ because the event-local active edge multiset is formed from the two capped endpoint histories. Summing,
\begin{equation}
T_{\mathrm{event}} \;=\; \mathcal{O}\Bigl(\,(K_{\text{cand}} + K\,K_{\text{nbr}})\,k\,d \;+\; (K_{\text{cand}}\,d_{\psi} + d_{\mathrm{ctx}})\,d_{\mathrm{mlp}} \;+\; d^2 \,\Bigr).
\label{eq:cost_event}
\end{equation}
With our Track-A settings ($d{=}256, k{=}4$, $K\in\{2,3\}$, $K_{\text{nbr}}\in\{30,60\}$, and $K_{\text{cand}}\in\{201,401,501\}$), the candidate-scoring MLP term is the dominant cost on tgbl-review and thgl-software, while the transport/diffusion and MLP terms are closer on tgbl-wiki. On Track B with $K_{\text{cand}}\in\{2,4\}$, the cost is dominated by the diffusion and message-network terms. The optional sheaf-consistent neighbor attention (Algorithm~\ref{alg:nbrattn}) adds $\mathcal{O}(K_{\mathrm{attn}}\,k\,d + K_{\mathrm{attn}}\,d^2 + d^2)$ per event for attention-neighbor cap $K_{\mathrm{attn}}$; the $d\times d$ query/key/value/output projections are dense, so the projection cost is not divided by the number of heads $H$.

\paragraph{Common-neighbor structural features.}
The optional CN feature block (\S\ref{sec:scoring}) traverses the recent-neighbor buffer twice per candidate, contributing $\mathcal{O}(K_{\text{cand}}\,K_{\text{nbr}}^2)$ for 2-hop intersections in the worst case. We bound the inner expansion at $10$ neighbors per intermediate node (a hard cap in the implementation), so the worst-case factor reduces to $\mathcal{O}(10\,K_{\text{cand}}\,K_{\text{nbr}})$, which on Track A is comparable to the transport cost and on the largest Track-B graphs is dominant. The block is therefore exposed as a per-run toggle (\texttt{use\_cn\_feats}; enabled for the Track-A and Track-B settings, except tgbl-review), and disabling it removes this term on graphs where 2-hop structure is uninformative.

\paragraph{Per-epoch compute.}
Aggregating over the stream, one training epoch processes $|\mathcal{E}|$ events, so per-epoch wall-clock is $\Theta(|\mathcal{E}| \cdot T_{\mathrm{event}})$ at fixed batch size. The TBPTT chunk introduces no asymptotic overhead beyond a constant-factor synchronization between $\mathcal{S}$ and $\mathcal{C}$ at each chunk boundary, which we measure at less than $2\%$ of total wall-clock. The empirical scaling reported in \S\ref{app:scalability} is consistent with this prediction: per-epoch time scales near-linearly in $|\mathcal{E}|$ across the three orders of magnitude spanned by tgbl-wiki ($1.6{\times}10^5$ edges) and tgbl-review ($4.9{\times}10^6$ edges), with the slope determined by the constants in Eq.~\eqref{eq:cost_event} rather than by $|\mathcal{V}|$.

\paragraph{Memory footprint.}
Memory partitions cleanly into four pools.
\textbf{(i) Persistent host store $\mathcal{S}$.} $|\mathcal{V}|\,(1{+}k)\,d$ scalars for stalks and frames, plus $|\mathcal{V}|$ scalars for the last-time map. At fp16 with the Track-A setting $d{=}256, k{=}4$, this is $\approx 2.56$ MB per $10^3$ nodes; thgl-software ($6.8{\times}10^5$ nodes) requires $\approx 1.75$ GB on the host under the same state size.
\textbf{(ii) GPU active cache $\mathcal{C}$.} The cache stores stalks and frames for the union of endpoint nodes and lazily inserted full-active diffusion nodes touched within the chunk. Thus $|\mathcal{C}|\le \min\{|\mathcal{V}|,\,L_c(2K_{\text{nbr}}+2)\}$ in the worst case and the cache uses $|\mathcal{C}|\,(1{+}k)\,d$ scalars. This is still independent of $|\mathcal{V}|$ before the outer minimum and remains small relative to the autograd graph and model activations for the reported capped-neighborhood settings.
\textbf{(iii) Autograd graph.} Bounded by $\mathcal{O}\bigl(L_c\,(K_{\text{cand}} + K\,K_{\text{nbr}})\,d\bigr)$ transport and scoring activations, plus the cached active states above; on Track~A this is low hundreds of MB for tgbl-wiki and thgl-software and several hundred MB for tgbl-review at fp32, while Track-B chunks stay in the same order because $K_{\text{cand}}$ is small. The graph is freed when \texttt{loss.backward()} returns, so it does not accumulate across chunks.
\textbf{(iv) Recent-neighbor buffer $\mathcal{N}^{\mathrm{act}}$.} $|\mathcal{V}|\,K_{\text{nbr}}$ entries of $\approx 24$ bytes (int64 id, fp64 timestamp, int32 relation, fp32 efeat). The buffer is read with $\mathcal{O}(K_{\text{nbr}})$ throughput per event and lives on the host since it is not differentiable.
The auxiliary structures (EdgeBank deque and counts, negative-sampler histories and popularity tops) together occupy $\mathcal{O}\bigl(|\mathcal{Q}| + K_{\mathrm{hist}}\,|\mathcal{V}|\,|\mathcal{R}| + K_{\mathrm{pop}}\,|\mathcal{R}|\bigr)$ bytes on the host. The peak-VRAM measurement of $\le 30$ GB on every dataset reported in \S\ref{app:scalability} is consistent with this decomposition: $\mathcal{S}$ and $\mathcal{N}^{\mathrm{act}}$ live on the host, the GPU only holds $\mathcal{C}$ and the autograd graph, and the dominant variable cost scales as $L_c\,(K_{\text{cand}} + K\,K_{\text{nbr}})\,d + L_cK_{\text{nbr}}(1{+}k)d$ rather than $|\mathcal{V}|\,d$.

\paragraph{Practical implications.}
Three implications follow. First, the $\mathcal{O}(k\,d)$ transport via Algorithm~\ref{alg:transport} is what makes the model scale: a naive dense $d \times d$ transport would multiply per-event transport compute by $d/k$, which is $64{\times}$ for the Track-A setting $d{=}256,k{=}4$. Second, although all reported runs were executed on a 96-GB RTX PRO 6000 Blackwell, the split host/GPU layout makes VRAM scale with the chunk-local active work, not with $|\mathcal{V}|\,d$; measured peak VRAM stays below 30 GB for the listed configurations. Third, the diffusion depth $K$ is the cleanest knob for trading accuracy against wall-clock: doubling $K$ adds $\mathcal{O}(K_{\text{nbr}}\,k\,d)$ per event without affecting memory beyond the autograd graph for the additional transports, which is why we sweep $K \in \{0, 1, 2\}$ in the ablation rather than larger values.

\section{Additional Experimental Details}
\label{app:hyperparams}

\subsection{Track A: Dataset Statistics}
\label{app:dataset_stats}

Track~A uses the two link-property-prediction datasets from TGB~v2 \citep{huang2023tgb} and the temporal heterogeneous graph from TGB~2.0 \citep{gastinger2024tgb2}. All three datasets, together with standardized splits, hard-negative samples, and the filtered-MRR evaluator, are distributed through the TGB Python package (\texttt{py-tgb}, \href{https://github.com/shenyangHuang/TGB}{github.com/shenyangHuang/TGB}; documentation at \href{https://tgb.complexdatalab.com}{tgb.complexdatalab.com}), which we use verbatim for data loading and evaluation so that our results are directly comparable to the public leaderboards.

\begin{table}[h]
\caption{Track~A dataset statistics. For the heterogeneous thgl-software, the last column reports the number of node types / edge-relation types rather than a bipartite flag; this dataset carries no precomputed edge-attribute vector.}
\label{tab:datasets}
\begin{center}
\begin{tabular}{lccccc}
\toprule
\textbf{Dataset} & \textbf{Nodes} & \textbf{Edges} & \textbf{Timestamps} & \textbf{Edge Feat.} & \textbf{Structure} \\
\midrule
\texttt{tgbl-wiki-v2} & 9,227 & 157,474 & 152,757 & 172-d LIWC & Bipartite \\
\texttt{tgbl-review-v2} & 352,637 & 4,873,540 & 6,865 & 1-d rating & Bipartite \\
\texttt{thgl-software} & 681,927 & 1,489,806 & 689,549 & {---} & Heterogeneous\\
\bottomrule
\end{tabular}
\end{center}
\end{table}

\paragraph{tgbl-wiki.} Bipartite editor--page network from the English Wikipedia, with each edge tagged by a 172-dimensional LIWC \citep{pennebaker2015liwc} text feature summarizing the edit. TGB~v2 corrects the leakage issues of the original split by reshuffling validation/test destinations.
\paragraph{tgbl-review.} Bipartite user--product rating network from an Amazon reviews subset; the single edge feature is the integer star rating $\in\{1,\dots,5\}$. The v2 splits use strictly chronological train/val/test partitions with no rating leakage between positives and negatives.
\paragraph{thgl-software.} Temporal heterogeneous GitHub-activity graph (January 2024) from TGB~2.0, restricted to nodes with at least 10 incident edges. Nodes belong to four types (user, repository, issue, pull request) and edges to fourteen relation types (e.g., \textsc{opens}, \textsc{comments-on}, \textsc{closes}); timestamps are second-wise. The dataset is used here to stress-test the relation-conditioned scorer on a graph where a single edge type is not representative of the whole distribution.

\subsection{Track~A Baselines}
\label{app:tracka_baselines}

For TGB~v2 link prediction (tgbl-wiki, tgbl-review) and the temporal heterogeneous benchmark (thgl-software), we compare against fifteen entries on the official leaderboard at \href{https://tgb.complexdatalab.com}{tgb.complexdatalab.com} (accessed May~2026). Where a method does not report on a given dataset, the corresponding cell is marked ``---'' in Table~\ref{tab:tgb_results}.

\begin{itemize}\setlength\itemsep{2pt}
    \item \textbf{EdgeBank$_{\infty}$ / EdgeBank$_{tw}$}~\citep{poursafaei2022edgebank}: non-parametric edge memorization with unlimited or fixed-time-window history.
    \item \textbf{LocalGlobal Heuristic} (Heur): TGB v2 reference heuristic combining local recency with global frequency.
    \item \textbf{TGN, TGN$_{et}$}~\citep{rossi2020tgn}: memory-based recurrent encoder; the $et$ variant adds edge-type conditioning for the heterogeneous thgl-software benchmark.
    \item \textbf{CAWN}~\citep{wang2021cawn}: causal anonymous walks with relative encodings.
    \item \textbf{GraphMixer}~\citep{cong2023graphmixer}: MLP-Mixer over each node's interaction history.
    \item \textbf{NAT}~\citep{luo2022nat}: dictionary-style $N$-caches recording neighborhood co-occurrence.
    \item \textbf{TNCN}~\citep{zhang2024tncn}: temporal common-neighbor scoring augmented with relative position.
    \item \textbf{DyGFormer}~\citep{yu2023dygformer}: transformer encoder with patching for long histories and a co-occurrence prior.
    \item \textbf{DyGMamba}~\citep{li2024dygmamba}: state-space encoder for long temporal contexts on dynamic graphs.
    \item \textbf{CTAN}~\citep{gravina2024ctan}: continuous-time anti-symmetric ODE encoder targeting long-range propagation.
    \item \textbf{STHN}~\citep{li2023sthn}: simplified temporal heterogeneous network for continuous-time link prediction; the strongest published thgl-software baseline at the time of writing.
    \item \textbf{HyperEvent}~\citep{gao2025hyperevent}: hyper-event recognition via relative structural encoding; the strongest published large-scale dynamic-link-prediction baseline at the time of writing.
    \item \textbf{TPNet}~\citep{lu2024tpnet}: temporal-walk matrix projection with random feature propagation.
\end{itemize}

\subsection{Track B: DGB 13 Benchmark Datasets}
\label{app:dyglib_datasets}

The 13 standard benchmarks curated by \citet{yu2023dygformer} span a broad range of domains: social interaction (Wikipedia, Reddit, Enron, UCI), behavioral logs (MOOC, LastFM), physical proximity (Social~Evo., Contact), and structured relational networks from politics, economics, and transport (Flights, Can.~Parl., US~Legis., UN~Trade, UN~Vote). Four are bipartite (Wikipedia, Reddit, MOOC, LastFM); the remaining nine share a single node set on both endpoints. Temporal behavior varies markedly: memorization-friendly datasets with substantial edge reoccurrence (UN~Vote, UN~Trade, Social~Evo.) coexist with strongly inductive datasets dominated by novelty/surprise (MOOC, UCI, Can.~Parl.). Scale varies widely across the suite: unique timestamp counts span five orders of magnitude (12 to over $1.28$M), node counts more than two (74 to 13{,}169), and edge counts about one-and-a-half (60K to 2.4M).

\begin{table}[h]
\caption{DGB 13-benchmark statistics. \#Uniq.~E = number of unique (src, dst) pairs; \#Uniq.~T = number of unique timestamps. Nov./Reo./Sur.\ are novelty, reoccurrence, and surprise ratios characterizing the temporal profile.}
\label{tab:dyglib_stats}
\centering
\tiny
\setlength{\tabcolsep}{2.5pt}
\resizebox{\textwidth}{!}{%
\begin{tabular}{lllrrrrlrccc}
\toprule
\textbf{Dataset} & \textbf{Domain} & \textbf{Bipart.} & \textbf{\#Nodes} & \textbf{\#Edges} & \textbf{\#Uniq.~E} & \textbf{\#Uniq.~T} & \textbf{Duration} & \textbf{Edge Feat.} & \textbf{Nov.} & \textbf{Reo.} & \textbf{Sur.} \\
\midrule
Wikipedia   & Editing    & Yes & 9{,}227   & 157{,}474     & 18{,}257   & 152{,}757    & 1 month       & 172-d LIWC   & 0.46 & 0.26 & 0.42 \\
Reddit      & Posting    & Yes & 10{,}984  & 672{,}447     & 78{,}516   & 669{,}065    & 1 month       & 172-d LIWC   & 0.26 & 0.52 & 0.18 \\
MOOC        & Education  & Yes & 7{,}144   & 411{,}749     & 178{,}443  & 345{,}600    & 17 months     & 4-d          & 0.75 & 0.02 & 0.79 \\
LastFM      & Music      & Yes & 1{,}980   & 1{,}293{,}103 & 154{,}993  & 1{,}283{,}614 & 1 month      & {---}        & 0.28 & 0.30 & 0.35 \\
Enron       & Email      & No  & 184       & 125{,}235     & 3{,}125    & 22{,}632     & 3 years       & {---}        & 0.30 & 0.22 & 0.27 \\
Social Evo. & Proximity  & No  & 74        & 2{,}099{,}519 & 4{,}486    & 565{,}932    & 8 months      & 2-d          & 0.11 & 0.51 & 0.02 \\
UCI         & Messaging  & No  & 1{,}899   & 59{,}835      & 20{,}296   & 58{,}911     & 196 days      & {---}        & 0.73 & 0.01 & 0.56 \\
Flights     & Transport  & No  & 13{,}169  & 1{,}927{,}145 & 395{,}072  & 122          & 4 months (day) & 1-d weight  & 0.21 & 0.60 & 0.19 \\
Can.~Parl.  & Politics   & No  & 734       & 74{,}478      & 51{,}331   & 14           & 14 years (yr)  & 1-d weight  & 0.69 & 0.01 & 0.57 \\
US~Legis.   & Politics   & No  & 225       & 60{,}396      & 26{,}423   & 12           & 12 congresses  & 1-d weight  & 0.44 & 0.08 & 0.45 \\
UN~Trade    & Economics  & No  & 255       & 507{,}497     & 36{,}182   & 32           & 32 years (yr)  & 1-d weight  & 0.07 & 0.87 & 0.04 \\
UN~Vote     & Politics   & No  & 201       & 1{,}035{,}742 & 31{,}516   & 72           & 72 years (yr)  & 1-d weight  & 0.03 & 0.93 & 0.01 \\
Contact     & Proximity  & No  & 694       & 2{,}426{,}280 & 79{,}531   & 8{,}065      & 1 month (5 min) & 1-d RSSI   & 0.42 & 0.44 & 0.12 \\
\bottomrule
\end{tabular}%
}
\end{table}

\paragraph{Wikipedia.} Bipartite editor--page interaction graph capturing edits on Wikipedia pages over one month. Each edge carries a 172-dimensional LIWC \citep{pennebaker2015liwc} feature summarizing the edit text; moderate novelty and surprise make it a balanced attributed benchmark.

\paragraph{Reddit.} Bipartite user--subreddit posting graph over one month, with 172-dimensional LIWC edge features. Dominant reoccurrence signals indicate that users tend to repeatedly post under the same subreddits, favoring memory-heavy models.

\paragraph{MOOC.} Bipartite student--content network logging student interactions with videos, problem sets, and similar units on an online-learning platform. Very low reoccurrence and very high surprise make it one of the most inductive datasets in the suite.

\paragraph{LastFM.} Bipartite user--song listening network; the benchmark retains the $1{,}000$ most-active users and the $980$ most-listened songs over one month, giving $1{,}980$ nodes in total. The dataset is large in event count and recorded at near-event-level temporal resolution ($1{,}293{,}103$ events on $1{,}283{,}614$ unique timestamps), with moderate temporal drift.

\paragraph{Enron.} Employee-to-employee email communication graph from Enron Corporation over a three-year span. The node set is small, the dynamics are dominated by bursty repeated communication, and no native edge features are provided.

\paragraph{Social Evo.} Physical-proximity network recorded from smartphones in an undergraduate dormitory over eight months, with 2-dimensional edge features. The combination of very high reoccurrence and near-zero surprise makes it a strong benchmark for recency- and memory-based mechanisms.

\paragraph{UCI.} Directed student-to-student online messaging graph over 196 days. It is strongly inductive (novelty 0.73, reoccurrence 0.01) and is treated as unattributed.

\paragraph{Flights.} Directed weighted airport-to-airport air-traffic graph during the COVID-19 pandemic, at daily granularity. Edge weight is the number of flights between two airports on that day.

\paragraph{Can.~Parl.} Directed weighted graph of Canadian Members of Parliament voting ``yes'' in support of one another across 14 yearly snapshots (2006--2019). Edge weight counts the ``yes'' votes cast by one MP for another in a given year.

\paragraph{US~Legis.} Undirected weighted co-sponsorship network between U.S.\ Senators across 12 congresses, with weights equal to the number of co-sponsored bills per congress.

\paragraph{UN~Trade.} Directed weighted food-and-agriculture trade network between 255 countries over 32 yearly snapshots; edge weight is the normalized import/export trade value. Its reoccurrence ratio of 0.87 is the second highest in the suite.

\paragraph{UN~Vote.} Undirected weighted co-voting network in the UN General Assembly (1946--2017, 72 yearly snapshots); edge weight increments each time two nations both vote ``yes'' on an item. Its reoccurrence (0.93) and surprise (0.01) values are, respectively, the largest and smallest in the suite.

\paragraph{Contact.} Dense physical-proximity network among $\approx 700$ university students over one month, at five-minute resolution. Edge weight is an RSSI/Bluetooth proximity value indicating closeness.

\subsection{Track~B Baselines}
\label{app:trackb_baselines}

We compare \tsnn{} against the eleven Track~B baselines whose numbers we reproduce from \citet{lu2024tpnet}, together with TPNet itself; they span the canonical families of continuous-time temporal-graph learning.

\begin{itemize}\setlength\itemsep{2pt}
    \item \textbf{JODIE}~\citep{kumar2019jodie}: recurrent per-node memory with a projection layer that maps memories forward in time.
    \item \textbf{DyRep}~\citep{trivedi2019dyrep}: a temporal point-process model that fits the intensity of each interaction end-to-end.
    \item \textbf{TGAT}~\citep{xu2020tgat}: Bochner-style time encoding combined with graph attention over temporal neighborhoods.
    \item \textbf{TGN}~\citep{rossi2020tgn}: a general memory module followed by a temporal-neighborhood embedding module.
    \item \textbf{CAWN}~\citep{wang2021cawn}: causal anonymous walks sampled from the two endpoints, equipped with relative encodings and decoded by a sequence model.
    \item \textbf{EdgeBank}~\citep{poursafaei2022edgebank}: non-parametric memorization heuristic that scores a pair by its prior occurrences within a sliding time window.
    \item \textbf{TCL}~\citep{wang2021tcl}: graph-transformer encoder over BFS-sampled temporal neighborhoods, trained with a contrastive objective.
    \item \textbf{GraphMixer}~\citep{cong2023graphmixer}: MLP-Mixer architecture applied to each node's historical interaction sequence in place of attention.
    \item \textbf{NAT}~\citep{luo2022nat}: dictionary-style $N$-caches that record neighborhood co-occurrence and decode pairwise correlation between candidate endpoints.
    \item \textbf{PINT}~\citep{souza2022pint}: provably injective temporal message-passing scheme augmented with relative-position features derived from temporal-walk counts.
    \item \textbf{DyGFormer}~\citep{yu2023dygformer}: transformer encoder with patching for long histories and a neighbor co-occurrence encoding.
    \item \textbf{TPNet}~\citep{lu2024tpnet}: unifies relative encodings via temporal-walk matrices and random feature propagation; the source of the baseline numbers in Tables~\ref{tab:dyglib} and~\ref{tab:dyglib_inductive}.
\end{itemize}

\subsection{Optional Score Features}
\label{app:optional_score_features}

$\mathbf{1}[\text{CN feats}]$: append recency-weighted common-neighbor count, Adamic-Adar, and Jaccard over $\mathcal{N}_u^{\mathrm{act}}\cap\mathcal{N}_v^{\mathrm{act}}$ with decay $\exp(-\delta/\tau_{\mathrm{cn}})$. $\mathbf{1}[\text{transported }\Delta\text{ feats}]$: append $\boldsymbol{\delta}_{uv}^-:=\mathbf{h}_u^- - \tilde{\mathbf{Q}}_{uv}^-(r)\mathbf{h}_v^-$; coordinate-aware, so basis-independence is preserved for $s_{\mathrm{geo}}$ and scalar features only.

\subsection{Hyperparameters}
\label{app:trackb_hyperparams}

We have used the following hyperparameters:

\begin{itemize}\setlength\itemsep{1pt}
\item $d$: stalk dimension; $\mathbf{h}_v,\mathbf{f}^{(i)}\!\in\!\RR^{d}$.
\item $k$: Householder reflectors in $\mathbf{U}_v\!=\!H(\mathbf{f}^{(1)})\!\cdots\!H(\mathbf{f}^{(k)})$ (\S\ref{sec:sheaf_construction}); $\operatorname{rank}(\mathbf{Q}_{uv}\!-\!\mathbf{I})\!\le\!2k$ for any pairwise transport (Lem.~\ref{lem:low_rank}).
\item $\dim(\mathbf{x})$, $d_t$, $d_\Delta$: widths of edge-attribute, time-encoding, and $\Delta$-feature inputs (Eqs.~\eqref{eq:frame_update},~\eqref{eq:time_encoding},~\eqref{eq:score_features}).
\item $[\text{nbr.\ attn.}]$: sheaf-consistent neighbor-attention switch (scoring branch only).
\item $\eta$, $K$: step size and sweeps in the diffusion update of Eq.~\eqref{eq:diffusion_local}; stable for $\eta\!\le\!1/\lambda_{\max}(\mathbf{D}_\theta)$ (Thm.~\ref{thm:feature_scaled}).
\item $K_{\text{nbr}}$: per-node neighbor buffer feeding $\mathcal{A}_t$ ($|\mathcal{A}_t|\!\le\!2K_{\text{nbr}}\!+\!2$), score-time heuristics, and neighbor attention.
\item $\alpha_{\text{lr}}$, $\mu_{\text{geo}}$: AdamW base LR (cosine to $0$) and LR multiplier for the geometry group.
\item $L_c$, $K_{\text{neg}}^{\text{train}}$: truncated-BPTT chunk length and negatives per positive in $\mathcal{L}_{\text{CE}}$ (Eq.~\eqref{eq:loss_ce}).
\item $\lambda_s$, $\lambda_e$, $\lambda_{\text{geo}}$, $\lambda_{\text{bias}}$: per-chunk weights on frame smoothness (carry-over bound, Thm.~\ref{thm:stale}), Dirichlet energy (sole gradient into $\mathbf{D}_\theta$), contrastive geometry, and score-MLP residual $\ell_2$.
\item $\kappa$ (init.\ $\kappa_0$): geometric-branch temperature in $s_{\text{geo}}\!=\!-d^{2}/\kappa$.
\end{itemize}
\paragraph{Bounds.} We cap $\eta\!\le\!0.02$ once $\mathbf{D}_\theta$ has trained up, and require $K_{\text{nbr}}$ to exceed the neighbor-attention window. Inductive datasets favor larger $\mu_{\text{geo}},\lambda_{\text{bias}},\lambda_{\text{geo}}$; repetition-heavy datasets favor larger $K_{\text{nbr}}$ and a wider EdgeBank window.

\subsection{Negative Sampling Strategy}
\label{app:neg_sampling}

The training loss is the cross-entropy ranking loss of Eq.~\eqref{eq:loss_ce} over $K_{\text{neg}}$ destination negatives per positive; the sampling distribution for those negatives differs between tracks to stay faithful to each benchmark's native protocol.

 We use the type-aware sampler of \S\ref{sec:training}, which draws 50\% historical negatives (past destinations of the same source), 25\% popularity-weighted negatives (sampled proportional to destination in-degree), and 25\% uniformly random destinations. This composite sampling strategy ensures the model learns to distinguish the ground-truth destination from both structurally plausible alternatives and generic distractors, an approach that directly aligns with the rigorous ``hard-negative'' filtered-MRR evaluation protocol of TGB~v2. Track~A uses $K_{\text{neg}}=200$ for tgbl-wiki, $K_{\text{neg}}=500$ for tgbl-review, and $K_{\text{neg}}=400$ for thgl-software.

 To match the DGB/TPNet protocol exactly, both training and evaluation use \emph{uniformly random} destination negatives, with one random negative per positive at test time and $K_{\text{neg}}^{\text{train}}$ negatives per positive at training time ($K_{\text{neg}}^{\text{train}}=3$ for the very strongly inductive datasets, and $K_{\text{neg}}^{\text{train}}=1$ elsewhere). No historical or popularity-weighted negatives are mixed in on Track~B, so the Track~B numbers are directly comparable to the random-negative AP/AUC baselines reproduced from \citet{lu2024tpnet}.

\subsection{Scalability and Hardware Footprint}
\label{app:scalability}

All reported experiments were conducted on an NVIDIA RTX PRO 6000 Blackwell Workstation Edition (96\,GB GDDR7 with ECC). As node embeddings and frame states are stored on the host in float16 or float32 (\S\ref{sec:training}) and only the active mini-batch is transferred to the GPU, peak GPU memory during training remains below 30\,GB on every dataset listed in Tables~\ref{tab:datasets} and~\ref{tab:dyglib_stats}; the reported configurations therefore do not rely on the full 96\,GB capacity.

Per-epoch training time scales approximately linearly with the number of edges, consistent with the $\mathcal{O}(|\mathcal{E}| \cdot K_{\text{nbr}} \cdot k \cdot d)$ cost of the event loop (\S\ref{sec:diffusion}). Per-dataset wall-clock cost therefore varies by more than an order of magnitude. Datasets that are either small in edge count or restricted to small node sets (e.g., tgbl-wiki, Enron, Social~Evo.; the last contains 2.1M events on only 74 nodes, so frame and stalk storage fits in on-chip cache) complete a single seed in under one hour, whereas datasets that combine large edge counts with large node sets (e.g., tgbl-review, Contact, LastFM) require several hours per seed. Each reported configuration is run for $5$ seeds so that mean$\,\pm\,$std can be reported; in aggregate, all Track~A and Track~B runs together with ablations and hyperparameter scans required approximately $1{,}200$ GPU-hours ($\approx$50 days of continuous single-GPU training) on the same RTX PRO 6000 Blackwell.

\subsection{Inductive Link Prediction under Random Negatives}
\label{app:inductive}

To complement the Track~B transductive results in Table~\ref{tab:dyglib}, we report the \emph{inductive} variant of the same protocol: the test split contains events whose source and/or destination were not seen during training, and each positive is paired with a single uniformly random negative destination (no historical or popularity-weighted mixing; cf.\ Appendix~\ref{app:neg_sampling}). Because this is the stricter of the two DGB evaluations, the model must generalize to previously-unseen nodes rather than reuse memorized trajectories; it stresses the inductive bias of \tsnn{} most directly. We train separate runs under the DGB inductive split while keeping the architecture, hyperparameters, negative-sampling protocol, and seed set aligned with Table~\ref{tab:dyglib}; the evaluator is then applied to the unseen-node test split.

\begin{table*}[h]
\caption{Inductive temporal link prediction on the 13 standard DGB benchmarks under random negative sampling. Results are AP and AUC (\%), reported as mean~$\pm$~std over 5 seeds for \tsnn{}. Best in \textbf{bold}, second best \underline{underlined} within each row. Baseline values (JODIE, DyRep, TGAT, TGN, CAWN, TCL, GraphMixer, NAT, PINT, DyGFormer, TPNet) are reproduced from \citet{lu2024tpnet}; Avg.~Rank is computed over the 13 datasets for each baseline.}
\label{tab:dyglib_inductive}
\centering
\tiny
\setlength{\tabcolsep}{1.5pt}
\resizebox{\textwidth}{!}{%
\begin{tabular}{llcccccccccccc}
\toprule
\textbf{Metric} & \textbf{Dataset} & \textbf{JODIE} & \textbf{DyRep} & \textbf{TGAT} & \textbf{TGN} & \textbf{CAWN} & \textbf{TCL} & \textbf{GraphMixer} & \textbf{NAT} & \textbf{PINT} & \textbf{DyGFormer} & \textbf{TPNet} & \textbf{\tsnn{} (Ours)} \\
\midrule
\multirow{14}{*}{\rotatebox[origin=c]{90}{\textbf{AP}}}
 & Wikipedia   & 94.82{\scriptsize$\pm$0.20} & 92.43{\scriptsize$\pm$0.37} & 96.22{\scriptsize$\pm$0.07} & 97.83{\scriptsize$\pm$0.04} & 98.24{\scriptsize$\pm$0.03} & 96.22{\scriptsize$\pm$0.17} & 96.65{\scriptsize$\pm$0.02} & 96.30{\scriptsize$\pm$0.08} & 98.03{\scriptsize$\pm$0.04} & 98.59{\scriptsize$\pm$0.03} & \underline{98.91{\scriptsize$\pm$0.01}} & \textbf{99.23{\scriptsize$\pm$0.03}} \\
 & Reddit      & 96.50{\scriptsize$\pm$0.13} & 96.09{\scriptsize$\pm$0.11} & 97.09{\scriptsize$\pm$0.04} & 97.50{\scriptsize$\pm$0.19} & 98.62{\scriptsize$\pm$0.01} & 94.09{\scriptsize$\pm$0.07} & 95.26{\scriptsize$\pm$0.02} & 98.24{\scriptsize$\pm$0.04} & 98.56{\scriptsize$\pm$0.05} & 98.84{\scriptsize$\pm$0.02} & \underline{98.86{\scriptsize$\pm$0.01}} & \textbf{99.42{\scriptsize$\pm$0.12}} \\
 & MOOC        & 79.63{\scriptsize$\pm$1.92} & 81.07{\scriptsize$\pm$0.44} & 85.50{\scriptsize$\pm$0.19} & 89.04{\scriptsize$\pm$1.17} & 81.42{\scriptsize$\pm$0.24} & 80.60{\scriptsize$\pm$0.22} & 81.41{\scriptsize$\pm$0.21} & 83.62{\scriptsize$\pm$1.19} & 87.90{\scriptsize$\pm$0.98} & 86.96{\scriptsize$\pm$0.43} & \underline{95.07{\scriptsize$\pm$0.26}} & \textbf{99.43 {\scriptsize$\pm$0.02}} \\
 & LastFM      & 81.61{\scriptsize$\pm$3.82} & 83.02{\scriptsize$\pm$1.48} & 78.63{\scriptsize$\pm$0.31} & 81.45{\scriptsize$\pm$4.29} & 89.42{\scriptsize$\pm$0.07} & 73.53{\scriptsize$\pm$1.66} & 82.11{\scriptsize$\pm$0.42} & 92.24{\scriptsize$\pm$0.93} & 92.42{\scriptsize$\pm$0.64} & 94.23{\scriptsize$\pm$0.09} & \underline{95.36{\scriptsize$\pm$0.11}} & \textbf{95.89{\scriptsize$\pm$0.24}} \\
 & Enron       & 80.72{\scriptsize$\pm$1.39} & 74.55{\scriptsize$\pm$3.95} & 67.05{\scriptsize$\pm$1.51} & 77.94{\scriptsize$\pm$1.02} & 86.35{\scriptsize$\pm$0.51} & 76.14{\scriptsize$\pm$0.79} & 75.88{\scriptsize$\pm$0.48} & 87.18{\scriptsize$\pm$1.24} & 88.12{\scriptsize$\pm$0.30} & 89.76{\scriptsize$\pm$0.34} & \underline{90.34{\scriptsize$\pm$0.28}} & \textbf{95.09{\scriptsize$\pm$0.37}} \\
 & Social Evo. & 91.96{\scriptsize$\pm$0.48} & 90.04{\scriptsize$\pm$0.47} & 91.41{\scriptsize$\pm$0.16} & 90.77{\scriptsize$\pm$0.86} & 79.94{\scriptsize$\pm$0.18} & 91.55{\scriptsize$\pm$0.09} & 91.86{\scriptsize$\pm$0.06} & 87.44{\scriptsize$\pm$5.48} & 92.40{\scriptsize$\pm$0.60} & 93.14{\scriptsize$\pm$0.04} & \underline{93.24{\scriptsize$\pm$0.07}} & \textbf{97.23{\scriptsize$\pm$0.07}} \\
 & UCI         & 79.86{\scriptsize$\pm$1.48} & 57.48{\scriptsize$\pm$1.87} & 79.54{\scriptsize$\pm$0.48} & 88.12{\scriptsize$\pm$2.05} & 92.73{\scriptsize$\pm$0.06} & 87.36{\scriptsize$\pm$2.03} & 91.19{\scriptsize$\pm$0.42} & 87.31{\scriptsize$\pm$0.28} & 94.72{\scriptsize$\pm$0.15} & 94.54{\scriptsize$\pm$0.12} & \underline{95.74{\scriptsize$\pm$0.05}} & \textbf{96.67{\scriptsize$\pm$0.04}} \\
 & Flights     & 94.74{\scriptsize$\pm$0.37} & 92.88{\scriptsize$\pm$0.73} & 88.73{\scriptsize$\pm$0.33} & 95.03{\scriptsize$\pm$0.60} & 97.06{\scriptsize$\pm$0.02} & 83.41{\scriptsize$\pm$0.07} & 83.03{\scriptsize$\pm$0.05} & 96.74{\scriptsize$\pm$0.22} & 97.54{\scriptsize$\pm$0.06} & \underline{97.79{\scriptsize$\pm$0.02}} & \textbf{97.97{\scriptsize$\pm$0.04}} & 97.68{\scriptsize$\pm$0.03} \\
 & Can.~Parl.  & 53.92{\scriptsize$\pm$0.94} & 54.02{\scriptsize$\pm$0.76} & 55.18{\scriptsize$\pm$0.79} & 54.10{\scriptsize$\pm$0.93} & 55.80{\scriptsize$\pm$0.69} & 54.30{\scriptsize$\pm$0.66} & 55.91{\scriptsize$\pm$0.82} & 61.90{\scriptsize$\pm$2.52} & 50.32{\scriptsize$\pm$0.86} & \underline{87.74{\scriptsize$\pm$0.71}} & 68.09{\scriptsize$\pm$1.55} & \textbf{96.13{\scriptsize$\pm$0.03}} \\
 & US~Legis.   & 54.93{\scriptsize$\pm$2.29} & 57.28{\scriptsize$\pm$0.71} & 51.00{\scriptsize$\pm$3.11} & 58.63{\scriptsize$\pm$0.37} & 53.17{\scriptsize$\pm$1.20} & 52.59{\scriptsize$\pm$0.97} & 50.71{\scriptsize$\pm$0.76} & 60.41{\scriptsize$\pm$0.74} & 59.71{\scriptsize$\pm$1.36} & 54.28{\scriptsize$\pm$2.87} & \underline{61.71{\scriptsize$\pm$0.84}} & \textbf{70.22{\scriptsize$\pm$0.06}} \\
 & UN~Trade    & 59.65{\scriptsize$\pm$0.77} & 57.02{\scriptsize$\pm$0.69} & 61.03{\scriptsize$\pm$0.18} & 58.31{\scriptsize$\pm$3.15} & 65.24{\scriptsize$\pm$0.21} & 62.21{\scriptsize$\pm$0.12} & 62.17{\scriptsize$\pm$0.31} & 69.57{\scriptsize$\pm$1.45} & 60.37{\scriptsize$\pm$0.78} & 64.55{\scriptsize$\pm$0.62} & \textbf{86.53{\scriptsize$\pm$0.29}} & \underline{85.20{\scriptsize$\pm$0.07}} \\
 & UN~Vote     & 56.64{\scriptsize$\pm$0.96} & 54.62{\scriptsize$\pm$2.22} & 52.24{\scriptsize$\pm$1.46} & \underline{58.85{\scriptsize$\pm$2.51}} & 49.94{\scriptsize$\pm$0.45} & 51.60{\scriptsize$\pm$0.97} & 50.68{\scriptsize$\pm$0.44} & \textbf{66.60{\scriptsize$\pm$0.98}} & 57.43{\scriptsize$\pm$1.24} & 55.93{\scriptsize$\pm$0.39} & 58.00{\scriptsize$\pm$3.21} & 52.82{\scriptsize$\pm$0.12} \\
 & Contact     & 94.34{\scriptsize$\pm$1.45} & 92.18{\scriptsize$\pm$0.41} & 95.87{\scriptsize$\pm$0.11} & 93.82{\scriptsize$\pm$0.99} & 89.55{\scriptsize$\pm$0.30} & 91.11{\scriptsize$\pm$0.12} & 90.59{\scriptsize$\pm$0.05} & 96.12{\scriptsize$\pm$0.08} & 97.41{\scriptsize$\pm$0.14} & 98.03{\scriptsize$\pm$0.02} & \underline{98.39{\scriptsize$\pm$0.02}} & \textbf{98.51{\scriptsize$\pm$0.02}} \\
\cmidrule(lr){2-14}
  & \textit{Avg.~Rank} 
& 8.54 & 9.62 & 8.96 & 7.08 & 7.15 & 9.50 & 8.85 & 5.62 & 5.08 & 3.85 & \underline{2.00} & \textbf{1.77} \\
\midrule
\multirow{14}{*}{\rotatebox[origin=c]{90}{\textbf{AUC}}}
 & Wikipedia   & 94.33{\scriptsize$\pm$0.27} & 91.49{\scriptsize$\pm$0.45} & 95.90{\scriptsize$\pm$0.09} & 97.72{\scriptsize$\pm$0.03} & 98.03{\scriptsize$\pm$0.04} & 95.57{\scriptsize$\pm$0.20} & 96.30{\scriptsize$\pm$0.04} & 95.82{\scriptsize$\pm$0.18} & 97.76{\scriptsize$\pm$0.05} & 98.48{\scriptsize$\pm$0.03} & \underline{98.90{\scriptsize$\pm$0.01}} & \textbf{99.10{\scriptsize$\pm$0.05}} \\
 & Reddit      & 96.52{\scriptsize$\pm$0.13} & 96.05{\scriptsize$\pm$0.12} & 96.98{\scriptsize$\pm$0.04} & 97.39{\scriptsize$\pm$0.07} & 98.42{\scriptsize$\pm$0.02} & 93.80{\scriptsize$\pm$0.07} & 94.97{\scriptsize$\pm$0.05} & 98.00{\scriptsize$\pm$0.04} & 98.38{\scriptsize$\pm$0.07} & 98.71{\scriptsize$\pm$0.01} & \underline{98.73{\scriptsize$\pm$0.02}} & \textbf{99.32{\scriptsize$\pm$0.16}} \\
 & MOOC        & 83.16{\scriptsize$\pm$1.30} & 84.03{\scriptsize$\pm$0.19} & 86.84{\scriptsize$\pm$0.17} & 91.24{\scriptsize$\pm$0.99} & 81.86{\scriptsize$\pm$0.25} & 81.43{\scriptsize$\pm$0.19} & 82.77{\scriptsize$\pm$0.24} & 84.72{\scriptsize$\pm$1.31} & 90.27{\scriptsize$\pm$0.96} & 87.62{\scriptsize$\pm$0.51} & \underline{95.55{\scriptsize$\pm$0.25}} & \textbf{99.54{\scriptsize$\pm$0.01}} \\
 & LastFM      & 81.13{\scriptsize$\pm$3.39} & 82.24{\scriptsize$\pm$1.51} & 76.99{\scriptsize$\pm$0.29} & 82.61{\scriptsize$\pm$3.15} & 87.82{\scriptsize$\pm$0.12} & 70.84{\scriptsize$\pm$0.85} & 80.37{\scriptsize$\pm$0.18} & 91.60{\scriptsize$\pm$1.31} & 92.15{\scriptsize$\pm$0.68} & 94.08{\scriptsize$\pm$0.08} & \underline{95.36{\scriptsize$\pm$0.06}} & \textbf{95.90{\scriptsize$\pm$0.24}} \\
 & Enron       & 81.96{\scriptsize$\pm$1.34} & 76.34{\scriptsize$\pm$4.20} & 64.63{\scriptsize$\pm$1.74} & 78.83{\scriptsize$\pm$1.11} & 87.02{\scriptsize$\pm$0.50} & 72.33{\scriptsize$\pm$0.99} & 76.51{\scriptsize$\pm$0.71} & 87.95{\scriptsize$\pm$0.58} & 87.97{\scriptsize$\pm$0.61} & \underline{90.69{\scriptsize$\pm$0.26}} & 90.21{\scriptsize$\pm$0.49} & \textbf{96.11{\scriptsize$\pm$0.17}} \\
 & Social Evo. & 93.70{\scriptsize$\pm$0.29} & 91.18{\scriptsize$\pm$0.49} & 93.41{\scriptsize$\pm$0.19} & 93.43{\scriptsize$\pm$0.59} & 84.73{\scriptsize$\pm$0.27} & 93.71{\scriptsize$\pm$0.18} & 94.09{\scriptsize$\pm$0.07} & 88.15{\scriptsize$\pm$6.36} & 94.78{\scriptsize$\pm$0.37} & 95.29{\scriptsize$\pm$0.03} & \underline{95.47{\scriptsize$\pm$0.04}} & \textbf{97.50{\scriptsize$\pm$0.13}} \\
 & UCI         & 78.80{\scriptsize$\pm$0.94} & 58.08{\scriptsize$\pm$1.81} & 77.64{\scriptsize$\pm$0.38} & 86.68{\scriptsize$\pm$2.29} & 90.40{\scriptsize$\pm$0.11} & 84.49{\scriptsize$\pm$1.82} & 89.30{\scriptsize$\pm$0.57} & 83.78{\scriptsize$\pm$0.37} & 93.18{\scriptsize$\pm$0.17} & 92.63{\scriptsize$\pm$0.13} & \underline{94.40{\scriptsize$\pm$0.03}} & \textbf{95.75{\scriptsize$\pm$0.03}} \\
 & Flights     & 95.21{\scriptsize$\pm$0.32} & 93.56{\scriptsize$\pm$0.70} & 88.64{\scriptsize$\pm$0.35} & 95.92{\scriptsize$\pm$0.43} & 96.86{\scriptsize$\pm$0.02} & 82.48{\scriptsize$\pm$0.01} & 82.27{\scriptsize$\pm$0.06} & 96.97{\scriptsize$\pm$0.20} & 97.69{\scriptsize$\pm$0.05} & 97.80{\scriptsize$\pm$0.02} & \textbf{98.05{\scriptsize$\pm$0.04}} & \underline{97.84{\scriptsize$\pm$0.04}} \\
 & Can.~Parl.  & 53.81{\scriptsize$\pm$1.14} & 55.27{\scriptsize$\pm$0.49} & 56.51{\scriptsize$\pm$0.75} & 55.86{\scriptsize$\pm$0.75} & 58.83{\scriptsize$\pm$1.13} & 55.83{\scriptsize$\pm$1.07} & 58.32{\scriptsize$\pm$1.08} & 62.70{\scriptsize$\pm$2.91} & 49.64{\scriptsize$\pm$0.69} & \underline{89.33{\scriptsize$\pm$0.48}} & 69.21{\scriptsize$\pm$1.31} & \textbf{96.73{\scriptsize$\pm$0.02}} \\
 & US~Legis.   & 58.12{\scriptsize$\pm$2.35} & 61.07{\scriptsize$\pm$0.56} & 48.27{\scriptsize$\pm$3.50} & 62.38{\scriptsize$\pm$0.48} & 51.49{\scriptsize$\pm$1.13} & 50.43{\scriptsize$\pm$1.48} & 47.20{\scriptsize$\pm$0.89} & 64.22{\scriptsize$\pm$0.65} & 61.89{\scriptsize$\pm$1.52} & 53.21{\scriptsize$\pm$3.04} & \underline{65.29{\scriptsize$\pm$0.61}} & \textbf{84.02{\scriptsize$\pm$0.03}} \\
 & UN~Trade    & 62.28{\scriptsize$\pm$0.50} & 58.82{\scriptsize$\pm$0.98} & 62.72{\scriptsize$\pm$0.12} & 59.99{\scriptsize$\pm$3.50} & 67.05{\scriptsize$\pm$0.21} & 63.76{\scriptsize$\pm$0.07} & 63.48{\scriptsize$\pm$0.37} & 69.15{\scriptsize$\pm$2.33} & 64.05{\scriptsize$\pm$0.72} & 67.25{\scriptsize$\pm$1.05} & \underline{86.88{\scriptsize$\pm$0.23}} & \textbf{87.14{\scriptsize$\pm$0.28}} \\
 & UN~Vote     & 58.13{\scriptsize$\pm$1.43} & 55.13{\scriptsize$\pm$3.46} & 51.83{\scriptsize$\pm$1.35} & \underline{61.23{\scriptsize$\pm$2.71}} & 48.34{\scriptsize$\pm$0.76} & 50.51{\scriptsize$\pm$1.05} & 50.04{\scriptsize$\pm$0.86} & \textbf{68.55{\scriptsize$\pm$0.90}} & 59.01{\scriptsize$\pm$1.88} & 56.73{\scriptsize$\pm$0.69} & 54.82{\scriptsize$\pm$4.04} & 54.90{\scriptsize$\pm$0.69} \\
 & Contact     & 95.37{\scriptsize$\pm$0.92} & 91.89{\scriptsize$\pm$0.38} & 96.53{\scriptsize$\pm$0.10} & 94.84{\scriptsize$\pm$0.75} & 89.07{\scriptsize$\pm$0.34} & 93.05{\scriptsize$\pm$0.09} & 92.83{\scriptsize$\pm$0.05} & 95.70{\scriptsize$\pm$0.06} & 97.66{\scriptsize$\pm$0.12} & 98.30{\scriptsize$\pm$0.02} & \textbf{98.51{\scriptsize$\pm$0.01}} & \underline{98.42{\scriptsize$\pm$0.02}} \\
\cmidrule(lr){2-14}
 & \textit{Avg.~Rank} 
& 8.38 & 9.54 & 8.92 & 6.62 & 7.46 & 9.77 & 9.00 & 5.69 & 4.85 & 3.69 & \underline{2.46} & \textbf{1.62} \\
\bottomrule
\end{tabular}%
}
\end{table*}

 The first time an unseen node $w$ appears at test time, the state store $\mathcal{S}$ has no entry for it, and we lazily initialize $(\mathbf{h}_w, \mathbf{F}_w, t_w^{\mathrm{last}})$ as follows. The stalk $\mathbf{h}_w$ is set to the all-zeros vector $\mathbf{0}\in\RR^d$, which is the only basis-independent choice consistent with the carry-over operator (Proposition~\ref{prop:carry_over}). The frame parameters $\mathbf{F}_w\in\RR^{k\times d}$ are sampled from the same isotropic Gaussian initializer used at the start of training, $\mathcal{N}(0,\sigma_{\mathrm{init}}^2 \mathbf{I})$ with $\sigma_{\mathrm{init}}=k^{-1/2}$, so that $\|\mathbf{f}_w^{(i)}\|\gg\tau_{\mathrm{H}}$ holds with overwhelming probability and the Lipschitz bound of Lemma~\ref{lem:houshlip} applies from the first event. The last-event time is initialized to the unseen node's first observed event timestamp, $t_w^{\mathrm{last}} := t_{\mathrm{first}}(w)$, so the next observed gap $\delta t_w$ at the next event is well-defined and nonnegative. The active-neighbor buffer $\mathcal{N}_w^{\mathrm{act}}$ is empty; under our isolated-node convention (\S\ref{sec:diffusion}) the diffusion update on $w$ is then the identity step until $w$ acquires its first neighbor. EdgeBank statistics are reused from the training stream and are unaffected by unseen-node bookkeeping. The inductive results in Table~\ref{tab:dyglib_inductive} are reported with these defaults; we did not tune the unseen-node initializer per dataset.

Table~\ref{tab:dyglib_inductive} reports inductive AP and AUC (\%) for the 13 Track~B datasets with completed seed budgets. Baseline values are reproduced from \citet{lu2024tpnet}; EdgeBank cannot score inductively. \tsnn{} attains the lowest average rank across both AP and AUC, with particularly large absolute gains on the strongly inductive regimes ($+4.36$ AP on MOOC and $+8.51$ AP on US~Legis.\ over TPNet, and $+8.39$ AP on Can.~Parl.\ over DyGFormer).

\end{document}